\documentclass{article}
\usepackage{PRIMEarxiv}

\usepackage{authblk}

\usepackage{bm}
\usepackage{dirtytalk}
\usepackage{dcolumn}
\usepackage[T1]{fontenc}
\usepackage{graphicx}
\usepackage{graphics}
\usepackage{hyperref}
\usepackage[utf8]{inputenc}
\usepackage{latexsym}
\usepackage{multirow}
\usepackage{rotating}
\usepackage{color}
\usepackage{xcolor}
\usepackage{wrapfig}
\usepackage{subcaption}

\usepackage[export]{adjustbox}

\usepackage{amsmath}
\usepackage{amsthm}
\usepackage{amssymb}

\usepackage{cleveref}
\usepackage{enumitem}

\newtheorem{definition}{Definition}
\newtheorem{thm}{Theorem}

\newtheorem{corollary}{Corollary}
\newtheorem{remark}{Remark}

\newtheorem{proposition}{Proposition}
\newtheorem{lemma}{Lemma}

\newcommand\be{\begin{equation}}
\newcommand\ba{\begin{eqnarray}}
\newcommand\ee{\end{equation}}
\newcommand\ea{\end{eqnarray}}

\newcommand{\R}{\mathbb R}

\newcommand{\dd}{\mathrm d}
\newcommand{\pos}{\mathrm q}
\newcommand{\mom}{\mathrm p}

\DeclareMathOperator{\diag}{diag}

\usepackage{url}            % simple URL typesetting
\usepackage{booktabs}       % professional-quality tables
\usepackage{amsfonts}       % blackboard math symbols
\usepackage{nicefrac}       % compact symbols for 1/2, etc.
\usepackage{microtype}      % microtypography
\usepackage{lipsum}
\usepackage{fancyhdr}       % header
\usepackage{graphicx}       % graphics
\graphicspath{{media/}}     % organize your images and other figures under media/ folder
\newcommand*{\email}[1]{\href{mailto:#1}{\nolinkurl{#1}} } 
\title{Symplectic Neural Flows for Modeling and Discovery}

\author[*]{Priscilla Canizares}
\author[*]{Davide Murari}
\author[*]{Carola-Bibiane Sch\"onlieb}
\author[*]{Ferdia Sherry}
\author[*]{Zakhar Shumaylov}
\affil[*]{Department of Applied Mathematics and Theoretical Physics, University of Cambridge}
\affil[ ]{\texttt{\{pc464,dm2011,cbs31,fs436,zs334\}@cam.ac.uk}\vspace{-2.5em}}

\newcommand{\one}[1]{{#1}}
\newcommand{\two}[1]{{#1}}

\begin{document}

%\nolinenumbers

%\preprint{}
\maketitle
%----- MAKETITLE -----%
\begin{abstract}
Hamilton's equations are fundamental for modeling complex physical systems, where preserving key properties such as energy and momentum is crucial for reliable long-term simulations. Geometric integrators are widely used for this purpose, but neural network-based methods that incorporate these principles remain underexplored. This work introduces \texttt{SympFlow}, a time-dependent symplectic neural network designed using parameterized Hamiltonian flow maps. This design allows for backward error analysis and ensures the preservation of the symplectic structure. \texttt{SympFlow} allows for two key applications: (i) providing a time-continuous symplectic approximation of the exact flow of a Hamiltonian system\textemdash purely based on the differential equations it satisfies, and (ii) approximating the flow map of an unknown Hamiltonian system relying on trajectory data. We demonstrate the effectiveness of \texttt{SympFlow} on diverse problems, including chaotic and dissipative systems, showing improved energy conservation compared to general-purpose numerical methods and accurate approximations from sparse irregular data. {We also provide a thorough theoretical analysis of \texttt{SympFlow}, showing it can approximate the flow of any time-dependent Hamiltonian system, and providing an a-posteriori error estimate in terms of energy conservation.}
\end{abstract}


\section{Introduction}\label{sec:introduction}
Recent advances have demonstrated the ability of machine learning models to predict dynamics from time series data \cite{schmidhuber1997long,graves2012supervised}, solve differential equations in an unsupervised manner \cite{lagaris1998artificial,raissi2019physics}, learn Hamiltonian or Lagrangian structures from data \cite{greydanusHamiltonian2019,dierkes2023hamiltonian,cranmer2020lagrangian,kaltsas2024}, and symbolically discover underlying differential equations \cite{brunton2016discovering,cory2024evolving,Cranmer2020}. While promising, the applicability of these methods is often hindered by challenges related to trainability \cite{krishnapriyan2021characterizing}, scalability \cite{grossmann2024can}, and the specific complexities of the systems under consideration \cite{mcgreivy2024weak}.  

A key challenge in numerically solving differential equations is preserving critical physical quantities of interest, such as energy or momentum. Numerically reproducing a solution's qualitative behavior is often essential for obtaining reliable long-term predictions. General-purpose numerical solvers frequently fail in doing so. 

To address this problem, there is a thorough study of geometric integrators \cite{hairerGeometricNumericalIntegration2006}, specifically designed to preserve first integrals of motion or other properties of the underlying system. Similar geometric extensions for neural network-based integrators remain underdeveloped. This, coupled with the inherent errors in learned models, leads to error propagation that becomes increasingly severe over time. As a result, effectively training such models becomes a significant challenge \cite{10191822,krishnapriyan2021characterizing}, hindering the feasibility of long-term simulations \cite{wang2023long,mccabe2023towards,michalowska2024neural}.

This work introduces \texttt{SympFlow}, a neural flow constrained to be symplectic by construction. The symplectic constraint enables it to provide accurate and reliable long-term solutions for Hamiltonian systems. \texttt{SympFlow} is a universal approximator in the space of Hamiltonian flows (\Cref{thm:universal}) and its architecture permits the extraction of the exact \say{shadow} Hamiltonian corresponding to the neural flow (\Cref{eq:iterativeHam}). This facilitates backward analysis through Hamiltonian matching and allows for error-free switching between the Hamiltonian and flow representations. This flexibility makes our approach adaptable to diverse tasks, including unsupervised learning scenarios where the objective is to solve a given Hamiltonian differential equation, and supervised settings where the network learns to evolve a system solely based on observed trajectories. In essence, \texttt{SympFlow} acts as both: (i) a Hamiltonian neural network \cite{greydanusHamiltonian2019} capable of error-free evolution, and (ii) a neural flow \cite{bilovs2021neural} from which the underlying Hamiltonian can be extracted. We demonstrate the efficacy of \texttt{SympFlow} across a range of supervised and unsupervised tasks, including those involving noise and dissipation.

\subsection{Outline of the paper}
This paper is structured as follows: \Cref{sec:related} provides an overview of related work. \Cref{sec:methodology} introduces \texttt{SympFlow}\footnote{See \href{https://github.com/davidemurari/sympflow}{https://github.com/davidemurari/sympflow} for a PyTorch implementation of our architecture.}, detailing its architecture, discussing the various training approaches and in \Cref{sec:theory} we prove various theoretical properties related to \texttt{SympFlow}. Finally, \Cref{sec:experiments} presents the numerous experimental results demonstrating its effectiveness and usefulness for long-term stable integration, both in the \emph{supervised} and \emph{unsupervised} settings. We conclude with \Cref{sec:conclusions} summarizing our findings and expanding on future research directions.

\subsection{Related Work}\label{sec:related}
\paragraph{Deep Learning Integrators}
The deep learning revolution has profoundly impacted scientific computing, particularly in simulating physical systems governed by differential equations.

Physics-Informed Neural Networks (PINNs) \cite{lagaris1998artificial, raissi2019physics, wang2021physics} have emerged as a prominent approach, integrating domain knowledge, physical laws, and constraints directly into the learning process. PINNs have been demonstrated effective in solving forward and inverse problems for partial differential equations (PDEs) \cite{raissi2019physics, sun2020surrogate, zhu2019physics, karumuri2020simulator, sirignano2018dgm} and regularizing learning \cite{zhang2024biophysics}.

Among the efforts in using PINNs to solve differential equations, we mention in particular the works \cite{haitsiukevich2022learning,mattheakis2022hamiltonian}, where the authors focus on applying these techniques to Hamiltonian systems. Both of these papers focus on solving a single initial value problem, i.e., they fix an initial condition, and do not constrain their neural networks to reproduce the expected qualitative behavior of the Hamiltonian flow, such as being symplectic.

The surge in interest in PINNs stems from the limitations of traditional numerical solvers, which often struggle with computationally demanding scenarios, such as high-dimensional problems, non-linear and non-smooth PDEs requiring expensive fine grid discretization, and the need for repeated simulations across varying domain geometries, parameters, and different initial and boundary conditions \cite{han2018solving}.

Incorporating physical constraints, such as energy conservation, within PINNs has led to the development of structured neural networks \cite{celledoni2021structure, galimberti2023hamiltonian, jin2020sympnets, celledoni2023dynamical,tierz2025graph,hernandez2022thermodynamics}, for example enforcing point symmetry equivariance into the network architectures \cite{arora2024invariant, lagrave2022equivariant, shumaylov2024liealgebracanonicalizationequivariant}. It is crucial to note that while PINNs leverage physical constraints for regularization, they do not typically enforce these constraints explicitly within their architecture. This contrasts with our work, which does not follow a PINN-based strategy but focuses on constructing \emph{neural networks that intrinsically preserve the symplectic structure} of the phase space. 

\paragraph{Symplecticity in Neural Networks}
\one{The study of Hamiltonian and Lagrangian systems in the context of neural networks is motivated by several applications. The most popular is the discovery and simulation of physical systems, \cite{offen_symplectic_2022,OBERBLOBAUM2023114780,bertalanLearningHamiltonianSystems2019,Burby_2020,jin2020sympnets,tapley2024symplectic,greydanusHamiltonian2019,thangamuthu2022unravelling}, as we do in our work. Still, we also mention their relevance to improve the network adversarial robustness \cite{zhao2023adversarial}, the network trainability by preventing vanishing gradient issues \cite{galimberti2023hamiltonian}, or also in generative modeling through the paradigm of flow matching \cite{liu2023generalized,koshizuka2022neural,neklyudov2023computational}.} In recent years, there has been a surge of research dedicated to integrating symplectic structures into neural networks modeling Hamiltonian systems. These efforts can be broadly classified into two categories: \emph{fixed-step} and \emph{variable-step} methods, with significant variations in how the symplectic structure is used.

\emph{Fixed-step methods} explicitly construct a symplectic mapping between consecutive time steps.  Examples include SympNet \cite{jin2020sympnets}, and H\'enonNet \cite{Burby_2020}, which can interface with separable and non-separable Hamiltonian systems. Despite the symplectic constraint, these networks are universal approximators within the space of symplectic maps \cite{jin2020sympnets,tapley2024symplectic} and lead to controlled prediction errors \cite{pmlr-v139-chen21r}. 

\emph{Variable-step methods} predominantly utilize a Hamiltonian Neural Network (HNN) \cite{greydanusHamiltonian2019} or Neural ODE \cite{chen2018neural} framework. HNNs directly learn the Hamiltonian function, ensuring the recovery of conservative dynamics. Conversely, Neural ODEs recover continuous dynamics by integrating neural network-based differential equations. Both approaches often rely on general-purpose numerical methods, which can disrupt the symplectic structure and the conservative nature of the system. 

\noindent To address this, several works embed symplectic integrators within the network architecture.  Notable examples include SRNN \cite{Chen2020Symplectic}, TaylorNet \cite{tong2021symplectic}, SSINN \cite{NEURIPS2020_439fca36}, and NSSNN \cite{xiong2021nonseparable}. SymODEN \cite{Zhong2020Symplectic} further extends this framework by incorporating an external control term into the Hamiltonian dynamics. These methods aim to predict continuous system trajectories while preserving the symplectic structure of the phase space. 

\noindent However, in the broader literature, recent research in neural flows \cite{bilovs2021neural} has demonstrated significant improvements in long-time integration compared to Neural ODE counterparts, without the use of numerical solvers. This motivates our extension of fixed-step methods, particularly SympNets \cite{jin2020sympnets}, to the flow framework. By directly learning the Hamiltonian flow-map and being able to associate it with an exact analytical Hamiltonian function, the analysis is considerably simplified, circumventing the need for intricate procedures, for example, through discrete gradients \cite{matsubara2020deep}.

In a similar fashion to \texttt{SympFlow}, also the TSympOCNet network architecture in \cite{zhang2024timedependentsymplecticnetworknonconvex} is a time-dependent symplectic neural network. Both these two architecture take inspiration from SympNets, and introduce a time-dependency modifying its architecture. Our work departs from \cite{zhang2024timedependentsymplecticnetworknonconvex} in how the time-dependency is introduced and also in the focus of the research. While our work aims to theoretically analyse the properties of the proposed model, and experimentally evaluate its effectiveness compared to unconstrained networks, the focus of \cite{zhang2024timedependentsymplecticnetworknonconvex} is on developing a model to be used in the context of path planning problems.

While this review is not exhaustive, it is worth noting the existence of alternative approaches such as learning modified generating functions as symplectic map representations \cite{pmlr-v139-chen21r}, or directly addressing constrained Hamiltonian systems \cite{finzi2020simplifying, celledoni2023learning}.

\paragraph{Non-conservative systems}
Real-world dynamical systems often exhibit energy dissipation due to irreversible processes such as heat transfer, friction, and radiation. Accurately capturing such dynamics necessitates incorporating these effects into the equations of motion.

While neural networks have shown promise in modeling conservative Hamiltonian systems, extending these approaches to non-conservative dynamics presents significant challenges. Existing efforts primarily focus on augmenting the Hamiltonian framework. Within the context of HNNs and NODEs, one approach involves the separate parameterization of the Hamiltonian and the dissipative term \cite{sosanya2022dissipative}. Another more prevalent approach leverages the framework of port-Hamiltonian dynamics. For instance, \cite{PhysRevE.104.034312} extends Hamiltonian Neural Networks (HNNs) to port-Hamiltonian systems, while \cite{zhong2019dissipative} adapts SymODEN to this framework.  

However, to the best of our knowledge, no prior work addresses non-conservative dynamics within the context of symplectic flows.

This work adopts a distinct and simpler strategy based on the formulation proposed by \cite{Galley_2013, Galley_2014}. This formulation recasts non-conservative dynamics within a classical Hamiltonian framework by doubling the phase-space variables and computing the evolution equations in the corresponding doubled space. The counterparts of the phase-space variables follow a time-reversed trajectory, thus keeping the total energy of the augmented system constant. The final form of the solution is obtained by projecting back onto the original system, that is, taking the physical limit, recovering the original non-conservative dynamics (see \Cref{app:NCA} for a description of the procedure). This method allows for a unified treatment of conservative and non-conservative systems, enabling direct application of existing neural network architectures and techniques developed for Hamiltonian systems, without explicit modeling of dissipative forces. We have adapted this technique to recast non-conservative dynamics in a symplectic form and seamlessly model dissipation within \texttt{SympFlow}.

\subsection{Main contributions}

The main contributions can be summarized as follows:
\begin{itemize}
\item We introduce a novel time-dependent symplectic neural flow, \texttt{SympFlow}, designed using parameterized Hamiltonian flow maps. This network can be used both for approximating the flow of a Hamiltonian system, given the governing equations, and learning the underlying Hamiltonian directly from observed trajectory data.
\item Theoretically, we show that \texttt{SympFlow} is a universal approximator in the space of Hamiltonian flows in \Cref{thm:universal}. Furthermore, the ability to extract the underlying Hamiltonian from a trained \texttt{SympFlow} enables a-posteriori backward analysis of the approximated system in \Cref{thm:aposteriori}.
\item Practically, we demonstrate that \texttt{SympFlow} can effectively model and learn both conservative and non-conservative dynamics, preserving the symplectic structure even in the presence of dissipation. This is demonstrated for three systems: a Simple Harmonic Oscillator (\Cref{sec:SHO}), Henon-Heiles (\Cref{sec:HH}) and Damped Harmonic Oscillator (\Cref{sec:DHO}). 
\item Numerical examples highlight that \texttt{SympFlow} exhibits improved long-term energy behavior compared to unstructured neural networks and is more data efficient in supervised learning tasks.
\end{itemize}

\section{Methodology}\label{sec:methodology}

In this work we focus on canonical Hamiltonian systems, that is systems of ordinary differential equations (ODEs) of the form:
\begin{equation}
\frac{\dd x}{\dd t} = \mathbf J \nabla H(x),\quad\text{with}\quad \mathbf J = \begin{pmatrix}0 & \mathbf{I}_d \\ -\mathbf{I}_d & 0\end{pmatrix},
\label{eq:hamiltonian_ode}
\end{equation}
for \two{$d\in\mathbb{N}$}, a state variable $x\in \R^{2d}$ and a twice-continuously differentiable \emph{Hamiltonian} function $H:\R^{2d}\to\R$. In \Cref{eq:hamiltonian_ode}, the matrices $\mathbf{I}_d,0\in\R^{d\times d}$ are the identity and zero matrices respectively. The phase space variable is typically partitioned into a position $q\in \R^d$ and momentum $p\in\R^d$, $x = (q,p)$. Under standard, non-restrictive, assumptions on $H$, the corresponding initial value problem has a unique solution for any initial condition and initial time \cite{arnoldOrdinaryDifferentialEquations1991}, which can be used to define the corresponding time-$t$ \emph{flow map} $\phi_{H,t}:\R^{2d} \to \R^{2d}$:
\begin{equation}
\frac{\dd}{\dd t} \phi_{H, t}(x_0) = \mathbf J \nabla H(\phi_{H, t}(x_0))\quad\text{and}\quad \phi_{H, 0}(x_0)= x_0.\label{eq:hamiltonian_flow_map}
\end{equation}
\two{The notation $\frac{\dd}{\dd t}\phi_{H,t}(x_0)$ stands for the time derivative of the flow map $\phi_H:\R\times\R^{2d}\to\R^{2d}$ defined as $\phi_H(t,x_0)=\phi_{H,t}(x_0)$ for every $(t,x)\in\R\times \R^{2d}$, i.e., $\frac{\dd}{\dd t}\phi_{H,t}:=\partial_t \phi_H(t,\cdot)$.} Time-independent Hamiltonian systems as in \Cref{eq:hamiltonian_flow_map} conserve the Hamiltonian energy function, i.e., $H(\phi_{H,t}(x_0)) = H(x_0)$ for every $t\geq 0$.

The exact flow map in \Cref{eq:hamiltonian_flow_map} is generally not accessible, and we have to approximate it. For Hamiltonian systems, the time-$t$ flow is \emph{symplectic} since the Jacobian matrix $D\phi_{H, t}(x)$ satisfies the identity $[D\phi_{H, t}(x)] ^\top\mathbf J [D\phi_{H, t}(x)] = \mathbf J$. Consequently, the flow $\phi_{H,t}$ preserves the canonical phase space volume \cite{hairerGeometricNumericalIntegration2006}. It is thus desirable that, when approximating $\phi_{H,t}$, such qualitative properties are reproduced by the approximate map. \two{We provide the proof of the symplecticity of $\phi_{H,t}$, together with some additional background material on Hamiltonian systems and symplectic maps, in \Cref{app:background}.}

\subsection{The \texttt{SympFlow} architecture}

\texttt{SympFlow} generalizes the gradient modules of SympNets \cite{jin2020sympnets} to accommodate time dependence. Each of its layers is the exact time-$t$ flow of a suitable time-dependent Hamiltonian system. This design ensures the approximated Hamiltonian flow retains the essential symplectic structure across layers.
\begin{wrapfigure}{R}{0.4\textwidth}
    \centering
\includegraphics[width=\linewidth]{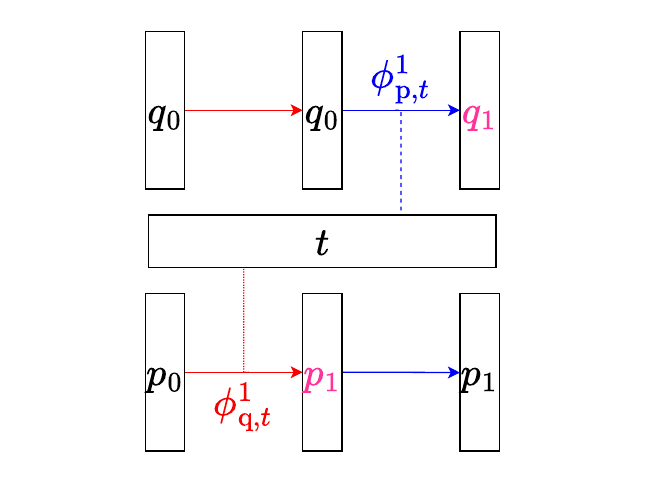}
    \caption{One layer of the \texttt{SympFlow} \two{ where first $\phi^1_{\pos,t}$ updates $p_0$ to $p_1$ and fixes $q_0$, and then the map $\phi_{\mom,t}^1$ fixes $p_1$ and updates $q_0$ to $q_1$, in analogy to splitting methods.}}
    \label{fig:architecture}
\end{wrapfigure}

More specifically, \texttt{SympFlow} is defined by composing exact flow maps of time-dependent Hamiltonians, each of which depends either on position or momentum, but not both. Hence, given an arbitrary \two{twice} continuously differentiable function $V_{\pos}:\R\times \R^d \to \R$, we can consider the flow map (starting from time $0$):
\begin{equation}\phi_{\pos,t}(q,p)=\begin{pmatrix}q \\ p - \left(\nabla_q V_{\pos}\left(t, q\right) - \nabla_q V_{\pos}\left(0, q\right)\right)\end{pmatrix},\label{eq:pos_flow_map}\end{equation}
which corresponds to the Hamiltonian $    H_{\pos, t}(q, p) = \dot V_{\pos}(t, q),$ where $\dot{V}_{\pos}$ stands for $\partial_t V_{\pos}$ and the subscript $\pos$ indicates that the Hamiltonian depends on position, but not momentum. Similarly, for a \two{twice} continuously differentiable function $V_{\mom}:\R\times \R^d \to \R$, we can consider the flow map (starting from time $0$):
\begin{equation}
    \phi_{\mom,t}(q,p)=\begin{pmatrix}q + \left(\nabla_p V_{\mom}\left(t, p\right) - \nabla_p V_{\mom} \left(0, p\right)\right)\\ p\end{pmatrix},\label{eq:mom_flow_map}
\end{equation}
which corresponds to the Hamiltonian $H_{\mom, t}(q, p) = \dot V_{\mom}(t,p)$. As above, the subscript $\mom$ indicates that the Hamiltonian depends on momentum but not position.

Although the Hamiltonians above take a very particular form, they naturally arise when applying splitting integration methods to separable Hamiltonians, \two{see \Cref{app:background}}. This work capitalizes on these methods, which have previously been used to good effect in designing neural networks with desirable structural properties \cite{celledoni2023dynamical}. This connection with separable Hamiltonian systems only provides an intuitive way to interpret the network layers. However, there is no inherent limitation in applying \texttt{SympFlow} to non-separable systems; see \Cref{thm:universal}. By parameterizing $V_{\pos}$ and $V_{\mom}$ as multi-layer perceptrons (MLPs) and composing such steps, we obtain a time-dependent symplectic map, the parameters of which can be optimized to perform different tasks such as fitting data and, more generally, minimizing an objective function. A \texttt{SympFlow} $\bar{\psi}:\R\times \R^{2d}\to\R^{2d}$ with $L\in\mathbb{N}$ layers is a map of the form:
\begin{equation}
\bar{\psi}(t,\cdot) = \phi_{\mom, t}^L \circ \phi_{\pos, t}^L\circ \ldots \circ \phi_{\mom, t}^1 \circ \phi_{\pos, t}^1.\label{eq:sympflow}
\end{equation}
We represent one layer of the architecture in \Cref{fig:architecture}. \two{We note here that individual flows depend on $i$, meaning they are associated with different Hamiltonian systems. Thus, the composition stands for the flow of a new Hamiltonian at time $t$, \emph{not} $L\cdot t$.} We remark that since \texttt{SympFlow} is obtained by composing \emph{exact} flow maps, it satisfies $\bar{\psi}(0,x)=x$ for every point $x\in\R^{2d}$. 

In \Cref{eq:sympflow}, $\phi_{q, t}^i$ denotes the map in \Cref{eq:pos_flow_map} where $V_{\pos}$ is replaced by a $V_{\pos}^i$, and similarly for $\phi_{\mom,t}^i$. We will interchangeably use the notation $\bar{\psi}(t,x)=\bar{\psi}_t(x)$ to denote the action of a \texttt{SympFlow}.

\subsection{The Hamiltonian of the \texttt{SympFlow}}\label{sec:hamSympFlow}
One of the key properties of our architecture is that it is the composition of Hamiltonian flows. This implies that it is also a Hamiltonian flow, as described by the next proposition.
\begin{proposition}[\two{Proposition 2.54 \cite{de2006symplectic}}]
Let $H^1,H^2:\R\times \R^{2d}\to\R$ be \two{twice} continuously differentiable functions, and $\phi_{H^1,t},\phi_{H^2,t}:\R^{2d}\to\R^{2d}$ the exact time-$t$ flows (starting from time $0$) of the Hamiltonian systems they define. Then, the map $\psi_t=\phi_{H^2,t}\circ \phi_{H^1,t}:\R^{2d}\to\R^{2d}$ is the exact time-$t$ flow of the Hamiltonian system defined by the Hamiltonian function
\begin{equation}\label{eq:compositionHam}
H^3(t,x)=H^2(t,x)+H^1\Big(t,\phi_{H^2,t}^{-1}(x)\Big).
\end{equation}
\label{pr:symplectic}
\end{proposition}
\two{The notation $\phi_{H^2,t}^{-1}$ in \Cref{eq:compositionHam} denotes the inverse of the time-$t$ flow of the Hamiltonian system defined by $H^2$, i.e., the inverse of the map $\phi_{H^2,t}:\mathbb{R}^{2d}\to\mathbb{R}^{2d}$ defined as $\phi_{H^2,t}(\cdot)=\phi_{H^2}(t,\cdot)$.} Thanks to \Cref{pr:symplectic}, we can associate a \texttt{SympFlow} with a time-dependent Hamiltonian function. To assemble such a Hamiltonian function, we can group the pairs of alternated momentum and position flows, finding the Hamiltonian associated with $\phi_{\mom,t}^i\circ \phi_{\pos,t}^i$, which is
\[
H^i_t(q,p)=\dot{V}_{\mom}^i(t,p) + \dot{V}_{\pos}^i(t,q - (\nabla_p V_{\mom}^i(t, p) - \nabla_p V_{\mom}^i(0, p))).
\]
The Hamiltonian of the \texttt{SympFlow} in \Cref{eq:sympflow} can then be expressed iteratively, aggregating from the last layer to the first as
\begin{equation}\label{eq:iterativeHam}
H_t^{L:i}(x)=H_t^{L:(i+1)}(x)+H_t^i\left(\phi_{H_t^{L:(i+1)}, t}^{-1}(x)\right)
,\,\,i=1,\ldots,L-1,
\end{equation}
where $H_t^{L:L} = H_t^L$, and
\[
\phi_{H_t^{L:i}, t}^{-1} = \left(\phi_{H_t^{L:(i+1)}, t}\circ \phi_{H_t^{i}, t}\right)^{-1} = \phi_{H_t^{i}, t}^{-1} \circ \phi_{H_t^{L:(i+1)}, t}^{-1}.
\]
To lighten the notation, we introduce the operator $\mathcal{H}$ sending a \texttt{SympFlow} $\bar{\psi}$ into one of its generating Hamiltonian functions $\mathcal{H}(\bar{\psi}):\R\times\R^{2d}\to\R^{2d}$, all of which differ by a function of the time variable $t$\footnote{\two{We remark that two twice-continuously differentiable Hamiltonian functions $H_1,H_2:\mathbb{R}\times \mathbb{R}^{2d}\to\mathbb{R}$ admitting a function $f:\mathbb{R}\to\mathbb{R}$ such that $H_1(t,x)=H_2(t,x)+f(t)$ for every $(t,x)\in\mathbb{R}\times\mathbb{R}^{2d}$ generate the same Hamiltonian vector fields, and hence $\phi_{H_1,t}=\phi_{H_2,t}$, given that $\nabla_x H_1(t,x)=\nabla_x H_2(t,x)$.}}. In this way, the Hamiltonian of the network $\mathcal{H}(\bar{\psi})$ corresponds to $H_t^{L:1}$ defined as in \Cref{eq:iterativeHam}. \two{The symbol $L:i$ denotes that the aggregation has occurred from the $L$th layer to the $i$th.}

\one{We remark that, since the set of time-$t$ flows of autonomous Hamiltonian systems does not form a group (in particular, it is not closed under composition as seen, for example, from \Cref{eq:compositionHam}), we directly define the layers of \texttt{SympFlow} as time-$t$ flows of non-autonomous Hamiltonian systems, despite our aim to approximate the flow of autonomous systems. This choice allows for more expressiveness, at no additional cost.}

\subsection{Training the \texttt{SympFlow}}\label{sec:training}
In \Cref{sec:experiments} we apply \texttt{SympFlow} to approximate the flow map of an autonomous Hamiltonian system of the form $\dot{x}(t)=\mathbf{J}\nabla H(x(t))\in\R^{2d}$ in both \emph{supervised} and \emph{unsupervised} settings. In what follows we describe the relevant training objectives for both of these. In all the experiments below, we assume there exists a forward invariant compact subset $\Omega\subset\R^{2d}$ \one{for the \emph{true} Hamiltonian system}, meaning that if $x_0\in\Omega$ also $\phi_{H,t}(x_0)\in\Omega$ for every $t\geq 0$.  This forward-invariance assumption allows us to make predictions for any time $t\geq 0$, once we are able to make them sufficiently accurate for initial conditions in $\Omega$, and time instants in a compact interval $[0,\Delta t]$. More explicitly, let us consider a function $\bar{\psi}:[0,\Delta t]\times\Omega\to\Omega$ providing an accurate approximation of the exact flow $\phi_{H,t}$ for $t\in [0,\Delta t]$. We can extend $\bar{\psi}$ to $[0,+\infty)$ via $\psi:[0,+\infty)\times \Omega\to\Omega$ defined as 
\begin{equation}\label{eq:longTimeExtension}
\psi(t,x_0):=\bar{\psi}_{t-\Delta t\lfloor t/\Delta t\rfloor}\circ \left(\bar{\psi}_{\Delta t}\right)^{\lfloor t/\Delta t\rfloor}(x_0),
\end{equation}
where we recall that $\bar{\psi}_t(x_0)=\bar{\psi}(t,x_0)$,
\[
\left(\bar{\psi}_{\one{\Delta t}}\right)^{\lfloor t/\Delta t\rfloor} = \underbrace{\bar{\psi}_{\one{\Delta t}}\circ \cdots \circ \bar{\psi}_{\one{\Delta t}}}_{\lfloor t/\Delta t\rfloor\text{ times}}, 
\]
and $\lfloor t/\Delta t\rfloor$ is the largest integer smaller or equal than $t/\Delta t\in [0,+\infty)$. The map $\psi(t,x_0)$ provides an approximation of $\phi_{H,t}(x_0)$ for every $t\geq 0$ and $x_0\in\Omega$. 

\paragraph{Regression loss term}
In the \emph{supervised} setting, we aim to approximate the flow map $\phi_{H,t}:\Omega\to\Omega$ of an unknown Hamiltonian system with Hamiltonian function $H$. For this supervised problem, we suppose to have access to observed trajectories, all collected in the set
\[
\left\{y_m^n=\phi_{H,t_{m}^n}(x^n_0)+\varepsilon_{m}^n:\,\,n=1,...,N,\,m=1,...,M\right\},
\]
where $\varepsilon_m^n$ is a perturbation due to noise or discretization errors, $x_0^n\in\Omega\subset\R^{2d}$, $t_m^n\in [0,\Delta t]$, and \two{$N,M\in\mathbb{N}$ are, respectively, the number of initial conditions in the dataset and the number of time samples for each of their trajectories}. The training process is thus purely based on data, and we minimize the mean squared error
\begin{equation}\label{eq:supervisedLoss}
\mathcal{L}(\bar{\psi}) = \frac{1}{NM}\sum_{n=1}^N\sum_{m=1}^M\left\|\bar{\psi}(t_m^n,x^n_0) - y_m^n\right\|_2^2,
\end{equation}
where $\bar{\psi}:[0,\Delta t]\times\R^{2d}\to\R^{2d}$ can be a \texttt{SympFlow} or a generic Multilayer Perceptron (MLP). Both networks take time as an input, allowing us to deal with non-uniformly sampled trajectory data. An example dataset for the simple harmonic oscillator is shown in \Cref{fig:dataSHO}.

\begin{wrapfigure}{R}{0.25\textwidth}
    \centering
    \includegraphics[width=0.8\linewidth]{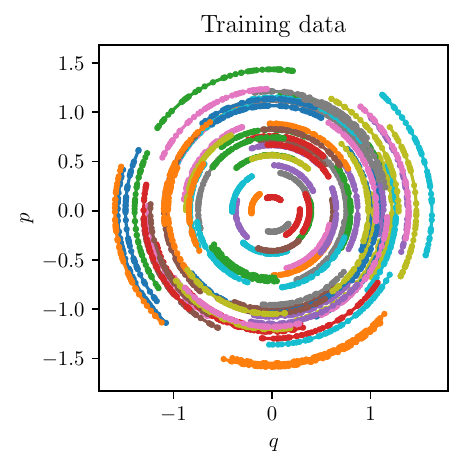}
    \caption{Training data with $N=100$ initial conditions, $M=50$ sampling times, and no noise, i.e., $\varepsilon_m^n=0$.}
    \label{fig:dataSHO}
\end{wrapfigure}
In the \emph{unsupervised} setting, based on the analysis in \Cref{sec:hamSympFlow}, we consider a loss function composed of two terms, namely a residual loss term and a Hamiltonian matching term, to find the weights of a \texttt{SympFlow}. 

\paragraph{Residual loss term}The first term is typical for operator learning tasks, see for example \cite{wang2023long}, and is defined as
\begin{equation}\label{eq:L1}
\mathcal{L}_1(\bar{\psi})=\frac{1}{N}\sum_{\one{n}=1}^N\left\|\left.\frac{\mathrm{d}}{\mathrm{d} t}\bar{\psi}(t,x_0^{\one{n}})\right|_{t=t_{\one{n}}} - \mathbf{J}\nabla H\left(\bar{\psi}(t_{\one{n}},x_0^{\one{n}})\right)\right\|_2^2,
\end{equation}
where $x_0^{\one{n}}\in \Omega$ and $t_{\one{n}}\in [0,\Delta t]$ for every $\one{n}=1,...,N$. For this, $\bar{\psi}:[0,\Delta t]\times\R^{2d}\to\R^{2d}$ can be a \texttt{SympFlow} or an MLP. \two{In \Cref{eq:L1}, $N\in\mathbb{N}$ denotes the number of collocation points we use to assemble the residual loss.
}

\paragraph{Hamiltonian Matching term}The Hamiltonian structure of the \texttt{SympFlow} gives us a natural way to regularize the training phase. To that end, we use a Hamiltonian matching approach \cite{canizares2024hamiltonian} that introduces an additional term in the loss function, defined as
\begin{equation}\label{eq:L2}
\mathcal{L}_2(\bar{\psi})=\frac{1}{M}\sum_{\one{m}=1}^M \left(\mathcal{H}(\bar{\psi})(t_{\one{m}},x^{\one{m}}_0) - H(x^{\one{m}}_0) \right)^2,
\end{equation}
where $x^{\one{m}}_0\in \Omega$ and $t_{\one{m}}\in [0,\Delta t]$ for every $\one{m}=1,...,\one{M}$. 

In the experimental analysis, we will compare \texttt{SympFlow} to an unconstrained MLP. To have an unbiased comparison, the MLP is trained with the residual loss function to which we add the alternative energy regularization term
\begin{equation}\label{eq:mlpReg}
\widetilde{\mathcal{L}}_2(\bar{\psi}) = \frac{1}{M}\sum_{\one{m}=1}^M\left(H(\bar{\psi}(t_{\one{m}},x^{\one{m}}_0)) - H(x^{\one{m}}_0)\right)^2,
\end{equation}
where $t_{\one{m}}\in [0,\Delta t]$, $x^{\one{m}}_0\in\Omega$, and $\bar{\psi}:[0,\Delta t]\times \Omega\to\Omega$ is the MLP. \two{In \Cref{eq:L2} and \Cref{eq:mlpReg}, $M\in\mathbb{N}$ denotes the number of collocation points we use to assemble the regularization term. Furthermore, $N$ in \Cref{eq:L1}, and $M$ in \Cref{eq:L2} and \Cref{eq:mlpReg}, are independent since one could use different points for the two components of the loss.} This regularization term informs this unconstrained network that the Hamiltonian energy should be conserved, \one{similarly to what is done in \cite{mattheakis2022hamiltonian}}. We remark that \Cref{eq:mlpReg} differs from \Cref{eq:L2} in that we do not have a modified Hamiltonian $\mathcal{H}(\bar{\psi})$ in the MLP case, and hence we only promote the conservation of the actual Hamiltonian energy $H$. 

The loss function is then given by
\begin{equation}\label{eq:fullLoss}
\mathcal{L}(\bar{\psi})=\mathcal{L}_1(\bar{\psi})+\one{\gamma}\mathcal{L}_2(\bar{\psi}),\,\,\one{\gamma\in\{0,1\}},
\end{equation}
with $\mathcal{L}_1$ and $\mathcal{L}_2$ defined as in \Cref{eq:L1} and \Cref{eq:L2}. For the MLP, we replace $\mathcal{L}_2$ with $\widetilde{\mathcal{L}}_2$. \one{In the experiments, we compare the case without regularization, $\gamma=0$, with the one with regularization, $\gamma=1$. Such a regularization significantly improves the results for the MLP. This does not occur for \texttt{SympFlow}, which performs well in the unregularized case as well{, which, sometimes, also outperforms the regularized network.}} Notice that $\mathcal{H}(\bar{\psi})$ is the exact Hamiltonian behind the \texttt{SympFlow} $\bar{\psi}$, and the loss $\mathcal{L}_2(\bar{\psi})$ provides a comparison between the exact flow $\phi_{H,t}$ and the \texttt{SympFlow} via zeroth-order information, i.e., based on the exact energies behind those maps. Importantly, the $\mathcal{L}_2$ term allows us to perform a \emph{backward error analysis}\textemdash while \texttt{SympFlow} does not generally solve the ODE under consideration, it solves a time-dependent Hamiltonian ODE, for a Hamiltonian which is driven towards $H$ during the training process. The fact that $\mathcal{H}(\bar{\psi})$ is time-dependent implies that the flow $\bar{\psi}$ does not conserve $\mathcal{H}$, as it is expected for non-autonomous Hamiltonian systems. However, if \Cref{eq:L2} is small enough, the Hamiltonian $\mathcal{H}(\bar{\psi})$ will not be strongly dependent on time, hence leading to an almost conservation of $H$ by the \texttt{SympFlow}. We further expand on this aspect below.

\section{Theoretical analysis of the SympFlow}\label{sec:theory}

This section provides a theoretical a-posteriori analysis of a \texttt{SympFlow} approximating the target map $\phi_{H,t}$. This result relies on the assumption that the map $\phi_{H,t}$ can be accurately approximated for $t\in [0,\Delta t]$ by a \texttt{SympFlow}. This is  possible since \texttt{SympFlows} are universal approximators in the space of Hamiltonian flows. We state this result in \Cref{thm:universal}, and prove it in \Cref{app:universal}.
\begin{thm}[Universal approximation theorem for \texttt{SympFlow}]\label{thm:universal}
Let $H:\R\times\R^{2d}\to\R$ be a \two{twice-}continuously differentiable function, and fix $\Omega\subset\R^{2d}$ compact. \one{Suppose that $\Omega$ is forward invariant for $\phi_{H,t}$.} For any $\varepsilon>0$, there is a \texttt{SympFlow} $\bar{\psi}_t$ such that
\[
\sup_{\substack{t\in [0,\Delta t] \\ x\in\Omega}}\left\|\bar{\psi}_t(x) - \phi_{H,t}(x)\right\|_{\infty}<\varepsilon.
\]
\end{thm}
\begin{proof}[Outline of the proof of \Cref{thm:universal}]
The proof is based on the following steps:
\begin{enumerate}[align=left,leftmargin=0.5cm]
\item Approximate the flow $\phi_{H,t}$ with the flow $\phi_{\widetilde{H},t}$ of a polynomial Hamiltonian system with Hamiltonian $\widetilde{H}:\R\times\Omega\to \R$.
\item Split the exact polynomial flow $\phi_{\widetilde{H},t}$ into sufficiently many $N$ substeps of size $t/N$.
\item Approximate to $\mathcal{O}(1/N^2)$ each of the flows defining the $N$ substeps with flows of separable Hamiltonian systems.
\item Approximate the flows of separable Hamiltonian systems of the previous point with \texttt{SympFlows}.
\item Combine the approximations and use the fact that the composition of \texttt{SympFlows} is again a \texttt{SympFlow}.
\end{enumerate}
The details can be found in \Cref{app:universal}.
\end{proof}

Based on the time extension provided in \Cref{eq:longTimeExtension} and \Cref{remark:invariance}, one can obtain a function which accurately approximates a target flow map $\phi_{H,t}:\Omega\to\Omega$ for every time $t\geq 0$. This time extension is again the composition of Hamiltonian flows, hence possessing an underlying time-dependent Hamiltonian. For a fixed time, $t\geq 0$, the function in \Cref{eq:longTimeExtension} satisfies
\begin{align}\label{eq:longTimeHamiltonian}
&\frac{\mathrm{d}}{\mathrm{d} t}\psi(t,x_0) = \frac{\mathrm{d}}{\mathrm{d} t}\left(\bar{\psi}_{t-\Delta t\lfloor t/\Delta t\rfloor}\circ \left(\bar{\psi}_{\Delta t}\right)^{\lfloor t/\Delta t\rfloor}\right)(x_0)\\
&= \mathbf{J}\nabla \left(\mathcal{H}(\bar{\psi})\right)\left(t-\Delta t\lfloor t/\Delta t\rfloor, \psi(t,x_0)\right)\nonumber,
\end{align}
almost everywhere. We provide details about the derivation of \Cref{eq:longTimeHamiltonian} in \Cref{app:longTimeHam}. From \Cref{eq:longTimeHamiltonian}, we see that the long-time extension of a \texttt{SympFlow} is again the solution of a non-autonomous Hamiltonian system with time-dependent Hamiltonian function $\mathcal{H}(\psi):[0,+\infty)\times\R^{2d}\to\R^{2d}$ defined as 
\begin{equation}\label{eq:hamExtension}
\mathcal{H}(\psi)(t,x)=\mathcal{H}(\bar{\psi})(t-\Delta t\lfloor t/\Delta t\rfloor, x). 
\end{equation}
Therefore, the Hamiltonian in \Cref{eq:hamExtension} is piecewise continuous, and this is because $\psi$ is not differentiable at the time instants $t_k=k\Delta t$, $k\in\mathbb{N}$.
 
\begin{thm}[A-posteriori error estimate]\label{thm:aposteriori}
\two{Let $H:\mathbb{R}^{2d}\to\mathbb{R}$ be a twice-continuously differentiable function,} $\Delta t>0$, $\Omega\subset\mathbb{R}^{2d}$ compact, and $\bar{\psi}:[0,\Delta t]\times \Omega\to\Omega$ be a \texttt{SympFlow}. \one{Suppose that $\Omega$ is forward invariant for both $\phi_{H,t}$ and $\bar{\psi}_t$, i.e., for every $x_0\in\Omega$ and $t\in [0,\Delta t]$ one has $\phi_{H,t}(x_0),\bar{\psi}_t(x_0)\in\Omega$.} Assume that for every $x\in\Omega$ and $t\in [0,\Delta t]$ 
\begin{equation}\label{eq:hamCondition}
\left|\mathcal{H}(\bar{\psi})(t,x)-H(x)\right|\leq \varepsilon_1,
\end{equation}
and also
\begin{equation}\label{eq:residualCondition}
\left\|\frac{\mathrm{d}}{\mathrm{d} t}\bar{\psi}_t(x)-\mathbf{J}\nabla H(\bar{\psi}_t(x))\right\|_2\leq \varepsilon_2
\end{equation}
for a pair of values $\varepsilon_1,\varepsilon_2>0$. Then there exist $c_1,c_2>0$ such that for every $t\geq 0$ and $x\in\Omega$
\begin{align}
\left|\mathcal{H}({\psi})(0,x)-\mathcal{H}(\psi)(t,\psi_t(x))\right |&\leq \two{c_1(\varepsilon_1+\varepsilon_2t)},\label{eq:boundModifiedHam}\\
\left |H(x)-H(\psi_t(x))\right |&\leq \two{c_2\varepsilon_2t},\label{eq:boundTrueHam}
\end{align}
where $\psi$ is defined as in \Cref{eq:longTimeExtension}, and $\mathcal{H}(\psi)$ as in \Cref{eq:hamExtension}.
\end{thm}
\begin{remark}\label{remark:invariance}
\one{Forcing $\Omega$ to be forward invariant for $\bar{\psi}_t$ before training is extremely challenging. Still, since this is an a-posteriori analysis, it assumes that the network has already been trained to match the exact flow, and almost-forward invariance is what we observe experimentally. Such a property is a consequence of a successful training procedure coupled with the symplectic nature of \texttt{SympFlow}. Indeed, its symplecticity experimentally leads to almost energy conservation and, consequently, to almost forward invariance of $\Omega$.}
\end{remark}
We prove \Cref{thm:aposteriori} in \Cref{app:aposteriori}. We remark that \Cref{eq:boundTrueHam} can be obtained with similar techniques for any neural network satisfying \Cref{eq:residualCondition}, while \Cref{eq:boundModifiedHam} only in the case one has a \texttt{SympFlow} and some a-posteriori bound as in \Cref{eq:hamCondition}. Furthermore, we also point out that the assumptions in \Cref{eq:hamCondition} and \Cref{eq:residualCondition} correspond to saying that the loss function in \Cref{eq:fullLoss} is smaller than a specific constant, hence why this is an a-posteriori error estimate.

\section{Numerical Experiments}\label{sec:experiments}
In this section, we demonstrate the effectiveness of our proposed architecture, \texttt{SympFlow}, in two tasks: (i) solving the equations of motion of a given time-independent Hamiltonian system, and (ii) approximating the solution map of an unknown Hamiltonian system based on data. 

The proposed experimental analysis compares the results obtained with \texttt{SympFlow} and an unconstrained neural network, which we will refer to as MLP. We provide some details on the network architectures in \Cref{app:architectures}.

We consider three test problems: the simple harmonic oscillator, the damped harmonic oscillator, and the H\'enon--Heiles system. Details on these Hamiltonian systems are provided in \Cref{sec:SHO}, \Cref{sec:DHO}, and \Cref{sec:HH}. For each test problem, we first consider the unsupervised task of solving the equations of motion. We then move on to consider the supervised task of estimating the flow based on irregular trajectory samples. The training process for these experiments follows the steps presented in \Cref{sec:training}. A PyTorch implementation of our architecture and these experiments can be found at the repository \href{https://github.com/davidemurari/sympflow}{https://github.com/davidemurari/sympflow}.

For the three systems, we show the effectiveness of the proposed methodology by presenting quantitative and qualitative comparisons between the various models. For the simple harmonic oscillator, we also analyze the impact of noise, the value of $N$, and the value of $M$ on the approximation accuracy. In this case, we will model the noise affecting trajectories with random normal variables of zero mean and standard deviation $\varepsilon$, where $\varepsilon$ is used to quantify the noise intensity. Our trajectory data is synthetically generated with a Runge--Kutta $(5,4)$ integrator with tight tolerance. Consequently, the trajectories are also affected by discretization error.

As we will see across all the simulations, \texttt{SympFlow} leads to considerably improved long-time energy behavior compared to MLP. \one{Furthermore, the Hamiltonian regularization considerably benefits the MLP for the more challenging problems such as Hénon-Heiles, but does not improve the already good performance of the unregularized \texttt{SympFlow}.}

\two{We provide a thorough evaluation of \texttt{SympFlow} with additional experiments in the appendix. In \Cref{app:timing}, we compare the computational time of the training and inference phases for \texttt{SympFlow} with that of the MLP. We test networks with a varying number of layers $L\in \{4,8,16\}$, report the loss curves, and also compare the quality of the obtained approximation. Our experiments in this section of the main body of the paper focus on a comparison of \texttt{SympFlow} with an MLP, which is the strategy proposed in \cite{mattheakis2022hamiltonian} to approximate $\phi_{H,t}$. We also compare our \texttt{SympFlow} with the Symplectic Recurrent Neural Network procedure presented in \cite{Chen2020Symplectic}. These experiments are collected in \Cref{app:srnn} and showcase \texttt{SympFlow}'s robustness to noise.}

\subsection{One dimensional Simple Harmonic Oscillator}\label{sec:SHO}

The one-dimensional simple harmonic oscillator is a foundational problem in both physics and engineering. Studying its dynamics offers valuable insights into the behavior of classical mechanical systems as well as quantum systems. The simple harmonic oscillator equations of motion are analytically solvable, making it an ideal benchmark for evaluating the accuracy and performance of both classical and neural network-based solvers.

Without loss of generality, we consider a spring with a spring constant $k$, where one end is attached to a point mass $m$ and the other end is fixed in place. The Hamiltonian of this system is given by
\begin{align}\label{eq:simpleHOHam}
H(q,p) = \frac{1}{2m} p^2 + \frac{k}{2}q^2,
\end{align}
where $q,p$ are the position and momentum of the point mass. In our experiments, we fix the recovery constant to $k=1$ and the mass to $m=1$. The equations of motion of this system write
\begin{equation}\label{eq:simpleHOEq}
\dot{q} = \frac{\partial H}{\partial p } = \frac{p}{m}\;, \ \ \ 
\dot{p} = -\frac{\partial H}{\partial q } = -kq\;.
\end{equation}
We now move to the two experimental settings we consider for this system. All the experiments fix $\Delta t=1$ and $\Omega = [-1.2,1.2]^2\subset\R^2$.

\subsubsection{Unsupervised experiments}
In this section, we compare the \texttt{SympFlow} architecture with an unconstrained MLP on solving the equations of motion in \Cref{eq:simpleHOEq}. We train the \texttt{SympFlow} in three different ways, aiming to identify the impact of the two terms in the loss function, see \Cref{eq:fullLoss}. The MLP is instead either trained with or without the regularization term in \Cref{eq:mlpReg}\one{, i.e., with a residual term alone or coupled with Hamiltonian regularization}. When a network is trained with Hamiltonian regularization, it means it is trained to minimize the loss in \Cref{eq:fullLoss} \one{with $\gamma=1$}. If there is no such regularization, it means that we only train with the residual-based loss function in \Cref{eq:L1}, \one{i.e., $\gamma=0$ in \Cref{eq:fullLoss}}. \one{In \Cref{app:mixed}, } we also consider a mixed training procedure consisting of a first training phase with Hamiltonian regularization, and a later fine-tuning of the weights by training the model for a few more epochs without Hamiltonian matching. \Cref{fig:simpleHO} reports the results obtained with these four training configurations. We see that the \texttt{SympFlow} architecture consistently outperforms the unconstrained MLP in predicting the correct qualitative behavior of the solution. This is evident both looking at the reproduced orbit, which are all associated to the initial condition $(q_0,p_0)=(1,0)$, but also in the long-time energy behavior. We remark that the orbit of $(1,0)$ is wholly contained in $\Omega$, so the networks are expected to approximate it reliably. The second column in \Cref{fig:simpleHO} shows the variation of the true Hamiltonian energy in \Cref{eq:simpleHOHam} along the network-produced solutions. Even though we extrapolate over the long integration time $[0,1000]$, the \texttt{SympFlow} outperforms the MLP regardless of how it has been trained. This is a consequence of the symplectic nature of the network. In this experiment, it is unclear whether the different training procedures considered have a significant impact. We explore this further in the other two test problems. 
\begin{figure}[ht!]
    \centering
    \includegraphics[scale=0.5]{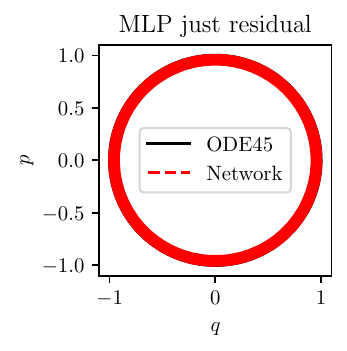}
    \includegraphics[scale=0.5]{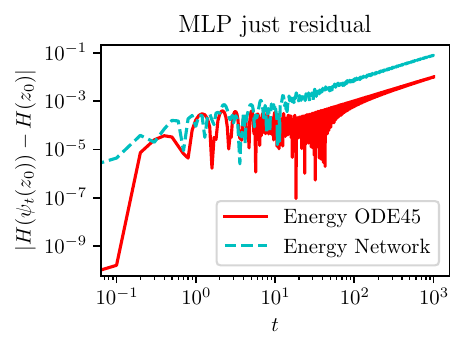}
    \includegraphics[scale=0.5]{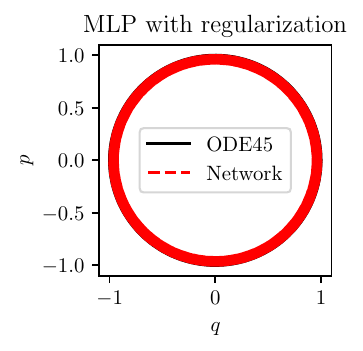}
    \includegraphics[scale=0.5]{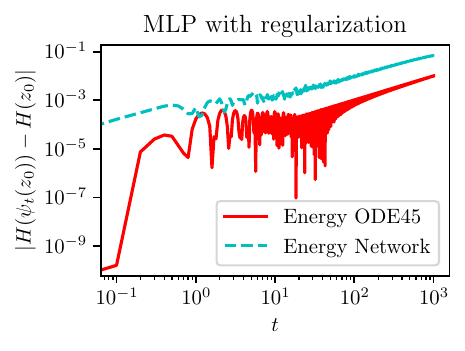}
    
    \includegraphics[scale=0.5]{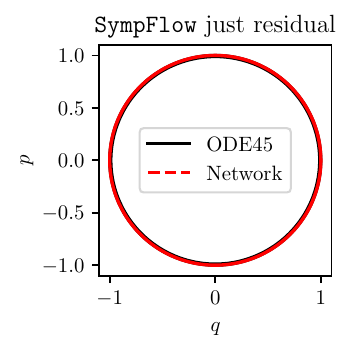}
    \includegraphics[scale=0.5]{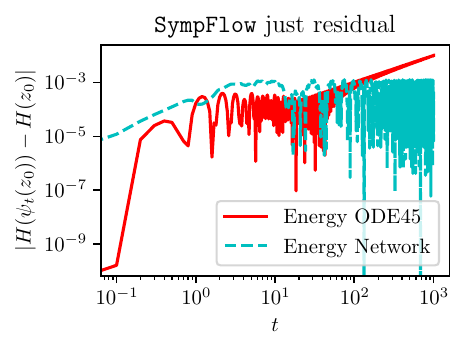}
    \includegraphics[scale=0.5]{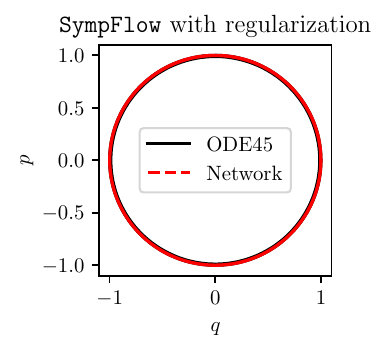}
    \includegraphics[scale=0.5]{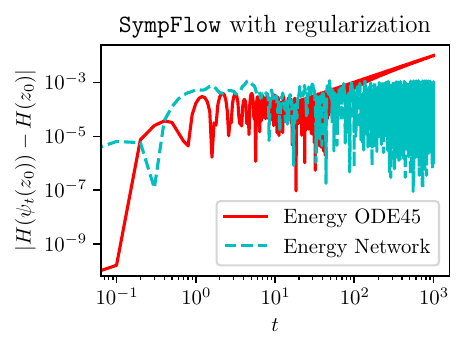}

    \caption{\textbf{Unsupervised experiment --- Simple Harmonic Oscillator: }Comparison of the results obtained with predictions up to time $T=1000$ and $\Delta t=1$.}
    \label{fig:simpleHO}
\end{figure}

\subsubsection{Supervised experiments}

We now move to the supervised experiments, where we recall that the models are trained based on the mean squared error loss function in \Cref{eq:supervisedLoss}. We first show in \Cref{fig:simpleHO-supervised} the obtained results with $N=100$ initial conditions, each sampled at possibly different $M=50$ time instants. In this case, we assume there is no noise, i.e., $\varepsilon=0$. The results in \Cref{fig:simpleHO-supervised} show a similar pattern to the unsupervised experiments, where the \texttt{SympFlow} outperforms the MLP over long time simulations. Comparing this figure with \Cref{fig:simpleHO}, we notice that the unsupervised case leads to a smaller energy variation than the supervised one for \texttt{SympFlow} predictions. This difference is expected since the unsupervised experiment, even though it does not rely on trajectory data, relies on the knowledge of the Hamiltonian system one is trying to solve, hence having access to the true Hamiltonian function while training the network.
\begin{figure}[ht!]
    \centering
    \includegraphics[scale=.5]{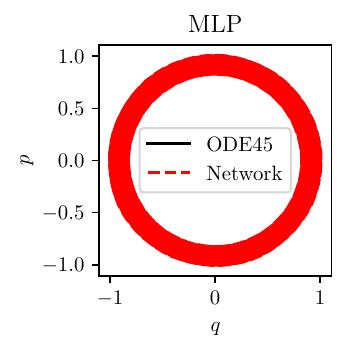}
    \includegraphics[scale=.5]{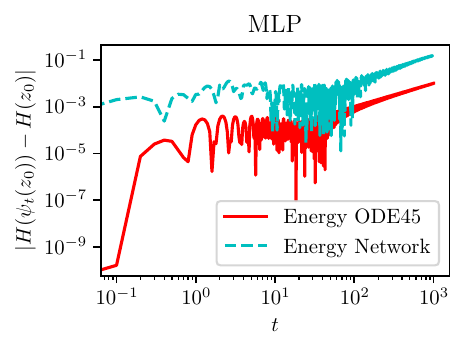}
    \includegraphics[scale=.5]{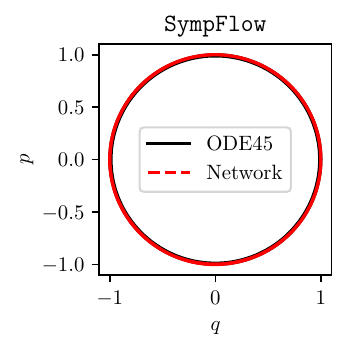}
    \includegraphics[scale=.5]{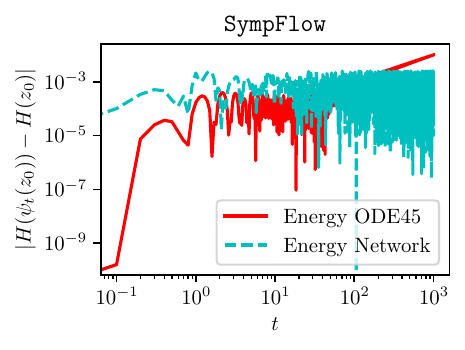}
    \caption{\textbf{Supervised experiment --- Simple Harmonic Oscillator: }Comparison of the MLP and the SympFlow trained with a dataset of parameters $N=100$, $M=50$, and $\varepsilon=0$.}
    \label{fig:simpleHO-supervised}
\end{figure}

We now evaluate the impact of the three parameters $N$, $M$, and $\varepsilon$ on the performance of \texttt{SympFlow} and MLP. We denote with $\psi:\R\times \R^{2d}\to\R^{2d}$ the extension of the network $\bar{\psi}$ over the real line, defined as in \Cref{eq:longTimeExtension}. To quantitatively evaluate the results and compare the models, we consider the average relative error
\begin{equation}\label{eq:averageSolError}
\frac{1}{I}\sum_{i=1}^{I}\frac{\left\|\psi(k\Delta t,x_i^0)-\phi_{H,k\Delta t}(x_i^0)\right\|_2}{\left\|\phi_{H,k\Delta t}(x_i^0)\right\|_2},\,\,k=1,10,100
\end{equation}
over $I=100$ randomly sampled initial conditions in $\Omega$. We remark that even though the quantity above is expressed using the exact flow map $\phi_{H,t}$, in practice, we replace it with a reference numerical solution obtained using a Runge--Kutta $(5,4)$ integrator. We also test the average relative Hamiltonian energy variation
\begin{equation}\label{eq:averageHamError}
\frac{1}{I}\sum_{i=1}^{I}\frac{\left|H(\psi(k\Delta t,x_i^0))-H(x_i^0)\right |}{\left|H(x_i^0)\right|},\,\,k=1,10,100.
\end{equation}
We collect the values of these two metrics in \Cref{fig:relative_results_combined}. Each of the three subfigures lets one of the three parameters vary and fixes the other two. We now discuss some outcomes from this experimental analysis:
\begin{itemize}
\item The energy variation for \texttt{SympFlow} is always smaller than for MLP. Furthermore, this gap widens as time progresses.
\item The relative error for the two models is comparable in many configurations, with \texttt{SympFlow} generally achieving lower errors. This is expected for two reasons:
\begin{enumerate}
    \item The simple harmonic oscillator does not have complicated dynamics, so we get accurate results with both models. We will later demonstrate how the more complex dynamics of the Hénon--Heiles system posed a significant challenge for the MLP, further highlighting the strengths of \texttt{SympFlow}.
    \item The fact that \texttt{SympFlow} is symplectic does not ensure improved quantitative (long-time) behavior, but qualitative. This is also why one would not expect a symplectic time integrator to be more accurate than a classical one. However, as for the energy behavior above, the qualitative properties of the produced solutions are considerably improved.
    \end{enumerate}
\item Even in the presence of noise, the \texttt{SympFlow} outperforms MLP. This can be seen by looking at the error at time $\Delta t$ for different values of $\varepsilon$ in \Cref{fig:relative_varying_epsilon}. The symplectic constraint on \texttt{SympFlow} prevents it from overfitting the data and learning the noise.
\item Both models benefit from additional data as $N$ and $M$ increase, with the MLP showing a stronger dependence on data in several settings. This can be interpreted as a consequence of the fact that the MLP has no encoded structure, and to infer the correct behavior, it needs a considerable amount of data.
\end{itemize}
\begin{figure}[ht!]
    \centering
    \begin{subfigure}{\textwidth}
        \centering
        \includegraphics[width=\textwidth]{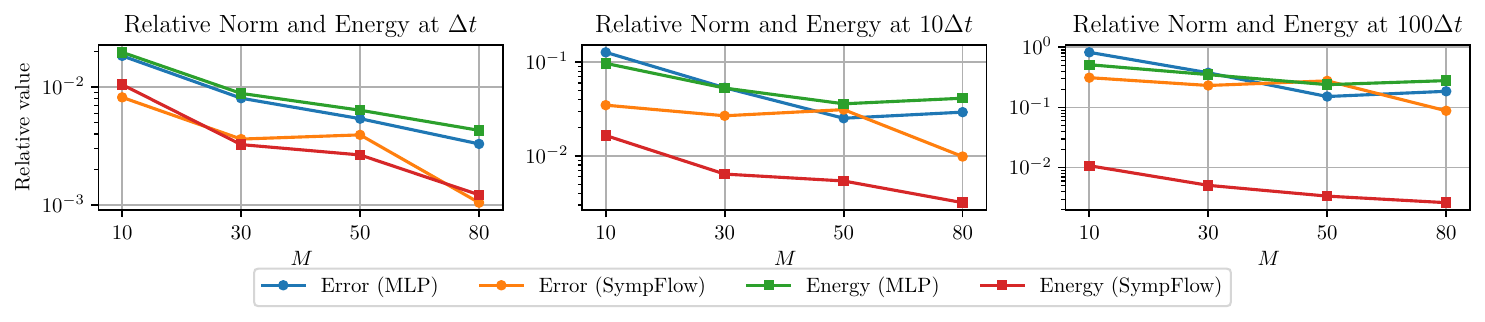}
        \caption{Fixed $N=100$ and $\varepsilon=0$.}
        \label{fig:relative_varying_m}
    \end{subfigure}

    \begin{subfigure}{\textwidth}
        \centering
        \includegraphics[width=\textwidth]{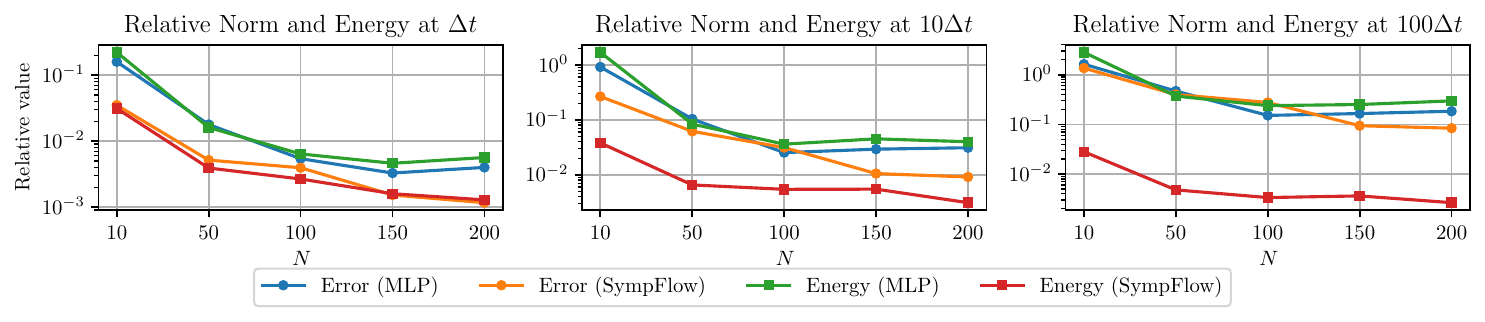}
        \caption{Fixed $M=50$ and $\varepsilon=0$.}
        \label{fig:relative_varying_n}
    \end{subfigure}

    \begin{subfigure}{\textwidth}
        \centering
        \includegraphics[width=\textwidth]{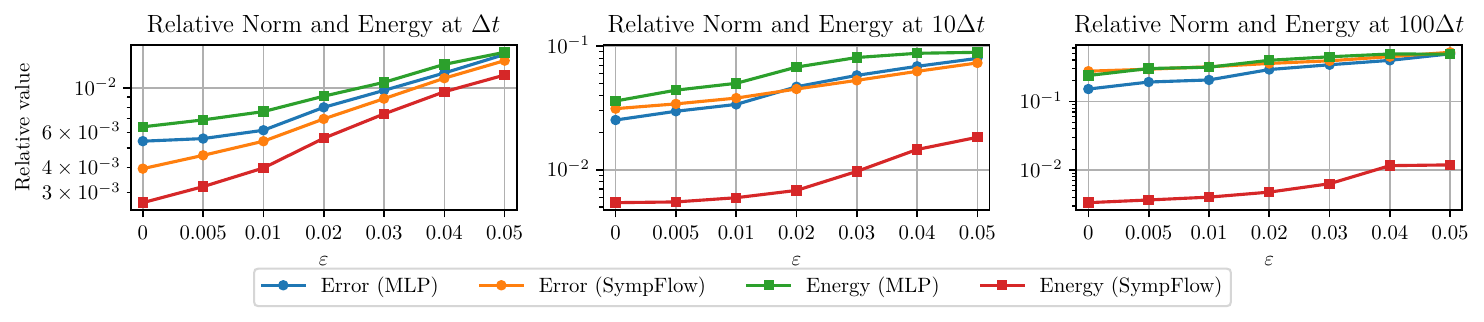}
        \caption{Fixed $N=100$ and $M=50$.}
        \label{fig:relative_varying_epsilon}
    \end{subfigure}

    \caption{\textbf{Supervised experiment --- Simple Harmonic Oscillator: }Relative $\ell^2-$norms of the error and the energy variations at times $\Delta t,10\Delta t,100\Delta t$.}
    \label{fig:relative_results_combined}
\end{figure}

\subsection{Modelling Dissipation: Damped Harmonic Oscillator}\label{sec:DHO}
Real-world dynamical systems often dissipate energy to their surroundings through mechanisms such as heat, friction, or radiative losses. These processes reduce the energy of the system until it reaches equilibrium. To accurately capture the dynamics of dissipative physical systems, it is critical to account for those effects within their equations of motion. To apply our approach to non-conservative systems, we adopt the formulation introduced in \cite{Galley_2013}, which allows to express their dynamics as conservative systems in a phase space having doubled dimension compared to the physical space. This formalism enables to accommodate dissipation while retaining key features of the Hamiltonian dynamics, such as symplecticity, and hence apply \texttt{SympFlow} in a meaningful way. For more details on this formalism, see \Cref{app:NCA}. In this formulation, the degrees of freedom are doubled, and a point in this augmented space is given by $z_{a,b}:=(q_a,q_b,\pi_a,\pi_b)\in\mathbb{R}^{4d}$, where $\pi_a$ and $\pi_b$ are the so-called non-conservative momenta. The Hamiltonian of the augmented system is given by the function
\be
A(q_a,q_b,\pi_a,\pi_b)= H(q_a,\pi_a) - H(q_b,\pi_b)- K(q_a,q_b,\pi_a,\pi_b)\;,
\ee
where $H:\R^{2d}\to\R$ is the Hamiltonian of the system, and $K:\R^{4d}\to\R$ is an interaction term modeling the action of the non-conservative forces. In our experiments, the function $K$ takes the form \cite{Tsang_2015, Galley_2014} $K(q_a,q_b,\pi_a,\pi_b) = -\lambda/2(\dot{q}_{a}+\dot{q}_{b})(q_{a}-q_{b})$, 
giving rise to a damped harmonic oscillator. Although a trajectory described by the original phase-space variables, $(q(t),p(t))$, does not follow a conservative dynamics, the trajectory described in the augmented space, is indeed conservative. Hence, in the augmented space, the map $(t,z_{a,b}(0))\mapsto z_{a,b}(t)$ is symplectic and the equations of motion can be recast in a symplectic form:  
\begin{equation}
\dot{z}_{a,b}(t) = {\bf J}\nabla A(z_{a,b}(t))\,, 
\end{equation}
with the corresponding equations of motion given by:
\begin{align}\label{heq}
	 & \dot{q} =   \frac{\partial H}{\partial p}  - \bigg[ \frac{\partial K}{\partial \dot{q}_{I}}   \bigg]_{\text{PL}}           \\
	 & \dot{p} = - \frac{\partial H}{\partial q }  + \bigg[ \frac{\partial K}{\partial \dot{q}_{I}}  \bigg]_{\text{PL}}\;,  \nonumber
\end{align}
where PL stands for physical limit and corresponds to the restriction to the linear subspace
\[
\{(q_a,q_b,\pi_a,\pi_b)\in\R^{4d}:\,\,q_a=q_b,\text{ and }\pi_a=-\pi_b\},
\]
see \Cref{app:NCA}.

\subsubsection{Unsupervised experiments}
Similarly to the harmonic oscillator example in \Cref{sec:SHO}, we solve the system's equations of motion with the \texttt{SympFlow} architecture and with an unconstrained MLP. \one{The MLP is not regularized in this situation, since the physical energy is not preserved for this system.} In this section, we experiment training 
\texttt{SympFlow} with: i) Hamiltonian regularization \Cref{eq:fullLoss} with $\gamma=1$ and b) residual-based loss function \Cref{eq:L1}. We train both networks for $50,000$ epochs, setting $\Delta t=1$. \Cref{fig:unsup_solutions} shows the damped harmonic oscillator solutions for both MLP and \texttt{SympFlow}, and for increasing values of the damping constant $\lambda$. We can see that both networks are able to capture the dissipative dynamics of the system for the different values of $\lambda$. This result is further supported by the left subfigure of \Cref{fig: sympflow_pinn_comparison}, where we also notice that adding an energy regularization term does not seem to improve \texttt{SympFlow}'s overall performance.

%% SOLUTIONS DHO
\begin{figure}[ht!]
\centering
\begin{subfigure}{.3\textwidth}
\centering
\includegraphics[width=.9\linewidth]{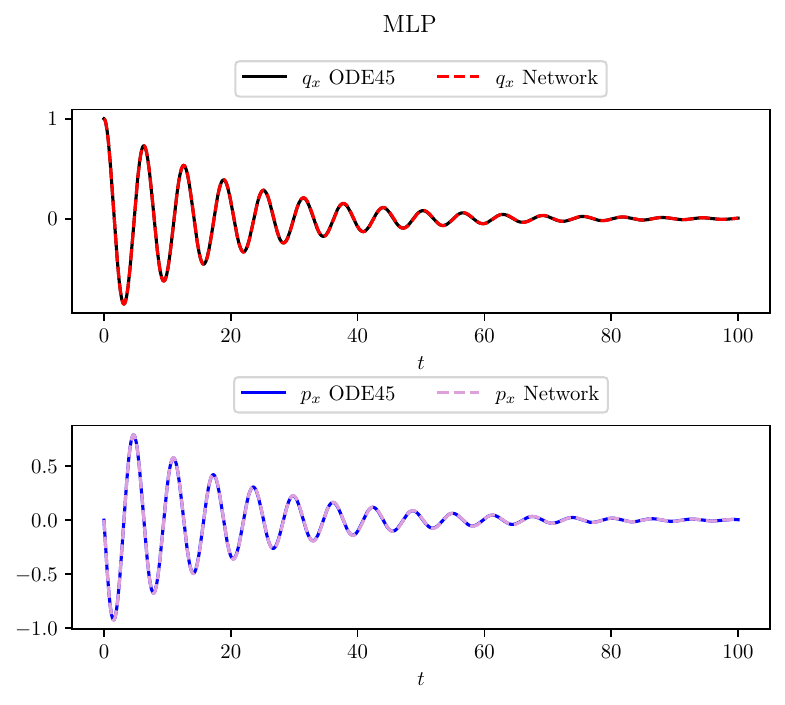} \\
\includegraphics[width=.9\linewidth]{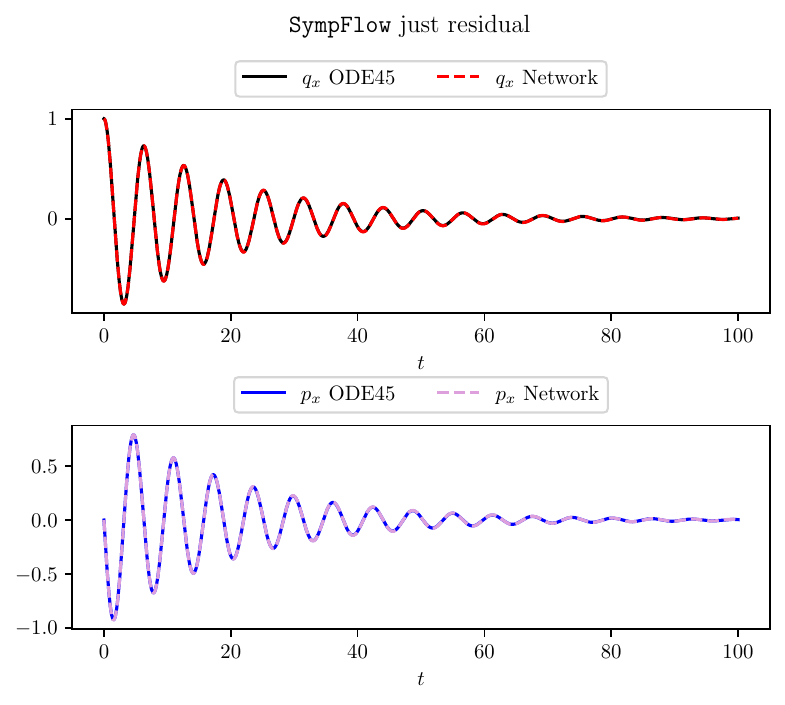}
\end{subfigure}
\begin{subfigure}{.3\textwidth}
\centering
\includegraphics[width=.9\linewidth]{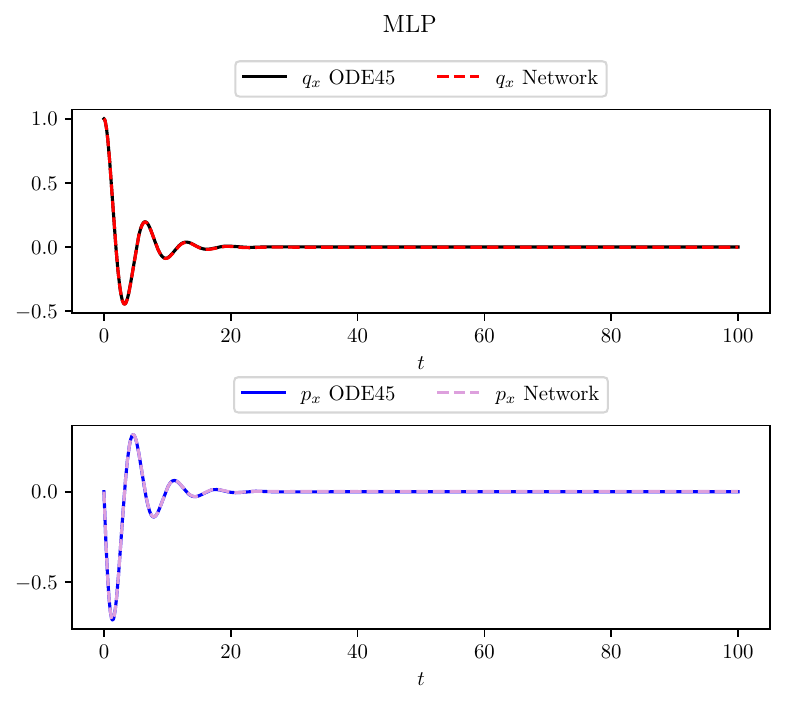}  \\
\includegraphics[width=.9\linewidth]{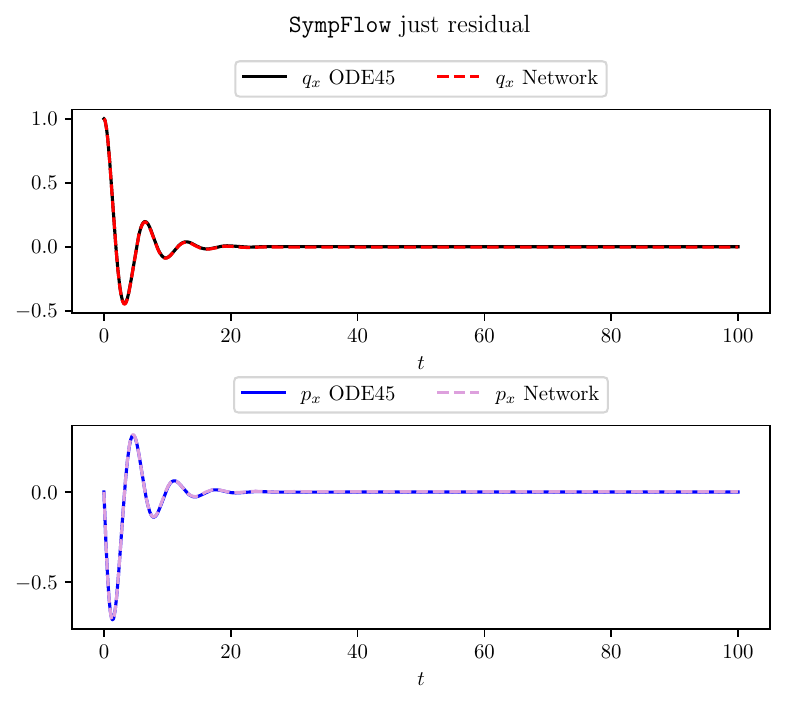}
\end{subfigure}
\begin{subfigure}{.3\textwidth}
\centering
\includegraphics[width=.9\linewidth]
{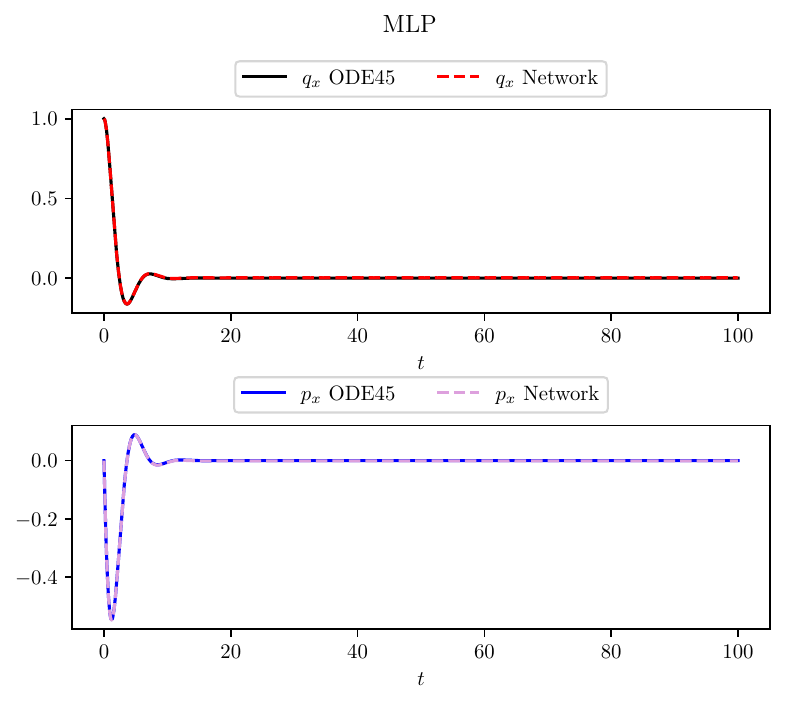}  \\
\includegraphics[width=.9\linewidth]{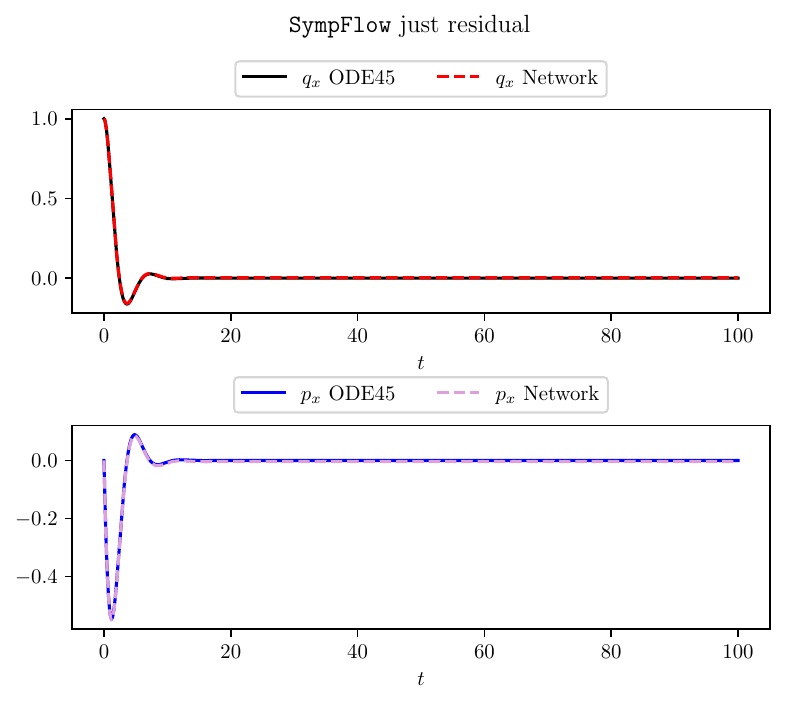}
\end{subfigure}
\caption{Damped harmonic oscillator solutions, where $q_x$ is the $1D$ position and $p_x$ is the (conservative) $1D$ momentum of the oscillator's mass for different values of the damping constant $\lambda$. {\bf Left:} $\lambda = 0.1$, {\bf center:} $\lambda = 0.5$,  and {\bf right:} $\lambda = 1$.}
\label{fig:unsup_solutions}
\end{figure}

\subsubsection{Supervised experiments}
In this section we compare the performance of the MLP network and \texttt{SympFlow}, trained with $N=50$ initial conditions and $M=10$ time samples for each of them. \Cref{fig: MLP_supervised}, \Cref{fig: sympflow_supervised}, and \Cref{fig: sympflow_pinn_comparison} show that \texttt{SympFlow} outperforms the MLP network when the dynamics of the system is more complex (case $\lambda=0.1$) and presents oscillations with decreasing amplitude, whereas for solutions with less complex behavior both methods seem to perform similarly well. 

%%% DHO PINN LOW N AND M
\begin{figure}[ht!]
\centering
\begin{subfigure}{.78\textwidth}
\centering
\includegraphics[width=.45\linewidth]{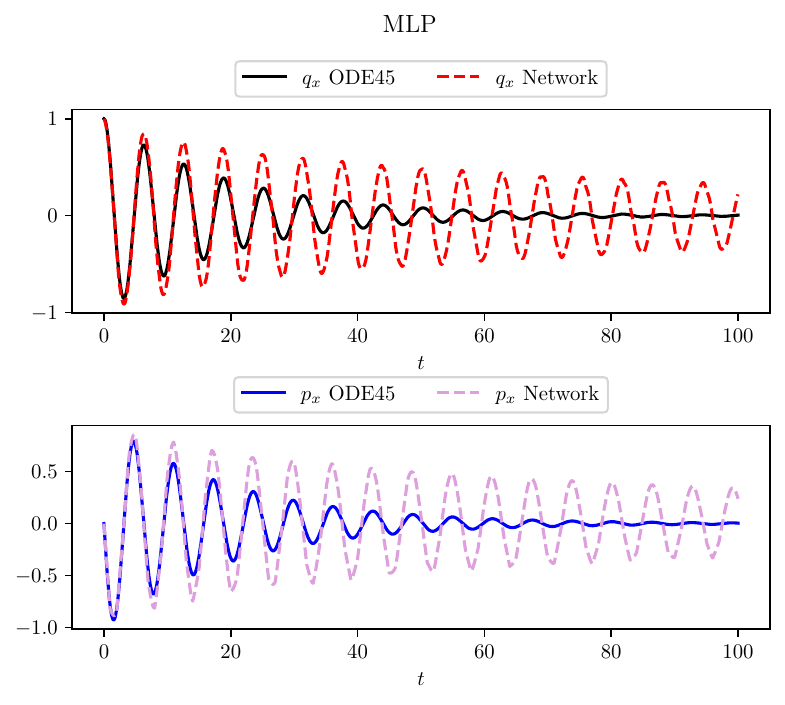}
\includegraphics[width=.45\linewidth]{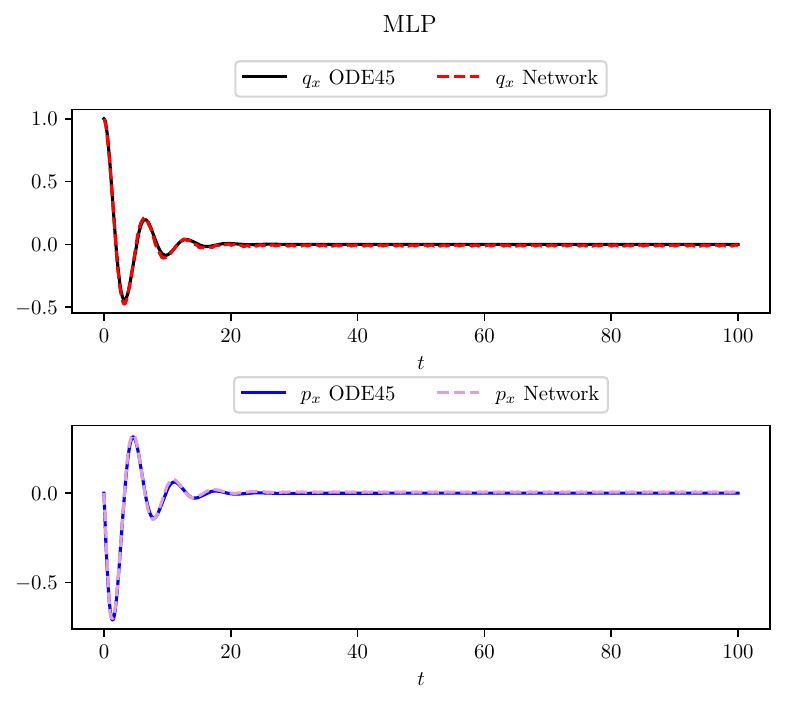}
\end{subfigure}
\begin{subfigure}{.2\linewidth}
\centering
\includegraphics[width=.7\linewidth]{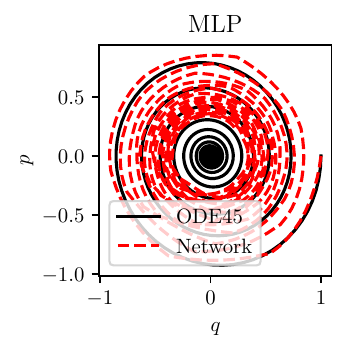}
\includegraphics[width=.7\linewidth]{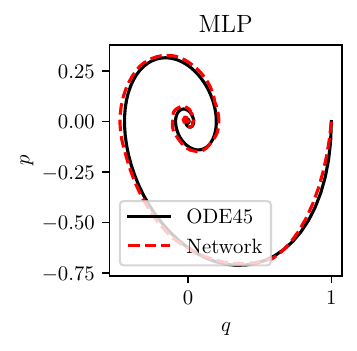}
\end{subfigure}
\caption{\textbf{MLP Supervised experiment --- Damped Harmonic Oscillator:} Trained with $N = 50$, $M=10$, and no noise, $\varepsilon=0$.}
\label{fig: MLP_supervised}
\end{figure}

%%% DHO SYMPFLOW LOW N AND M
\begin{figure}[ht!]
\centering
\begin{subfigure}{.78\textwidth}
\centering
\includegraphics[width=.45\linewidth]{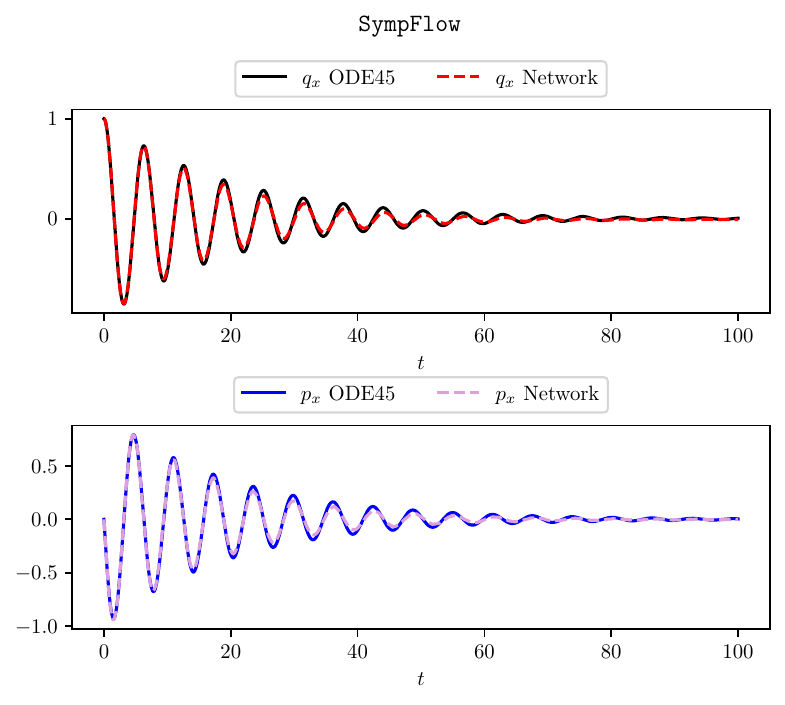}
\includegraphics[width=.45\linewidth]{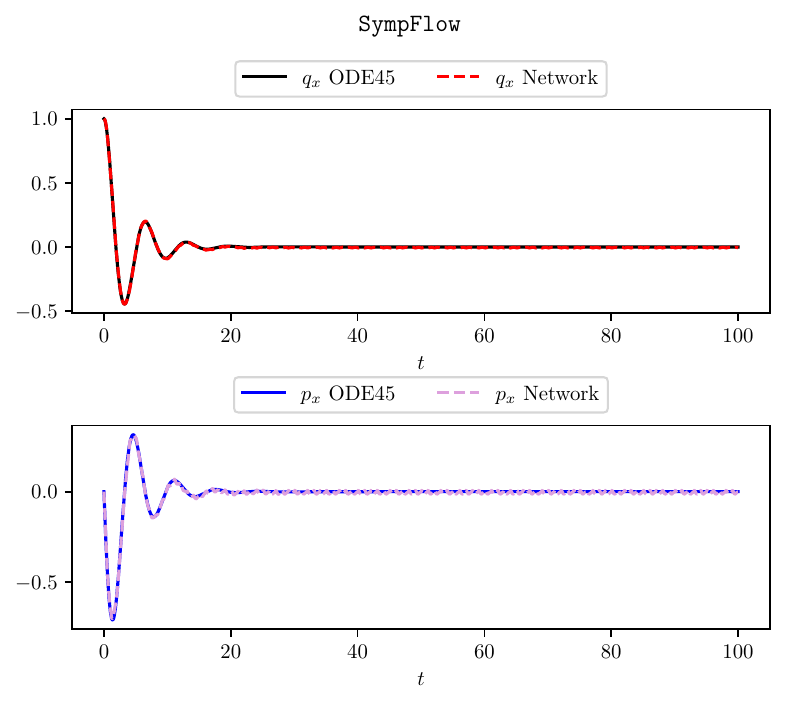}
\label{fig:tbd_2}
\end{subfigure}
\begin{subfigure}{.2\textwidth}
\centering
\includegraphics[width=.7\linewidth]{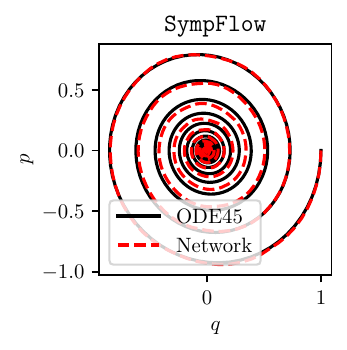}
\includegraphics[width=.7\linewidth]{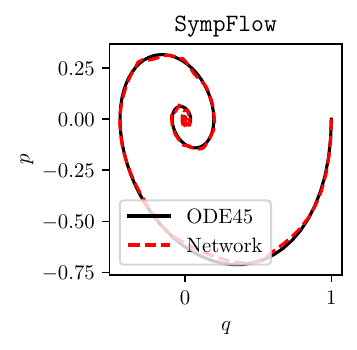}
\label{fig:tbd_1}
\end{subfigure}
\caption{\textbf{\texttt{SympFlow} supervised experiment --- Damped Harmonic Oscillator:} Trained with $N = 50$, $M=10$, and no noise, $\varepsilon=0$.}
\label{fig: sympflow_supervised}
\end{figure}

%%% UNSUPERVISED VS SUPERVISED
\begin{figure}[ht!]
\centering
\includegraphics[width=.7\textwidth]{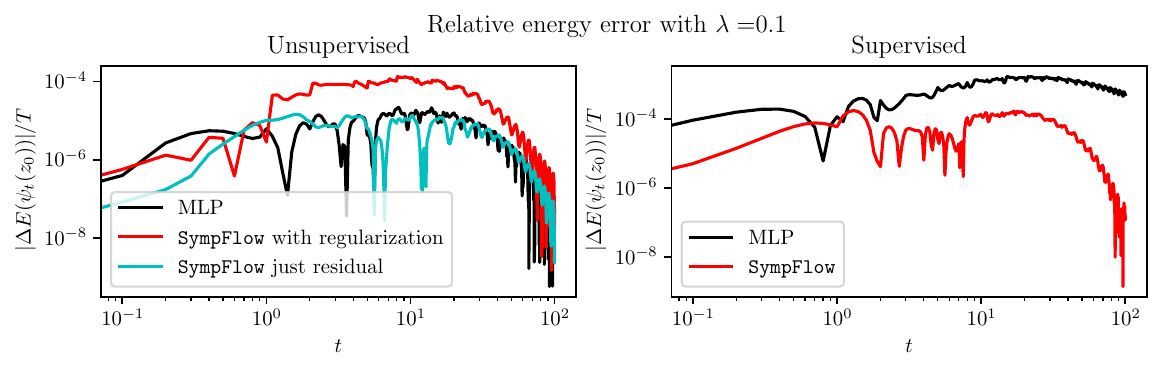}
\caption{Comparison of the damped harmonic oscillator relative energy, normalized over the integration time, for unsupervised (left) and supervised (right) experiments. In the unsupervised case, both the MLP and \texttt{SympFlow}. In the supervised case, the solutions have been computed with $N=50$ initial conditions and $M=10$ sampling points. In this setting, with fewer initial conditions, \texttt{SympFlow} outperforms MLP.}
\label{fig: sympflow_pinn_comparison}
\end{figure}

\subsection{H\'enon--Heiles system}\label{sec:HH}
The H\'enon--Heiles system is a model for studying non-linear dynamics, chaos, and the transition from regular to chaotic motion in physics. Initially developed to understand the motion of a star in a galactic potential, the system has since become a paradigmatic example in the study of chaos theory. The H\'enon--Heiles system is a Hamiltonian system based on the Hamiltonian function \cite[Section I.3]{hairerGeometricNumericalIntegration2006}
\begin{equation}\label{eq:HeilesHam}
H(q,p) = \frac{1}{2}(p_x^2+p_y^2)+\frac{1}{2}(q_x^2+q_y^2) + q_x^2q_y-\frac{q_y^3}{3},\,\,q=(q_x,q_y),p=(p_x,p_y)\in\mathbb{R}^2.
\end{equation}
The equations of motion for this system write
\begin{equation}\label{eq:HeilesEq}
\dot{q}_x = p_x,\,\,\dot{q}_y = p_y,\,\,\dot{p}_x = -q_x-2q_xq_y,\,\,\dot{p}_y = -q_y-(q_x^2-q_y^2).
\end{equation}
The equations of motion in \Cref{eq:HeilesEq} provide a considerable challenge to numerical integrators and network-based simulations since they exhibit chaotic behavior corresponding to specific energy levels. Due to the chaotic nature of the system, it is unreasonable to expect a long-term agreement of the approximation with the exact solution. However, for chaotic systems, one would like to capture the correct global behavior of the trajectories. A common strategy to test this is to consider a two-dimensional Poincar\'e section and compare the obtained results for different methods. We proceed in this way to compare \texttt{SympFlow} with an MLP. To that end, we start by considering the Poincar\'e section defined as
\begin{equation}\label{eq:poincare}
\mathcal{S}=\left\{(0,q_y,p_x,p_y):\,\,q_y,p_x,p_y\in\R\right\},
\end{equation}
and we will provide a planar representation, in the variables $(q_y,p_y)$, of the intersection of the trajectories with the section $\mathcal{S}$. For a definition of the notion of Poincar\'e sections, see \cite{strogatz2018nonlinear}.

We now move to the two experimental settings we consider for this system. All the experiments fix $\Delta t=1$ and $\Omega = [-1,1]^4\subset\R^4$. Some of these initial conditions correspond to energy levels leading to chaotic dynamics. Since the training process relies on the short integration interval $[0,1]$, this does not seem to affect the quality of the recovered models negatively.
\subsubsection{Unsupervised experiments}
As for the previous two test problems, we consider the unsupervised problem of solving the differential equations in \Cref{eq:HeilesEq} for arbitrary initial conditions in $\Omega\subset\R^4$, and for time instants in $[0,\Delta t]$. The plots consider the initial condition $(0.3, -0.3, 0.3, 0)$, leading to initial energy $H_0=0.13$, corresponding to chaotic dynamics. \one{We collect in \Cref{fig:henonHeiles} the experiments with the two training regimes for \texttt{SympFlow} and an MLP, i.e., with and without regularization by \Cref{eq:L2} and \Cref{eq:mlpReg} to promote Hamiltonian conservation}. We see that \texttt{SympFlow} leads to better long-term energy behavior, and that the mixed training regime leads to slightly improved results, \one{see \Cref{app:mixed}}. We also plot the Poincar\'e cuts associated with the section in \Cref{eq:poincare}. A reference cut obtained with Runge--Kutta $(5,4)$ can be found in \Cref{app:heiles}. The \texttt{SympFlows} lead to much more qualitatively accurate solutions, given that the obtained cuts resemble the expected behavior. \one{We also notice that the regularization term considerably improves the MLP predictions, but is not so significant for \texttt{SympFlow}, which captures the correct behavior without regularization.} {The regularized \texttt{SympFlow} performs slightly worse than the unregularized one, suggesting that it would be worth further exploring the design of a more suitable regularization term, or refining the balance between the two loss terms.} \Cref{app:heiles} also collects the plots of the solution curves obtained with this initial condition for the four trained models.
\begin{figure}[ht!]
    \centering
    \includegraphics[scale=0.5]{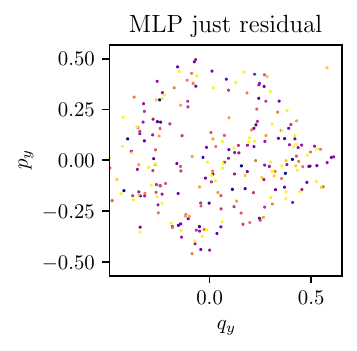}
    \includegraphics[scale=0.5]{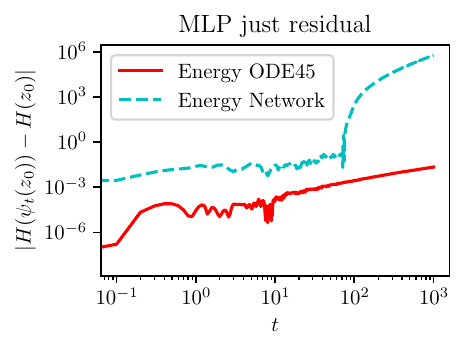}
    \includegraphics[scale=0.5]{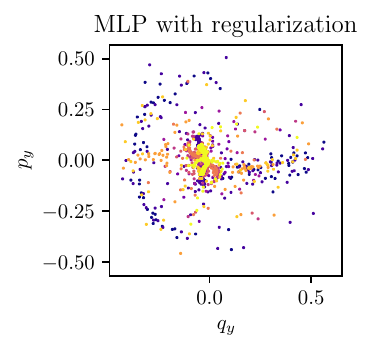}
    \includegraphics[scale=0.5]{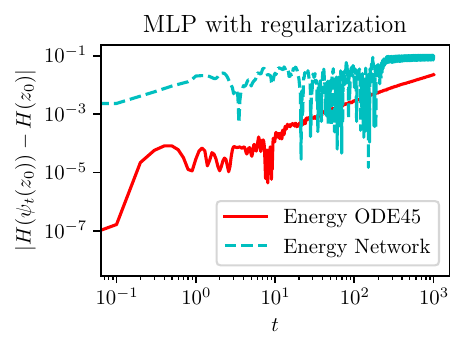}

    \includegraphics[scale=0.5]{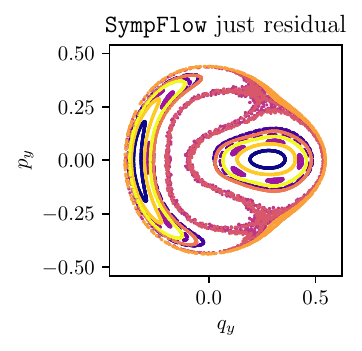}
    \includegraphics[scale=0.5]{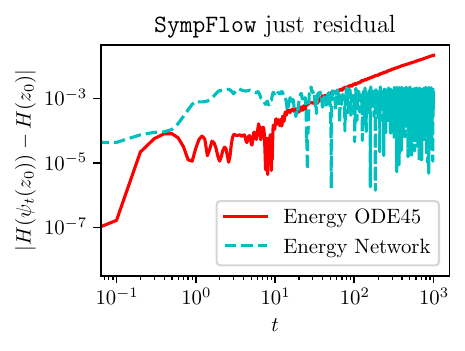}
    \includegraphics[scale=0.5]{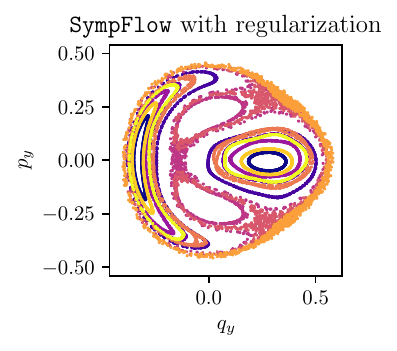}
    \includegraphics[scale=0.5]{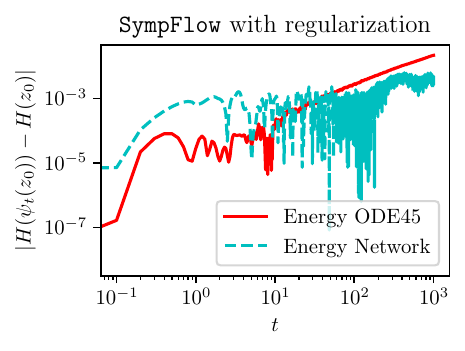}

    \caption{\textbf{Unsupervised experiment --- H\'enon--Heiles: }Comparison of the Poincar\'e sections and the energy behavior up to time $T=1000$.}
    \label{fig:henonHeiles}
\end{figure}

\subsubsection{Supervised experiments}
For the supervised experiment, we notice a pretty similar situation as in the unsupervised one, as can be seen in \Cref{fig:henonHeiles-supervised}. In this setup, it is also very hard to train the MLP model, as can be seen from the produced solution curves in \Cref{fig:heilesSupervisedSolutions} of \Cref{app:heiles}. Once more, \texttt{SympFlow} demonstrates the ability to capture the correct global qualitative behavior of the solutions as we can see in both the energy plot and the Poincar\'e cuts in \Cref{fig:henonHeiles-supervised}.
\begin{figure}[ht!]
    \centering
    \includegraphics[scale=0.5]{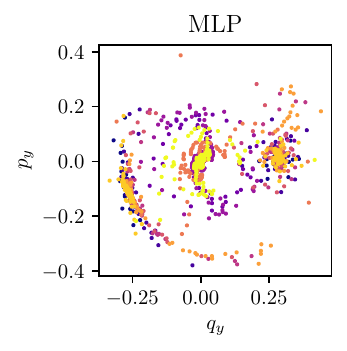}
    \includegraphics[scale=0.5]{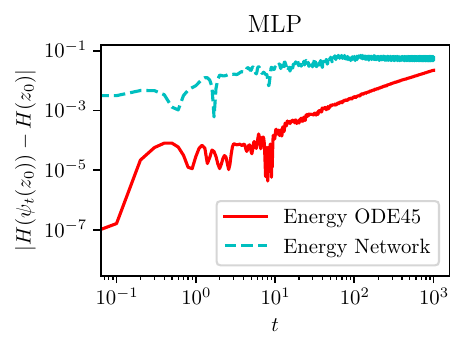}
    \includegraphics[scale=0.5]{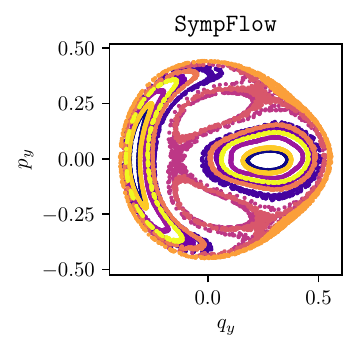}
    \includegraphics[scale=0.5]{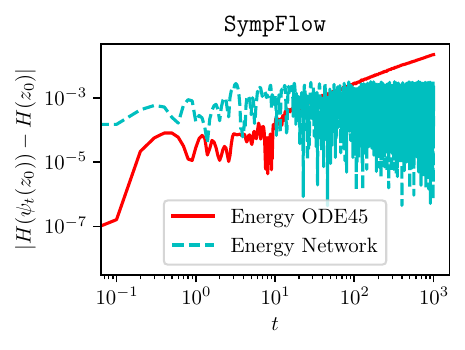}
\caption{\textbf{Supervised experiment --- H\'enon--Heiles: }Comparison of the MLP and the \texttt{SympFlow} trained with a dataset of parameters $N=100$, $M=50$, and $\varepsilon=0$. We show the energy plots for the time interval $[0,1000]$, and also the Poincar\'e sections.}
    \label{fig:henonHeiles-supervised}
\end{figure}

\section{Conclusions and further work}\label{sec:conclusions}

In this work we presented \texttt{SympFlow}, a symplectic neural flow able to provide accurate and reliable long-term solutions for generic Hamiltonian systems. We have shown that \texttt{SympFlow} is a universal approximator and it can be applied to (i) approximately solve the equations of motion of Hamiltonian systems and (ii) approximate the flow map of unknown Hamiltonian system based on trajectory data. Because of its structure, \texttt{SympFlow} admits an underlying Hamiltonian which could also be used to study an approximated physical model, a task which has previously been studied in the scientific machine learning literature \cite{bertalanLearningHamiltonianSystems2019, greydanusHamiltonian2019}. Potentially, both \texttt{SympFlow} functionalities, (i)-(ii), can be combined to model complex physical systems. We have have also demonstrated the advantages of \texttt{SympFlow} over an unconstrained MLP in several experimental tasks, including a chaotic system and a non-conservative system. The numerical experiments, show that the performance of the \texttt{SympFlow} is equal or superior to an MLP network for unsupervised tasks. For supervised tasks, however, \texttt{SympFlow} outperforms MLP, particularly when the number of initial conditions and time samples are significantly reduced, highlighting the data efficiency of \texttt{SympFlow}. 

There are several natural ways to extend this research. From the theoretical point of view, it would be interesting to further explore the energy behavior of the \texttt{SympFlow} since we experimentally observe a better error growth than the one in \Cref{thm:aposteriori}. We will also expand the applicability of \texttt{SympFlow} to higher-dimensional dynamical systems, such as spatially semi-discretized PDEs. To do so, the main effort will be in improving the computational efficiency of the \texttt{SympFlow} layers, which currently rely on automatic differentiation which can be costly for higher dimensional problems. {Additionally, since the Hamiltonian matching regularization term does not appear to benefit \texttt{SympFlow}, we will investigate how to better balance the two loss terms, as well as explore improved regularization strategies for training these constrained models.}

\section*{Acknowledgements}
The authors would like to thank Georg Maierhofer for useful discussions, Lidia Gomes Da Silva for suggesting revisions to some of the plots, and the reviewers for their thoughtful comments on the universality proof, which made it much more rigorous. DM acknowledges support from a grant from the Simons Foundation, the EPSRC programme grant in ‘The Mathematics of Deep Learning’, under the project EP/V026259/1, and the Norwegian University of Science and Technology. CBS acknowledges support from the Philip Leverhulme Prize, UK, the Royal Society Wolfson Fellowship, UK, the EPSRC advanced career fellowship, UK EP/V029428/1, EPSRC grants EP/S026045/1 and EP/T003553/1, EP/N014588/1, EP/T017961/1, the Wellcome Innovator Awards, UK 215733/Z/19/Z and 221633/Z/20/Z, the European Union Horizon 2020 research and innovation programme under the Marie Skłodowska-Curie grant agreement No.777826 NoMADS, the Cantab Capital Institute for the Mathematics of Information, UK and the Alan Turing Institute, UK. FS acknowledges support from the EPSRC advanced career fellowship EP/V029428/1. PC thanks the Alan Turing Institute for sponsoring her Daphne Jackson Fellowship. ZS acknowledges support from the Cantab Capital Institute for the Mathematics of Information and the Trinity Henry Barlow Scholarship scheme.

\appendix

\section{Derivation of the long-time Hamiltonian system}\label{app:longTimeHam}
Let us consider
\[
\psi(t,x_0):=\bar{\psi}_{t-\Delta t\lfloor t/\Delta t\rfloor}\circ \left(\bar{\psi}_{\Delta t}\right)^{\lfloor t/\Delta t\rfloor}(x_0),
\]
and evaluate its derivative in time. We evaluate the following limit
\[
\dot{\psi}(t,x_0)=\lim_{h\to 0} \frac{\psi(t+h,x_0)-\psi(t,x_0)}{h}
\]
where we notice that if $t\neq k\Delta t$, $k\in\mathbb{N}$, one has that $\lfloor (t+h)/\Delta t \rfloor = \lfloor t/\Delta t\rfloor$ as long as $h$ is small enough. Thus, setting $c=\lfloor t/\Delta t\rfloor$, we get
\begin{align*}
\dot{\psi}(t,x_0)&=\lim_{h\to 0}\frac{\bar{\psi}_{t+h-c\Delta t}((\bar{\psi}_{\Delta t})^{c}(x_0))-\bar{\psi}_{t-c\Delta t}((\bar{\psi}_{\Delta t})^{c}(x_0))}{h}\\
&=\mathbf{J}\nabla\mathcal{H}(\bar{\psi})(t-c\Delta t,\psi(t,x_0)).
\end{align*}
as written in \Cref{eq:longTimeHamiltonian}.

\section{Proof of the a-posteriori result}\label{app:aposteriori}
\begin{proof}[Proof of \Cref{thm:aposteriori}]
To start, we notice that
\begin{align*}
\mathcal{H}(\psi)(t,x)&=\mathcal{H}(\bar{\psi})(t-\Delta t\lfloor t/\Delta t\rfloor,x)=H(x) + \delta(t-\Delta t\lfloor t/\Delta t\rfloor,x)
\end{align*}
for some function $\delta:[0,\Delta t]\times\R^{2d}\to\R$ with $|\delta(t,x)|\leq \varepsilon_1$ for every $t\in [0,\Delta t]$, $x\in\Omega$. In this proof, when using the symbol $\nabla$ we always refer to the gradient with respect to the second input. By the second assumption, on the smallness of the residual, we know that
\[
\left\|\mathbf{J}\nabla\mathcal{H}(\psi)(t,\psi_t(x))-\mathbf{J}\nabla H(\psi_t(x))\right\|_2\leq \varepsilon_2,
\]
and since $\mathbf{J}$ preserves the $\ell^2$ norm, we get that $\|\nabla \delta (t,x)\|_2\leq \varepsilon_2$ too. This implies that
\[
\frac{\dd}{\dd t}H(\psi_t(x))=\nabla H(\psi_t(x))\cdot\mathbf{J}\nabla \delta(t-\Delta t\lfloor t/\Delta t\rfloor,\psi_t(x))\leq C\varepsilon_2,
\]
where $C$ upper bounds the Lipschitz constant of $H$ over $\Omega$. Thus, $\left | H(\psi_t(x))-H(x)\right |\leq c_2 \varepsilon_2 t$. Moving to the variation of $\mathcal{H}(\psi)$, we notice that
\begin{align*}
\mathcal{H}(\psi)(t,\psi_t(x))-\mathcal{H}(\psi)(0,x) &= \mathcal{H}(\psi)(t,\psi_t(x))-H(\psi_t(x))+H(\psi_t(x))-H(x)\\
&+H(x)-\mathcal{H}(\psi)(0,x),
\end{align*}
which allows us to conclude
\[
\left | \mathcal{H}(\psi)(t,\psi_t(x))-\mathcal{H}(\psi)(0,x)\right|\leq 2\varepsilon_1 + c_2\varepsilon_2t\leq \max\{2,c_2\}(\varepsilon_1 + \varepsilon_2 t)=:c_1(\varepsilon_1+\varepsilon_2 t).
\]
\end{proof}

\section{Description of the network architectures}\label{app:architectures}
In what follows, we provide details on the network architectures considered in \Cref{sec:experiments}. All the networks we compare have the same number of layers, where each layer is defined as a transformation acting on both parts, \( q \) and \( p \), which partition the phase-space variable \( x \in \mathbb{R}^{2d} \). The number of layers chosen in the numerical experiments in \Cref{sec:experiments} is based on \texttt{SympFlow}'s performance. Thus, in the unsupervised experiments, the number of layers is set to three, whereas in the supervised experiments is set to five. Additionally, in all experiments the activation function is $\sigma=\tanh$. 

Although we have set the MLP to have the same numer of layers than \texttt{SympFlow}, the highly constrained architecture of \texttt{SympFlow} results in a greater number of network parameters compared to the MLP. Nevertheless, we consider this comparison fair because, in all numerical experiments, the MLP accurately approximates the target map just as well as \texttt{SympFlow} over \( \Omega \times [0, \Delta t] \). This demonstrates that the poor long-term behavior of the MLP is not due to a lack of expressive power but rather to its lack of a symplectic structure.

\subsection{\texttt{SympFlow}}
Given a \texttt{SympFlow}, the $i-$th layer of $\bar{\psi}_t$ is given by the composition map $\phi_{\mom}^i\circ\phi_{\pos}^i$. Each of these two composed maps depends on network weights only via a potential function, which is denoted as $V_{\pos}^i$ and $V_{\mom}^i$. We model both potential functions with a feedforward neural network. We explicit it for $V_{\pos}^i$, and the same applies to $V_{\mom}^i$, but with different weights:
\[
V_{\pos}(t,q) = \ell_3\circ \sigma \circ \ell_2 \circ \sigma \circ \ell_1 \left(\begin{bmatrix} q \\ t \end{bmatrix}\right),
\]
where $\ell_1:\R^{d+1}\to\R^h$, $\ell_2:\R^h\to\R^h$, $\ell_3:\R^h\to\R$, $h=10$, are three parametrized affine maps of the form $\ell_j(x)=A_jx+b_j$ for suitably shaped weight $A_j$ and bias $b_j$, $j=1,2,3$. For the case of the simple harmonic oscillator, where $d=1$, we thus have $30$ parameters modeling $\ell_1$, $110$ modeling $\ell_2$, and $11$ modeling $\ell_3$. A \texttt{SympFlow} of $5$ layers will hence have $5\cdot 2\cdot (30+110+11) = 1510$ parameters.

\subsection{MLP}
An MLP with $L\in\mathbb{N}$ layers is based on a parametric maps of the form
\begin{equation}\label{eq:mlpLayer}
\varphi_i : \R^{c_i}\to\R^{c_{i+1}},\,\,\varphi_i(z) = \sigma \circ \ell_i(z),\,\,i=1,...,L,
\end{equation}
for an affine map parametrized as $\ell_i(z)=A_iz+b_i$, $A_i\in\R^{c_{i+1}\times c_i}$ and $b_i\in\R^{c_{i+1}}$.  For the first layer we have $c_i=2d+1$, since $z = [x,t]$, whereas for the last layer $c_{L+1}=2d$. All the intermediate dimensions are fixed to $c_2=...=c_L=10$. For the simple harmonic oscillator, considering an MLP of $5$ layers, we thus have that the first layer is parametrized by $40$ parameters, the intermediate three by $110$ each, and the last by $22$. This leads to a network with $40+330+22=392$ parameters.

Since composing maps as in \Cref{eq:mlpLayer} does not allow to enforce the initial condition, we modify the composition and define the MLP network as
\[
\bar{\psi}(t,x) = x + \tanh{(t)} \cdot \varphi_L \circ ... \circ \varphi_1\left(\begin{bmatrix} x \\ t \end{bmatrix}\right),
\]
so that $\bar{\psi}(0,x)=x$ for every $x\in\R^{2d}$ since $\tanh{(0)}=0$.

\two{\section{Background on symplectic maps and Hamiltonian dynamics}\label{app:background}
This section provides some background material on symplectic maps, Hamiltonian systems, and symplectic numerical methods.}

\two{
\begin{definition}
A differentiable map $F:\Omega\to\R^{2d}$, $\Omega\subseteq\R^{2d}$ open, is symplectic if, for every $x\in\Omega$, $(F'(x))^\top \mathbf{J}F'(x)=\mathbf{J}$.
\end{definition}
\begin{lemma}
Let $F,G:\R^{2d}\to\R^{2d}$ be two differentiable symplectic maps. Then, the maps $F\circ G,G\circ F:\R^{2d}\to\R^{2d}$ are symplectic as well. 
\end{lemma}
\begin{proof}
We prove the result for $F\circ G$, and the argument repeats analogously for $G\circ F$. By the chain rule, it holds that
\[
(F\circ G)'(x) = F'(G(x))G'(x).
\]
Thus,
\[
((F\circ G)'(x))^\top \mathbf{J} (F\circ G)'(x) = (G'(x))^\top (F'(G(x)))^\top \mathbf{J} F'(G(x))G'(x) = (G'(x))^\top\mathbf{J}G'(x) = \mathbf{J}
\]
as desired.
\end{proof}
}

\two{We then prove that the time-$t$ flow of a time-dependent Hamiltonian system is a symplectic diffeomorphism if the Hamiltonian function is twice continuously differentiable.}

\two{We recall that a time-dependent Hamiltonian system is described by the system of differential equations
\[
\dot{x}(t) = \mathbf{J}\nabla_x H(t,x(t)),
\]
for a time-dependent Hamiltonian energy function $H:\R\times\R^{2d}\to\R$. Such a definition could also be restricted to an open subset $\Omega\subseteq\R^{2d}$. We remark that the notation $\nabla_x$ stands for the gradient with respect to the second entry.}

\two{\begin{proposition}
Let $H:\R\times\Omega \to\R$ be a twice continuously differentiable function on $\Omega\subseteq\R^{2d}$ open. Then, for each fixed $t\in\R$, the time$-t$ flow map $\phi_{H,t}:\Omega\to\R^{2d}$ is symplectic whenever defined.
\end{proposition}
\begin{proof}
We recall that the flow map $\phi_{H,t}:\Omega\to\R^{2d}$ satisfies
\[
\frac{\dd}{\dd t}\phi_{H,t}(x_0) = \mathbf{J}\nabla_x H\left(t,\phi_{H,t}(x_0)\right)
\]
for every $t\geq 0$ and $x_0\in\Omega$. Differentiating both sides with respect to $x_0$, we get
\[
\frac{\partial}{\partial x_0}\frac{\dd}{\dd t}\phi_{H,t}(x_0) = \mathbf{J}\nabla^2_x H(t,\phi_{H,t}(x_0))\,\frac{\partial}{\partial {x_0}}\phi_{H,t}(x_0),
\]
where $\nabla^2_x H$ is the Hessian matrix of $H$ with respect to the second component. Changing the differentiation order on the left, and calling $S_{x_0}(t)=\partial_{x_0}\phi_{H,t}(x_0)$, we see that
\begin{equation}
\frac{\dd}{\dd t}S_{x_0}(t) = \mathbf{J}\nabla^2_x H(t,\phi_{H,t}(x_0))S_{x_0}(t).
\end{equation}
We can then compute
\begin{align*}
&\frac{\dd}{\dd t}\left(S_{x_0}(t)^{\top}\mathbf{J}S_{x_0}(t)\right) = \left(\frac{\dd}{\dd t}S_{x_0}(t)\right)^{\top}\mathbf{J}S_{x_0}(t)  + S_{x_0}(t)^{\top}\mathbf{J}\left(\frac{\dd}{\dd t}S_{x_0}(t)\right)\\
&=\left(S_{x_0}(t)^{\top}\nabla^2_x H(t,\phi_{H,t}(x_0))\mathbf{J}^{\top}\right)\mathbf{J}S_{x_0}(t) + S_{x_0}(t)^{\top}\mathbf{J}\left(\mathbf{J}\nabla^2_x H(t,\phi_{H,t}(x_0))S_{x_0}(t)\right).
\end{align*}
Since $\mathbf{J}^{\top}\mathbf{J}=I_{2d}=-\mathbf{J}^2$ we conclude that the quantity above vanishes and hence 
\[
S_{x_0}(t)^{\top}\mathbf{J}S_{x_0}(t)=S_{x_0}(0)^{\top}\mathbf{J}S_{x_0}(0).
\]
At time $t=0$, we recall that
\[
S_{x_0}(0)=\partial_{x_0}\phi_{H,0}(x_0) = \partial_{x_0}x_0=I_{2d}
\]
which allows us to conclude $S_{x_0}(t)^{\top}\mathbf{J}S_{x_0}(t)=\mathbf{J}$ for every $t\geq 0$ as desired.
\end{proof}
To numerically approximate the solutions of a dynamical system, one needs to use numerical methods. When the target system possesses specific qualitative properties, such as energy preservation, symplecticity, or volume preservation, it is desirable to rely on numerical methods that reproduce these properties. Designing methods compatible with the qualitative properties of dynamical systems is the primary goal of an area of numerical analysis known as \emph{geometric numerical integration}, as described in \cite{hairerGeometricNumericalIntegration2006}. Symplectic one-step numerical methods are defined by maps $\varphi^{t_0,\Delta t}:\mathbb{R}^{2d}\to\mathbb{R}^{2d}$, $\varphi^{t_0,\Delta t}(x(t_0))\approx x(t_0+\Delta t)$, which, when applied to a Hamiltonian system, are symplectic maps. The Lie-Trotter splitting method is an example of a symplectic method, also often referred to as the symplectic Euler method in the context of symplectic integration. We utilise this scheme in our proof of the universal approximation theorem, so we briefly introduce it here. Consider a time-dependent separable Hamiltonian function, i.e., a function $H:\mathbb{R}\times\mathbb{R}^{2d}\to\mathbb{R}$ of the form {$H(t,(q,p))=K(t,p)+V(t,q)$} where $K,V:\mathbb{R}\times\mathbb{R}^d\to\mathbb{R}$ could be thought as the kinetic and potential energies of the system. Here, we assume $x=(q,p)\in\mathbb{R}^{2d}$ is split into two components of the same dimension $q,p\in\mathbb{R}^d$. We can then write the Hamiltonian vector field of $H$ as the sum of two simpler Hamiltonian systems:
\[
\dot{x}(t) = \mathbf{J}\nabla_x H(t,x(t)) = \begin{bmatrix} \nabla_p K(t,p(t)) \\ 0_d \end{bmatrix} + \begin{bmatrix} 0_d \\ -\nabla_q V(t,q(t)) \end{bmatrix}.
\]
The two components on the right-hand side are Hamiltonian systems that can be exactly solved up to being able to integrate the gradients of the two components $K$ and $V$:
\begin{align*}
\phi_{K,t_0,\Delta t}(q_0,p_0) &= \left(q_0 + \int_{t_0}^{t_0+\Delta t} \nabla_p K(t,p_0)dt,\;p_0\right),\\
\phi_{V,t_0,\Delta t}(q_0,p_0) &= \left(q_0,\;p_0- \int_{t_0}^{t_0+\Delta t}\nabla_q V(t,q_0)dt\right).
\end{align*}
The notation $\phi_{K,t_0,\Delta t}(q_0,p_0)$ and $\phi_{V,t_0,\Delta t}(q_0,p_0)$ is used to express the solution at time $t_0+\Delta t$ of the two Hamiltonian initial value problems with Hamiltonians $K$ and $V$, respectively, starting from time $t_0$. The two maps in the previous equation are exactly the layers of our \texttt{SympFlow}. The symplectic Euler method, applied to approximate $\phi_{H,t_0,\Delta t}$ is thus defined as
\[
\varphi^{t_0,\Delta t} = \phi_{K,t_0,\Delta t} \circ \phi_{V,t_0,\Delta t},
\]
and it could also be written by composing the flows in the opposite order, despite generating a different map, which still goes under the same name. We note that if $H$ is time-independent, then the integrals are easily solvable, and the flows we are composing reduce to linear shifts in one of the two components.
}

\section{The classical mechanics of non-conservative systems: Damped Harmonic Oscillator}
\label{app:NCA}
In this section, we summarize the non-conservative Hamiltonian formulation used to extend \texttt{SympFlow} to dissipative systems \cite{Galley_2013, Galley_2014, Tsang_2015}. This approach enables the modeling of dissipative dynamics as a canonical Hamiltonian system but does not inherently preserve symplecticity. We outline the key aspects of this formulation and detail the modifications required to ensure symplecticity, allowing its integration into \texttt{SympFlow}.

The starting point of the non-conservative Hamiltonian formulation \cite{Galley_2013} is to move from the system's phase space to an augmented one, which is obtained by doubling the degrees of freedom, and back to the original space. The main idea is to capture the dissipative evolution in this augmented phase-space, ensuring that the global dynamics of the system is conservative. Hence, the configuration variable $q$, and the momentum $p$ are doubled giving rise to two different curves each: $q \rightarrow (q_a, q_b)$ and $ p \rightarrow (p_a, p_b)$. In this framework, a new non-conservative Lagrangian is derived from the augmented action with doubled variables:
\begin{align}
	{\cal S} = \int \Lambda(q_{a,b}, \dot{q}_{a,b}, t) \,\dd t \label{NC_action}
\end{align}
where the (non-conservative) Lagrangian is defined by:
\begin{align}
	\Lambda(q_{a,b}, \dot{q}_{a,b}, t) \! = \! L(q_a, \dot{q}_a, t) \!-\! L(q_b, \dot{q}_b, t) \!+\! K(q_{a,b}, \dot{q}_{a,b}, t)\;, \nonumber
\end{align}
where $L$ describes the conservative contribution to the full non-conservative Lagrangian $\Lambda$, and $K$ is a function that couples the variables together and accounts for all the non-conservative contribution to the system dynamics. Notice that, in conservative Hamiltonian systems, the function $K$ would vanish. Although, $K$ also vanishes in the physical limit (PL) \cite{Galley_2013}, that is, in the linear subspace $\{(q_a,q_b,p_a,p_b)\in\R^{4d}:\,\,q_a = q_b,\text{ and }p_a =p_b\}$, its derivatives do not necessarily vanish in the physical limit. This is important, since the derivatives of $K$ model the dissipative behavior of the system, as we will see below. 

The stationarity of the action in \Cref{NC_action} under the variations $q_{I}(t,\epsilon) = q_{I}(t,0) + \epsilon \eta _{I}(t)$ , that is $[ \dd S [ q_I] / \dd\epsilon ]_{\epsilon=0}=0$ for all $\eta_I(t_i)=\eta_I(t_f)=0$, $I =\{ a, b\}$, corresponds to the conditions
\begin{align}
0 = {} & \int_{t_i}^{t_f}   \bigg\{  \bigg[ \frac{ \partial \Lambda }{ \partial q_a} \frac{ \partial q_a }{ \partial \epsilon}   +   \frac{ \partial \Lambda }{ \partial \dot{q}_a}  \frac{\partial \dot{q}_a }{ \partial \epsilon} +
\frac{ \partial \Lambda }{ \partial q_b} \frac{ \partial q_b }{ \partial \epsilon} \nonumber+ \frac{ \partial \Lambda }{ \partial \dot{q}_b}  \frac{\partial \dot{q}_b }{ \partial \epsilon} \bigg]_{\epsilon=0}  \bigg\} \, \dd t
\\ \nonumber
= {}& \int_{t_i}^{t_f}  \bigg\{ \eta_a \cdot \bigg[ \frac{ \partial \Lambda }{ \partial q_a} - \frac{ \dd \pi _a}{\dd t} \bigg]_{\epsilon=0}
+ \eta_b \cdot \bigg[ \frac{ \partial \Lambda }{ \partial q_b} + \frac{ \dd \pi_b}{\dd t} \bigg]_{\epsilon=0} \bigg\} \, \dd t \nonumber \\
+{}& \bigg[ \eta_a(t) \cdot \pi_b(t) - \eta_b(t) \cdot \pi_a(t) \bigg]_{t=t_i}^{t_f}
	\label{vari}
\end{align}
where the quantities $\pi_I$ are the (non-conservative) conjugate momenta, defined by:
\ba
& \pi_{a}(q_I, \dot{q}_I, t) = \frac{\partial \Lambda(q_I, \dot{q}_I, t)}{\partial \dot{q}_a(t)} = \frac{\partial L(q_a, \dot{q}_a, t)}{\partial \dot{q}_a} + \frac{\partial K(q_I, \dot{q}_I, t)}{\partial \dot{q}_a} \\
& \pi_{b}(q_I, \dot{q}_I, t) = \frac{\partial \Lambda(q_I, \dot{q}_I, t)}{\partial \dot{q}_b} = -\frac{\partial L(q_b, \dot{q}_b, t)}{\partial \dot{q}_b} +\frac{\partial K(q_I, \dot{q}_I, t)}{\partial \dot{q}_b}
\ea

In \cite{Galley_2013}, the authors choose a different sign convention for the momenta hence working with Hamiltonian equations associated to a non-canonical symplectic structure over $\R^{4d}$. Our choice for the sign of $p$ is due to the fact that \texttt{SympFlow} preserves the canonical symplectic form of $\R^{4d}$. With this sign convention, the physical limit expressed with respect to the momenta corresponds to the subspace given by:
\begin{equation}\label{eq:pl}
\{(q_a,q_b,\pi_a,\pi_b)\in\R^{4d}:\,\,q_a=q_b,\text{ and }\pi_a=-\pi_b\}.
\end{equation}
We remark that $\pi_a=-\pi_b$ can be inferred by noticing that $\pi_a+\pi_b= p_a - p_b - \lambda(q_a-q_b)=0$, where $p_a$ and $p_b$ are the conservative momenta derived from the conservative Lagrangian as $p_I=\partial_{\dot{q}_I}L(q_I,\dot{q}_I)$, $I\in\{a,b\}$.

In the damped harmonic oscillator example, we consider a dissipative potential of the form $K = -\frac{\lambda}{2} (\dot{q}_{a}+\dot{q}_{b})(q_{a}-q_{b})$ \cite{Tsang_2015} (see \cite{Galley_2014} for other choices). The corresponding (non-conservative) Lagrangian is, then, given by:
\ba
\Lambda(q_I,\dot{q}_I) =\left(\frac{m}{2}\dot{q}_a^2 - \frac{k}{2}q_a^2\right)-\left(\frac{m}{2}\dot{q}_b^2 - \frac{k}{2}q_b^2\right)-\frac{\lambda}{2} \left(\dot{q}_a+\dot{q}_b\right)\left(q_a-q_b\right),
\ea
where $k$ is the oscillator's recovery constant, and $\lambda$ is the dissipation factor. The corresponding (non-conservative) conjugate momenta for each double variable are then given by:
\ba
\pi_a &=& \frac{\partial \Lambda(q_I,\dot{q}_I)}{\partial \dot{q}_a}=m\dot{q}_a-\frac{\lambda}{2}(q_a-q_b)=p_a-\frac{\lambda}{2}(q_a-q_b) \\
\pi_b &=& \frac{\partial \Lambda(q_I,\dot{q}_I)}{\partial \dot{q}_b}=-m\dot{q}_b-\frac{\lambda}{2}(q_a-q_b)=-p_b-\frac{\lambda}{2}(q_a-q_b).
\ea
That is, the non-conservative momenta are given by the conservative ones $p_a,p_b$, plus an extra term due to the dissipative contribution. We also remark that, in the physical limit, $\pi_a=p_a$ and $\pi_b=-p_b$. 

The non-conservative (augmented) Hamiltonian is obtained by applying the Legendre transform to the non-conservative Lagrangian $\Lambda$:
\ba
A(t,q_a,q_b,\pi_a,\pi_b) &=&\frac{1}{2m}(\pi_a^2-\pi_b^2)+\frac{\lambda}{2m}(q_a-q_b)(\pi_a-\pi_b)\nonumber \\
&+&\frac{k}{2}(q_a-q_b)(q_a+q_b)\nonumber\;.
\ea
Hence, the final Hamiltonian equations for the damped harmonic oscillator in the augmented space are given by: 
\ba
\begin{bmatrix} \dot{q}_a(t) \\ \dot{q}_b(t) \\ \dot{p}_a(t) \\ \dot{p}_b(t) \end{bmatrix} =
\begin{bmatrix}
\frac{\pi_a(t)}{m} + \frac{\lambda}{2m}(q_a(t)-q_b(t))  \\
-\frac{\pi_b(t)}{m} - \frac{\lambda}{2m}(q_a(t)-q_b(t)) \\
-\frac{\lambda}{2m}(\pi_a(t)-\pi_b(t))-kq_a(t)          \\
\frac{\lambda}{2m}(\pi_a(t)-\pi_b(t))+kq_b(t)
\end{bmatrix}.
\ea
Notice that the physical limit in \Cref{eq:pl} is an invariant submanifold of $\R^{4d}$ with respect to the dynamics, since $\dot{\pi}_a+\dot{\pi}_b=0$ whenever $q_a=q_b$, and also $\dot{q}_a-\dot{q}_b=0$. To ensure that our neural networks preserve this subspace as well, we apply a projection step at the end of the network, which projects over the subspace in \Cref{eq:pl} and is defined as
\[
\begin{bmatrix}
    q_a & q_b & \pi_a & \pi_b 
\end{bmatrix} \mapsto \begin{bmatrix}
    \frac{q_a+q_b}{2} & \frac{q_a+q_b}{2} & \frac{\pi_a-\pi_b}{2} & -\frac{\pi_a-\pi_b}{2}
\end{bmatrix}.
\]
We remark that this projection turns the \texttt{SympFlow} into a non-symplectic map in this context, but it is not important since the actual physical system we are integrating is not a conservative Hamiltonian system. The symplectic structure encoded in the network still seems to improve the approximation accuracy, as we show in the numerical experiments \Cref{sec:experiments}.

Similarly, the equations of motion for the double variables $(q_a,q_b,p_a,p_b)$ are given by:
\ba
\dot{p}_a & = \dot{\pi}_a + \frac{\lambda}{2}(\dot{q}_a-\dot{q}_b) =-\frac{\lambda}{m}p_2-kq_a    \\
\dot{p}_b & = -\dot{\pi}_b + \frac{\lambda}{2}(\dot{q}_a-\dot{q}_b) = -\frac{\lambda}{m}p_b-kq_b. \\
\dot{q}_a & = \frac{\lambda}{2m}(q_a-q_b) + \frac{p_a}{m} - \frac{\lambda}{2m}(q_a-q_b) = \frac{p_a}{m}  \\
\dot{q}_b & =-\frac{\lambda}{2m}(q_a-q_b) + \frac{p_b}{m} + \frac{\lambda}{2m}(q_a-q_b) = \frac{p_b}{m}.
\ea
Notice that in the figures in this paper the letters $p_a$ and $p_b$ denote the momenta. We emphasize that these momenta are the non-conservative ones, since they are the result of solving the augmented system. However, because of the projection, they coincide with the conservative ones, i.e., $\pi_a=p_a$ and $-\pi_b=p_b$, hence why we use this convention as discussed above.

\section{Proof of Theorem~\ref{thm:universal}}\label{app:universal}
%\section{Proof of \Cref{thm:universal}}\label{app:universal}
This appendix provides a proof of \Cref{thm:universal}, i.e., that any Hamiltonian flow can be approximated to arbitrary accuracy, uniformly on a compact time interval $[0,\Delta t]$, by a suitable \texttt{SympFlow}. Our proof builds on the derivations in \cite{supp:turaev2002polynomial}.

\subsection*{Step 1: Approximation with polynomials}
\begin{lemma}[Approximation with polynomial Hamiltonian]\label{lemma:weierstrass}
Let $H:\R\times\R^{2d}\to\R$ be a \two{twice-continuously differentiable} function. Fix $\Delta t>0$. Consider a compact set $\Omega\subset\R^{2d}$ \one{and assume $\Omega$ is forward invariant for $\phi_{H,t}$}. For any $\varepsilon>0$, there is a function $\widetilde{H}:\R\times\R^{2d}\to\R$ polynomial in the phase-space variable $x\in\R^{2d}$ and with coefficients continuously depending on time, whose exact flow $\phi_{\widetilde{H},t}$ approximates $\phi_{H,t}$ to accuracy $\varepsilon$ on $\Omega\times [0,\Delta t]$.
\end{lemma}
\begin{proof}
\two{Define
\[
\Omega_\varepsilon := \Omega + B_\infty(\varepsilon) = \left\{x\in\mathbb{R}^{2d}:\min_{y\in\Omega}\|x-y\|_\infty\leq \varepsilon\right\}.
\]
Call $\Omega_{\varepsilon,c}$ a compact and convex set containing $\Omega_\varepsilon$, and define
\[
L:=\sup_{\substack{s\in [0,\Delta t] \\ x\in \Omega_{\varepsilon,c}}}\left\|\nabla^2_x H(s,x)\right\|_{\infty},\,\text{ and }\,\widetilde{\varepsilon} := \frac{\varepsilon}{\Delta t \exp(\Delta t L)},
\]
where $\nabla_x^2 H$ is the Hessian in the second variable. $L$ is finite as $\Omega_{\varepsilon,c}$ is compact and $H$ is twice-continuously differentiable on the whole space. Consider a polynomial function $\widetilde{H}:\mathbb{R}\times\mathbb{R}^{2d}\to\mathbb{R}$ such that
\begin{align}
&\sup_{\substack{t\in [0,\Delta t] \\ x\in\Omega_\varepsilon}}\left|H(t,x)-\widetilde{H}(t,x)\right|<\widetilde{\varepsilon},\label{eq:closeFunctions} \\
&\sup_{\substack{t\in [0,\Delta t] \\ x\in\Omega_\varepsilon}}\left\|\nabla_x H(t,x) - \nabla_x \widetilde{H}(t,x)\right\|_{\infty}<\widetilde{\varepsilon},\label{eq:closeGradients}
\end{align}
whose existence is guaranteed by the Weierstrass Approximation Theorem \cite[Theorem 1.6.2]{supp:narasimhan1985analysis} since $\Omega_\varepsilon$ is compact. We remark that $\nabla_x H$ stands for the gradient of $H$ in the second variable.}

Since the approximation only holds on $\Omega_\varepsilon$, we introduce
\[
\tau(x) := \min\left\{\Delta t,\inf\{t\in [0,+\infty):\,\phi_{\widetilde{H},t}(x)\in \partial \Omega_{\varepsilon}\}\right\}.
\] 
Since $\widetilde H$ is polynomial, the associated vector field is $C^1$ (hence locally Lipschitz in the state variable).
Therefore, by the Picard--Lindel\"of Theorem, for every $x\in\Omega$ there exists a unique local solution
$t\mapsto \phi_{\widetilde H,t}(x)$, which is in particular continuous and satisfies $\phi_{\widetilde H,0}(x)=x\in\Omega\subset\Omega_\varepsilon$.
Consequently $\tau(x)>0$. Moreover, define
\[
B:=\sup_{\substack{t\in [0,\Delta t] \\ z\in \Omega_{\varepsilon,c}}}\left\|\nabla_x \widetilde{H}(t,z)\right\|_\infty <\infty,
\]
which is finite by continuity of $\nabla_x\widetilde H$ on the compact set $[0,\Delta t]\times\Omega_{\varepsilon,c}$.
Then, for every $x\in\Omega$ and $t\in[0,\tau(x)]$, we have
\[
\left\|x - \phi_{\widetilde{H},t}(x)\right\|_\infty
\leq \int_0^t \left\|\nabla_x \widetilde{H}(s,\phi_{\widetilde{H},s}(x))\right\|_\infty \,\mathrm ds
\leq Bt,
\]
since $\phi_{\widetilde H,s}(x)\in\Omega_\varepsilon\subset\Omega_{\varepsilon,c}$ for all $s\in[0,t]$.
If $\tau(x)<\min\{\Delta t,\varepsilon/B\}$, then
\[
\min_{y\in\Omega}\|y-\phi_{\widetilde{H},\tau(x)}(x)\|_\infty \le \left\|x - \phi_{\widetilde{H},\tau(x)}(x)\right\|_\infty \le B\tau(x)<\varepsilon,
\]
contradicting the definition
of $\tau(x)$. Hence $\tau(x)\geq \min\{\Delta t,\varepsilon/B\}>0$ for every $x\in\Omega$.

For any $x\in\Omega$ and $t\in [0,\tau(x)]$, we can write
\begin{align*}
&\left\|\phi_{H,t}(x) - \phi_{\widetilde{H},t}(x)\right\|_{\infty} = \left\|\mathbf{J} \int_0^t \left(\nabla_x H(s,\phi_{H,s}(x)) - \nabla_x \widetilde{H}(s,\phi_{\widetilde{H},s}(x))\right)\dd s\right\|_{\infty} = \\
&=\left\|\mathbf{J} \int_0^t \left(\nabla_x H(s,\phi_{H,s}(x)) - \nabla_x H(s,\phi_{\widetilde{H},s}(x)) + \nabla_x H(s,\phi_{\widetilde{H},s}(x)) - \nabla_x \widetilde{H}(s,\phi_{\widetilde{H},s}(x))\right)\dd s\right\|_{\infty}\\
&\leq \underbrace{L\int_0^t \left\|\phi_{H,s}(x) - \phi_{\widetilde{H},s}(x)\right\|_{\infty}\dd s}_{(a)} + \int_0^t \left\|\nabla_x H(s,\phi_{\widetilde{H},s}(x)) - \nabla_x \widetilde{H}(s,\phi_{\widetilde{H},s}(x))\right\|_{\infty} \dd s.
\end{align*}
$(a)$ follows from the Lipschitz regularity of the gradient of twice-continuously differentiable functions on compact and convex sets. 
Since $t\in [0,\tau(x)]$, then \Cref{eq:closeGradients} applies and we get
\[
\left\|\phi_{H,t}(x) - \phi_{\widetilde{H},t}(x)\right\|_{\infty} \leq L\int_0^t \left\|\phi_{H,s}(x) - \phi_{\widetilde{H},s}(x)\right\|_{\infty}\dd s + \widetilde{\varepsilon}t.
\]
Gronwall's inequality allows us to conclude that, for any $x\in \Omega$ and $t\in [0,\tau(x)]$ we can write
\begin{align*}
\left\|\phi_{H,t}(x) - \phi_{\widetilde{H},t}(x)\right\|_{\infty}\leq \widetilde{\varepsilon}t \exp\left(Lt \right)\leq \widetilde{\varepsilon}\Delta t \exp\left(L\Delta t \right) = \varepsilon.
\end{align*}
\two{We will now show that $\tau(x) = \Delta t$. Assume, by contradiction, that $\tau(x)<\Delta t$ for an $x\in\Omega$, i.e., $\phi_{\widetilde{H},\tau(x)}(x)\in\partial\Omega_\varepsilon$. Then, by the forward invariance of $\Omega$, $\phi_{{H},\tau(x)}(x)\in\Omega$ and hence
\[
\|\phi_{{H},\tau(x)}(x)-\phi_{\widetilde{H},\tau(x)}(x)\|_\infty \geq
\min_{y\in\Omega}\|y-\phi_{\widetilde{H},\tau(x)}(x)\|_\infty = \varepsilon.
\]
This is a contradiction since the previous derivation implies
\[
\|\phi_{{H},\tau(x)}(x)-\phi_{\widetilde{H},\tau(x)}(x)\|_\infty \leq \widetilde{\varepsilon}\tau(x) \exp\left(L\tau(x) \right)< \widetilde{\varepsilon}\Delta t \exp\left(L\Delta t \right) = \varepsilon.
\]
We conclude that $\tau(x)=\Delta t$ for every $x$ and the desired result follows.}
\end{proof}
\subsection*{Step 2: Splitting of the exact flow in $N$ substeps}
Thanks to \Cref{lemma:weierstrass}, we can now assume to work with $\widetilde{H}:\R\times\Omega\to\R$ which is a polynomial in the phase-space variable $x$ and with coefficients continuously depending on time. We can then split the flow $\phi_{\widetilde{H},t}$ into $N$ substeps of size $t/N$ as follows:
\begin{equation}\label{eq:split}
\phi_{\widetilde{H},t} = \widetilde{\phi}_{(N-1)t/N}\circ \cdots\circ \widetilde{\phi}_{nt/N} \circ \cdots \circ \widetilde{\phi}_{0},\,\, n=0,\ldots,N-1.
\end{equation}
In \Cref{eq:split}, we denote with $\widetilde{\phi}_{nt/N}(x_0)$ the exact solution at time $(n+1)t/N$ of the initial value problem
\[
\begin{cases}
    \frac{\dd x(s)}{\dd s} = \mathbf J \nabla_x \widetilde{H}(s,x)\\
    x(nt/N) = x_0.
\end{cases}
\]

\subsection*{Step 3: Approximation via a separable Hamiltonian system}

The goal is now getting closer to the form of differential equations defining the layers of the \texttt{SympFlow}. To do so, we work with \cite[Lemma 1]{supp:turaev2002polynomial}, which we now state adapting the notation to ours.

\begin{lemma}[Lemma 1 in \cite{supp:turaev2002polynomial}]\label{lemma:polynomialTuraev}
Let $\widetilde{H}:\R\times{\R^{2d}}\to\R$ be polynomial in the phase-space variable $x\in{\R^{2d}}$ and depending continuously on the first variable. {Fix $t\in [0,+\infty)$ and let $\Omega\subset\R^{2d}$ be an arbitrary compact set.} There exists a function $V:\R\times \R^d\to\R$ polynomial in the second variable and with coefficients depending continuously on the first variable and a set of $d$ integers $\omega_1,\ldots,\omega_d$, such that the Hamiltonian
\begin{equation}\label{eq:secondOrderHam_other}
\widehat{H}(s,x) = \pi\left(\left\|p\right\|_2^2 + q^\top \diag(\omega_1^2,\ldots,\omega_d^2)q\right)+ \frac{t}{N}V(s,q),\,\,x=(q,p)\in\R^{2d},
\end{equation}
with time-$1$ flow $\phi_{\widehat{H},1}$ which is $\mathcal{O}(t^2/N^2)$ close to $\widetilde{\phi}_{nt/N}$ {in $\Omega$, i.e., 
\[
\sup_{x\in \Omega}\left\|\phi_{\widehat{H},1}(x) - \widetilde{\phi}_{nt/N}(x)\right\|_\infty = \mathcal{O}(t^2/N^2).
\]
}
\end{lemma}
{We remark that the big-O notation $f(t)=\mathcal{O}(g(t))$, for a pair of functions $f,g:[0,+\infty)\to\mathbb{R}$ means that there is a constant $c>0$ such that $|f(t)|\leq c|g(t)|$ for all $t\geq 0$.}
We do not specify it in the notation since it is already quite heavy, but this approximation depends on the time instant $nt/N$, so $\widehat{H}$ would be $\widehat{H}_n$. We remark that the differential equation defined by $\widehat{H}$ in \Cref{eq:secondOrderHam_other},
\begin{equation}\label{eq:1storder}
\begin{bmatrix}
\dot{q}(s) \\ \dot{p}(s)
\end{bmatrix} = \begin{bmatrix}
    2\pi p(s) \\ - 2\pi\diag(\omega_1^2,\ldots,\omega_d^2)q(s) - \frac{t}{N}\nabla_q V(s,q(s))
\end{bmatrix},
\end{equation}
can be seen as a perturbation of the linear oscillator
\begin{equation}\label{eq:2ndODE}
\frac{\dd^2}{\dd s^2}q(s) = - 4\pi^2\diag(\omega_1^2,\ldots,\omega_d^2)q(s).% - \frac{t}{N}\nabla_q V(s,q(s)).
\end{equation}

\begin{corollary}\label{corollary:2ndOrderMLP}
Let $\widetilde{H}:\R\times{\R^{2d}}\to\R$ be polynomial in the phase-space variable $x\in{\R^{2d}}$ and depending continuously on the first variable. {Consider an arbitrary $\varepsilon>0$, $t\in [0,+\infty)$, and a compact set $\Omega\subset\R^{2d}$}. {Let
\begin{equation}\label{eq:functionSpace}
\mathcal{F}\two{_\sigma} = \left\{\R\times\R^d\ni(s,q)\mapsto w^\top\sigma\left(Aq+bs+c\right)\in\R:\,\,A\in\R^{h\times d},b,c,w,\in\R^h,\,\,h\in\mathbb{N}\right\},
\end{equation}
with $\sigma:\R\to\R$ continuously differentiable and not a polynomial, which is applied entrywise. There exists a Hamiltonian function
\begin{equation}\label{eq:secondOrderHam}
\bar{H}(s,x) = \bar{V}_{\mom}(s,p) + \bar{V}_{\pos}(s,q),\,\,\bar{V}_{\mom},\,\bar{V}_{\pos}\in\mathcal{F}_{\sigma}, x=(q,p)\in\R^{2d},
\end{equation}
whose exact time-$1$ flow $\phi_{\bar{H},1}$ is $\mathcal{O}(t^2/N^2)$ close to $\widetilde{\phi}_{nt/N}$, i.e., 
\[
\sup_{x\in \Omega}\left\|\widetilde{\phi}_{nt/N}(x)-\phi_{\bar{H},1}(x)\right\|_\infty =  \mathcal{O}(t^2/N^2).
\]
}
\end{corollary}
\Cref{lemma:polynomialTuraev} is fundamental to prove \Cref{corollary:2ndOrderMLP} since we first need to get a separable Hamiltonian, so that we can relate these maps with a \texttt{SympFlow}. Jumping over \Cref{lemma:polynomialTuraev} would have required a significantly new analysis to directly obtain \Cref{corollary:2ndOrderMLP}.
\begin{proof}[Proof of \Cref{corollary:2ndOrderMLP}]
{This proof mimics the one of Lemma \ref{lemma:weierstrass}. Lemma \ref{lemma:polynomialTuraev} ensures that there exists $\widehat{H}$ as in \eqref{eq:secondOrderHam_other}, and a constant $c>0$, such that
\[
\sup_{x\in\Omega}\left\|\phi_{\widehat{H},1}(x) - \widetilde{\phi}_{nt/N}(x)\right\|_\infty \leq c \frac{t^2}{N^2}.
\] 
Let us define 
\[
\widehat{\Omega}:=\left\{y\in\R^{2d}:\,\,y=\phi_{\widehat{H},s}(x),\,x\in\Omega,\,s\in [0,1]\right\},
\]
which is compact. Let $\widehat{\Omega}_\varepsilon = \widehat{\Omega} + B_\infty(\varepsilon)$, and call $\widehat{\Omega}_{\varepsilon,c}\subset\R^{2d}$ a compact and convex set containing $\widehat{\Omega}_\varepsilon$. Set
\[
L:=\sup_{\substack{s\in [0,1] \\ x\in \widehat{\Omega}_{\varepsilon,c}}}\left\|\nabla_x^2\widehat{H}(s,x)\right\|_\infty,\,\,\widehat{\varepsilon}:=\frac{\varepsilon}{\exp(L)}.
\]
}By \cite[Theorem 4.1]{supp:pinkus1999approximation}, there exists a pair of functions $\bar{V}_{\pos}=\bar{V}_{\pos}(s,q)$ and $\bar{V}_{\mom}=\bar{V}_{\mom}(s,p)$ that belong to $\mathcal{F}\two{_\sigma}$ and such that
\begin{equation}\label{eq:polyNet}
\sup_{\substack{s\in [0,1] \\ x\in\widehat{\Omega}_\varepsilon}}\left| \widehat{H}(s,x) -  \bar{H}(s,x)\right|<\widehat{\varepsilon},\,\,\sup_{\substack{s\in [0,1] \\ x\in\widehat{\Omega}_\varepsilon}}\left\|\nabla \widehat{H}(s,x) - \nabla \bar{H}(s,x)\right\|_{\infty}<\widehat{\varepsilon},
\end{equation}
where $\bar{H}$ is defined as in \Cref{eq:secondOrderHam}. This is a consequence of the fact that polynomials are \two{continuously differentiable} functions, and $\mathcal{F}\two{_\sigma}$ is dense in the space of continuously differentiable scalar-valued functions with respect to the $C^1-$topology. {Let us introduce the function $\tau:\Omega\to [0,1]$ defined as
\[
\tau(x) = \min\left\{1,\inf\left\{s\in [0,+\infty)\,\mid\,\phi_{\bar{H},s}(x)\in\partial\widehat{\Omega}_\varepsilon\right\}\right\}.
\]
For an $x\in\Omega$, we can then write
\begin{align*}
\left\|\widetilde{\phi}_{t/N}(x) - \phi_{\bar{H},1}(x)\right\|_\infty &= \left\|\widetilde{\phi}_{t/N}(x) - \phi_{\widehat{H},1}(x) + \phi_{\widehat{H},1}(x) - \phi_{\bar{H},1}(x)\right\|_\infty \\
&\leq c\frac{t^2}{N^2} + \left\|\phi_{\widehat{H},1}(x) - \phi_{\bar{H},1}(x)\right\|_\infty
\end{align*}
and hence it remains to control the last term by applying the approximation bound in \eqref{eq:polyNet}. Let us fix $s\in [0,\tau(x)]$, a non-empty interval because of the same arguments as in the proof of Lemma \ref{lemma:weierstrass}, and consider
\begin{align*}
\left\|\phi_{\widehat{H},s}(x) - \phi_{\bar{H},s}(x)\right\|_\infty &= \left\|\mathbf{J}\int_0^s \left(\nabla_x \widehat{H}(u,\phi_{\widehat{H},u}(x)) - \nabla_x \bar{H}(u,\phi_{\bar{H},u}(x))\right)\dd u\right\|_\infty \\
&\leq \int_0^s \left\|\nabla_x \widehat{H}(u,\phi_{\widehat{H},u}(x)) - \nabla_x \widehat{H}(u,\phi_{\bar{H},u}(x))\right\|_\infty \dd u \\
&+ \int_0^s \left\|\nabla_x \widehat{H}(u,\phi_{\bar{H},u}(x)) - \nabla_x \bar{H}(u,\phi_{\bar{H},u}(x))\right\|_\infty \dd u \\
&\leq \underbrace{L\int_0^s\left\|\phi_{\widehat{H},u}(x) - \phi_{\bar{H},u}(x)\right\|_\infty \dd u}_{\phi_{\widehat{H},u}(x),\phi_{\bar{H},u}(x)\in\widehat{\Omega}_{\varepsilon,c}} + \underbrace{\widehat{\varepsilon}s}_{\phi_{\bar{H},u}(x)\in \widehat{\Omega}_\varepsilon}.
\end{align*}
Gronwall's inequality then allows to obtain
\[
\left\|\phi_{\widehat{H},s}(x) - \phi_{\bar{H},s}(x)\right\|_\infty \leq \widehat{\varepsilon} s \exp(Ls)\leq \widehat{\varepsilon}\exp(L)=\varepsilon,\,\,s\in [0,\tau(x)].
\]
We can then show that $\tau(x)=1$ for every $x\in \Omega$. By contradiction, let us assume $0< \tau(x)<1$ for an $x\in\Omega$, i.e., $\phi_{\bar{H},\tau(x)}(x)\in \partial \widehat{\Omega}_\varepsilon$. By definition of $\widehat{\Omega}$, we have that $\phi_{\widehat{H},\tau(x)}(x)\in \widehat{\Omega}$, and hence
\[
\left\| \phi_{\widehat{H},\tau(x)}(x) - \phi_{\bar{H},\tau(x)}(x) \right\|_\infty \geq \min_{y\in \widehat{\Omega}}\left\|y - \phi_{\bar{H},\tau(x)}(x)\right\|_\infty = \varepsilon.
\]
This is a contradiction since we have just derived that
\[
\left\| \phi_{\widehat{H},\tau(x)}(x) - \phi_{\bar{H},\tau(x)}(x) \right\|_\infty \leq \widehat{\varepsilon}\tau(x)\exp(L\tau(x)) < \widehat{\varepsilon}\exp(L)=\varepsilon.
\]
We conclude that $\tau(x)=1$ for every $x\in\Omega$, and hence
\[
\left\|\widetilde{\phi}_{t/N}(x) - \phi_{\bar{H},1}(x)\right\|_\infty \leq c\frac{t^2}{N^2} + \varepsilon.
\]
By the arbitrarity of $\varepsilon$ in the proof, we conclude that taking $\varepsilon=t^2/N^2$, it follows that there exists $\bar{H}$ realizing
\[
\sup_{x\in\Omega}\left\|\widetilde{\phi}_{t/N}(x) - \phi_{\bar{H},1}(x)\right\|_\infty = \mathcal{O}(t^2/N^2).
\]
}
\end{proof}

\subsection*{Step 4: Approximation of $\phi_{\bar{H},1}$ with a \texttt{SympFlow}}

To conclude, we need to show that the flow map $\phi_{\bar{H},1}$ with $\bar{H}$ defined as in \Cref{eq:secondOrderHam} can be approximated by a suitable \texttt{SympFlow}. To do so, we apply the decomposition in \Cref{eq:split} to $\phi_{\bar{H},1}$. We split the time interval $[0,1]$ into $K$ substeps, leading to substeps that we denote as $\bar{\phi}_{k/K}$ associated to the integration interval $[k/K,(k+1)/K]$, with $k=0,\ldots,K-1$:
\two{
\[
\phi_{\bar{H},1} = \bar{\phi}_{(K-1)/K}\,\circ\,\cdots\circ\bar{\phi}_{2/K}\,\circ\,\bar{\phi}_{0/K},\,\,\bar{\phi}_{k/K}(x) := \phi_{\bar{H},k/K,1/K}(x).
\]
}
We can then approximate $\bar{\phi}_{k/K}$ up to $\mathcal{O}(1/K^2)$ with a $1-$layer \texttt{SympFlow} obtained by applying a Lie-Trotter splitting strategy to \Cref{eq:1storder}. More explicitly, \two{as expanded on in \Cref{app:background}}, we can write
\two{
\begin{align}
\bar{\phi}_{k/K} &= \bar{\phi}_{\bar{H}_1,k/K}\circ \bar{\phi}_{\bar{H}_2,k/K} + \mathcal{O}\left(\frac{1}{K^2}\right),\label{eq:lietrotter}\\
\bar{\phi}_{\bar{H}_1,k/K}(x) &= \begin{bmatrix} q + \int_{k/K}^{(k+1)/K}\nabla_p \bar{V}_\mom(s,p)\dd s \\ p \end{bmatrix},\,\bar{\phi}_{\bar{H}_2,k/K}(x) = \begin{bmatrix} q \\ p - \int_{k/K}^{(k+1)/K}\nabla_q \bar{V}_\pos(s,q)\dd s \end{bmatrix},\label{eq:H1H2}
\end{align}
}
with $x=(q,p)$, $\bar{H}_1(s,x)=\bar{V}_{\mom}(s,p)$ and $\bar{H}_2(s,x)=\bar{V}_{\pos}(s,q)$. \two{The map $\bar{\phi}_{\bar{H}_1,k/K}\circ \bar{\phi}_{\bar{H}_2,k/K}$ is a \texttt{SympFlow} layer and, not to interrupt the flow of the proof, we show this in \Cref{lemma:itIsSympFlow}.}
Having $K$ of these sub-steps, we can thus conclude that there exists a \texttt{SympFlow} of at most $K$ layers able to approximate within $\mathcal{O}(1/K)$ the time-$1$ flow $\phi_{\bar{H},1}$, { and, on compact sets, such an approximation is uniform in the phase space variable $x$}. This error accumulation is due to the fact that we compose $K$ maps $\mathcal{O}(1/K^2)$ close to the reference one, {an estimate following from the quadratic local error of Lie-Trotter, see \cite{blanes2024splitting}}.  Thus, we have a global error growing like $1/K$. The $\mathcal{O}$ notation hides the multiplicative constants, which depend on the Lipschitz regularity of the flow maps.

\subsection*{Step 5: Combining all the approximations to conclude the proof}

We can now conclude the proof by combining the various approximations done in the previous steps and using the fact that the composition of \texttt{SympFlows} is again a \texttt{SympFlow}. We first recall the decomposition
\[
\phi_{\widetilde{H},t} = \widetilde{\phi}_{(N-1)t/N}\circ \cdots\circ \widetilde{\phi}_{nt/N} \circ \cdots \circ \widetilde{\phi}_{0},\,\, n=0,\ldots,N-1.
\]
In steps 3 and 4 we showed that for each $n=0,\ldots,N-1$ there is a \texttt{SympFlow} of at most $K$ layers approximating within $\mathcal{O}(1/K)$ the exact flow $\phi_{\widehat{H},1}$, which is known to be $\mathcal{O}(\Delta t^2/N^2)$ close to $\widetilde{\phi}_{nt/N}$, see \Cref{corollary:2ndOrderMLP}. {We recall that all these approximations can be made uniform in the phase space variable $x$ when working on compact sets.} \two{More explicitly, we have shown that
\begin{align}
\phi_{\widetilde{H},t} &= \widetilde{\phi}_{(N-1)t/N} \circ \cdots \circ \widetilde{\phi}_{0}  \nonumber \\
&=\left(\phi_{\bar{H}_{N-1},1} + \mathcal{O}\left(\frac{\Delta t^2}{N^2}\right)\right) \circ \cdots \circ \left(\phi_{\bar{H}_0,1} + \mathcal{O}\left(\frac{\Delta t^2}{N^2}\right)\right) \nonumber \\
&= \phi_{\bar{H}_{N-1},1}\circ \cdots \circ \phi_{\bar{H}_0,1} + \mathcal{O}\left(\frac{\Delta t^2}{N}\right)\text{ (Step 3)}, \nonumber \\
\phi_{\bar{H}_n,1} &= \texttt{SympFlow}_{K,n} + \mathcal{O}\left(\frac{1}{K}\right),\,n=0,...,N-1\text{ (Step 4)}, \label{eq:Klayers}\\
\phi_{\widetilde{H},t} &= \texttt{SympFlow} + \mathcal{O}\left(\frac{N}{K}\right) + \mathcal{O}\left(\frac{\Delta t^2}{N}\right),\nonumber\\
\texttt{SympFlow} &= \texttt{SympFlow}_{K,N-1}\circ \cdots \circ \texttt{SympFlow}_{K,0}.\nonumber
\end{align}
We remark that the notation in \Cref{eq:Klayers} stands for a \texttt{SympFlow} depending on $n$, and with at most $K$ layers.} 
If we set $K=N^2$, we can thus conclude that the map $\phi_{\widetilde{H},t}$, and hence also $\phi_{H,t}$, can be approximated within $\mathcal{O}(\Delta t^2/N)$ by a \texttt{SympFlow} of at most $N^3$ layers. $N$ can then be selected based on the desired accuracy $\varepsilon$. This reasoning works assuming that $\Delta t$ is of moderate size, such as $\Delta t=1$, which we use in our experiments. Otherwise, one would have to slightly modify the reasoning.
$\blacksquare$

\two{
\begin{lemma}\label{lemma:itIsSympFlow}
Let $\sigma:\mathbb{R}\to\mathbb{R}$ be a continuously differentiable function and consider
\[
\bar{H}_1(s,x)=\bar{V}_\mom(s,p)=w_1^\top \sigma(A_1 p + b_1 s + c_1),\,\bar{H}_2(s,x)=\bar{V}_\pos(s,q)=w_2^\top \sigma(A_2 q + b_2 s + c_2),
\]
$\bar{H}_1,\bar{H}_2\in\mathcal{F}_\sigma$. Assume that all the entries of $b_1$ and $b_2$ are not zero. The map $\bar{\phi}_{\bar{H}_1,k/K}\circ \bar{\phi}_{\bar{H}_2,k/K}:\mathbb{R}^{2d}\to\mathbb{R}^{2d}$, defined as in \Cref{eq:H1H2}, is a layer of a \texttt{SympFlow}.
\end{lemma}
We note that the assumption that the entries of $b_1$ and $b_2$ never vanish is not restrictive in our case. If this were not satisfied, we could slightly perturb the entries of $b_1$ and $b_2$ and still achieve the same approximation rates. Working with zero entries would also be possible, but would result in a non-conventional \texttt{SympFlow} with some neurons having different activation functions from others.
\begin{proof}
Let us focus on
\[
\bar{\phi}_{\bar{H}_1,k/K}(x) = \begin{bmatrix} q + \int_{k/K}^{(k+1)/K}\nabla_p \bar{V}_\mom(s,p)\dd s \\ p \end{bmatrix},\,\,x=(q,p)\in\mathbb{R}^{2d}.
\]
We see that
\[
\int_{k/K}^{(k+1)/K}\nabla_p \bar{V}_\mom(s,p)\dd s = \int_0^{1/K} \nabla_p \bar{V}_\mom\left(\frac{k}{K}+s,p\right)\dd s,
\]
where
\[
\bar{V}_\mom\left(\frac{k}{K}+s,p\right) = w^\top \sigma(A p + b s + \bar{c}),\,\,\bar{c} = c + \frac{k}{K}b,
\]
suppressing the weight subscripts to lighten the notation. Assume $A\in\mathbb{R}^{h\times d},\,\bar{c},b\in\mathbb{R}^{h}$. Calling $\gamma:\mathbb{R}\to\mathbb{R}$ a primitive of $\sigma$, $\gamma'=\sigma$, {we} can also write
\[
\bar{V}_\mom\left(\frac{k}{K}+s,p\right) = \sum_{i=1}^h\partial_s\left(\frac{w_i}{b_i}\gamma((Ap)_i + b_i s + \bar{c}_i)\right) =: \partial_s\left(\bar{w}^\top \gamma(Ap+bs+\bar{c})\right),
\]
where $\bar{w}_i = w_i / b_i$, $i=1,...,h$. We can then conclude that
\begin{align*}
\int_{k/K}^{(k+1)/K}\nabla_p \bar{V}_\mom(s,p)\dd s &= \nabla_p \left(\bar{w}^\top \gamma\left(Ap+\frac{b}{K}+\bar{c}\right) - \bar{w}^\top \gamma\left(Ap+\bar{c}\right) \right)\\
&=\left(A^\top \diag(\bar{w}) \sigma\left(Ap+\frac{b}{K}+\bar{c}\right) - A^\top \diag(\bar{w})\sigma\left(Ap+\bar{c}\right) \right).
\end{align*}
In summary, we see that
\begin{align*}
\bar{\phi}_{\bar{H}_1,k/K}(x) &= \begin{bmatrix} q + \int_{k/K}^{(k+1)/K}\nabla_p \bar{V}_\mom(s,p)\dd s \\ p \end{bmatrix} \\
&= \begin{bmatrix}
    q + \left(A^\top \diag(\bar{w}) \sigma\left(Ap+\frac{b}{K}+\bar{c}\right) - A^\top \diag(\bar{w})\sigma\left(Ap+\bar{c}\right) \right) \\
     p
\end{bmatrix}\\
&=\begin{bmatrix} q + \nabla_p \tilde{V}_\mom(1/K,p) - \nabla_p \tilde{V}_\mom(0,p) \\ p  \end{bmatrix}
\end{align*}
where
\[
\tilde{V}_\mom(s,p) = \bar{w}^\top\gamma(Ap+bs+\bar{c})\in\mathcal{F}_\gamma.
\]
This allows us to conclude since $\tilde{V}_\mom$ is a single-layer neural network, i.e., a particular case of an MLP.
\end{proof}
\begin{remark}
\Cref{lemma:itIsSympFlow} also implies that \texttt{SympFlows} with layers based on Hamiltonians that are single hidden layer networks are universal, and one would not need to rely on more complicated MLPs from a theoretical standpoint. Practically, however, it may still be beneficial.
\end{remark}
}

\section{Additional experiments for the Simple Harmonic Oscillator}\label{app:mixed}
\one{We now include the results obtained from training \texttt{SympFlow} for the unsupervised experiments of the Simple Harmonic Oscillator by using the mixed training procedure. The results are in \Cref{fig:mixedSimpleHO}.}
\begin{figure}[ht!]
    \centering
    \includegraphics[scale=0.5]{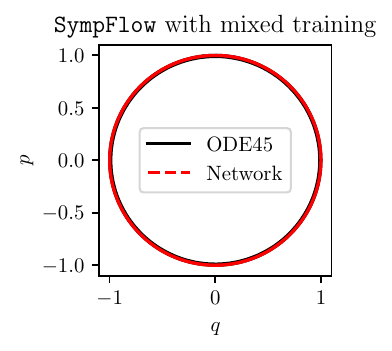}
    \includegraphics[scale=0.5]{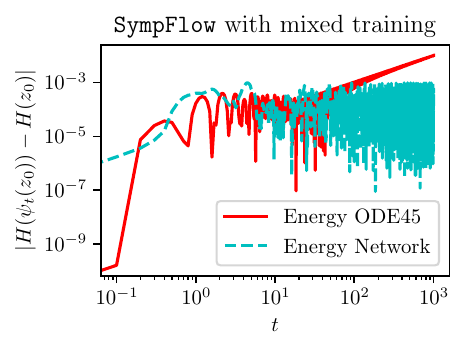}
    \caption{Mixed training procedure for the simple harmonic oscillator}
    \label{fig:mixedSimpleHO}
\end{figure}

\section{Additional experiments for H\'enon--Heiles}\label{app:heiles}
This section collects some additional numerical experiments for the H\'enon--Heiles system. \Cref{fig:heilesUnsupervisedSolutions} shows the solutions produced with the five training configurations considered in the paper. We see that the regularization is significantly improving the performance of the MLP, but not so much for the \texttt{SympFlow}, which already performs well just with the residual. The mixed training procedure leads to some improvements over the regularized one. \Cref{fig:heilesSupervisedSolutions} collects the solutions for the supervised models, whereas \Cref{fig:referenceSection} is a reference Poincarè section for H\'enon--Heiles obtained with Runge--Kutta (5,4) time integration.
\begin{figure}
    \centering
    \includegraphics[scale=.5]{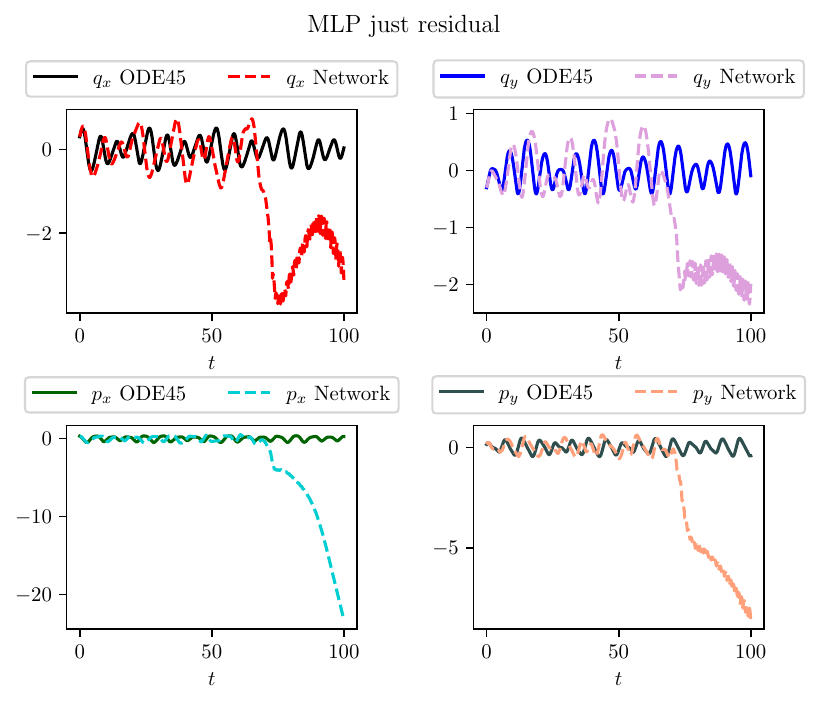}
    \includegraphics[scale=.5]{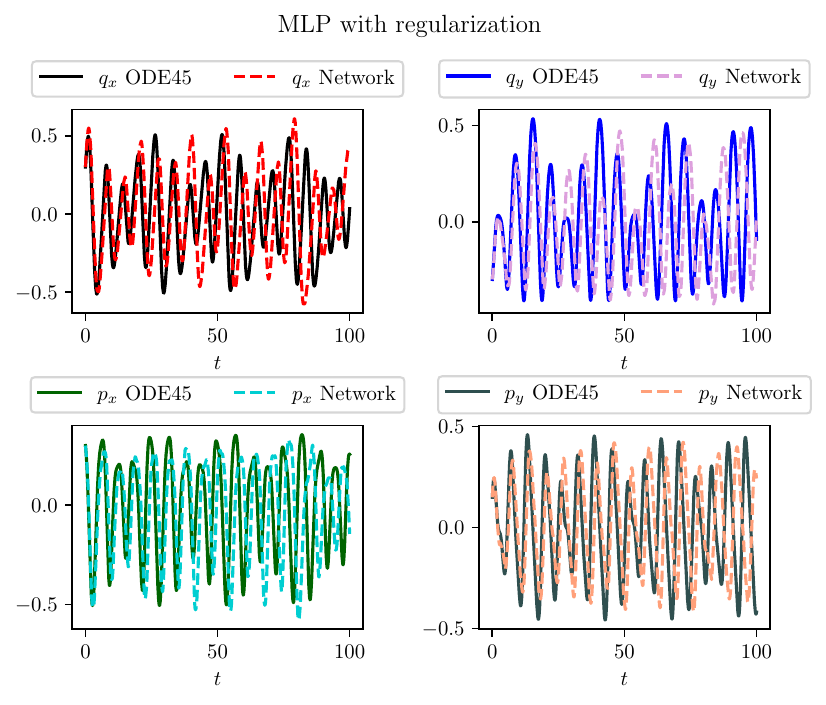}

    \includegraphics[scale=.5]{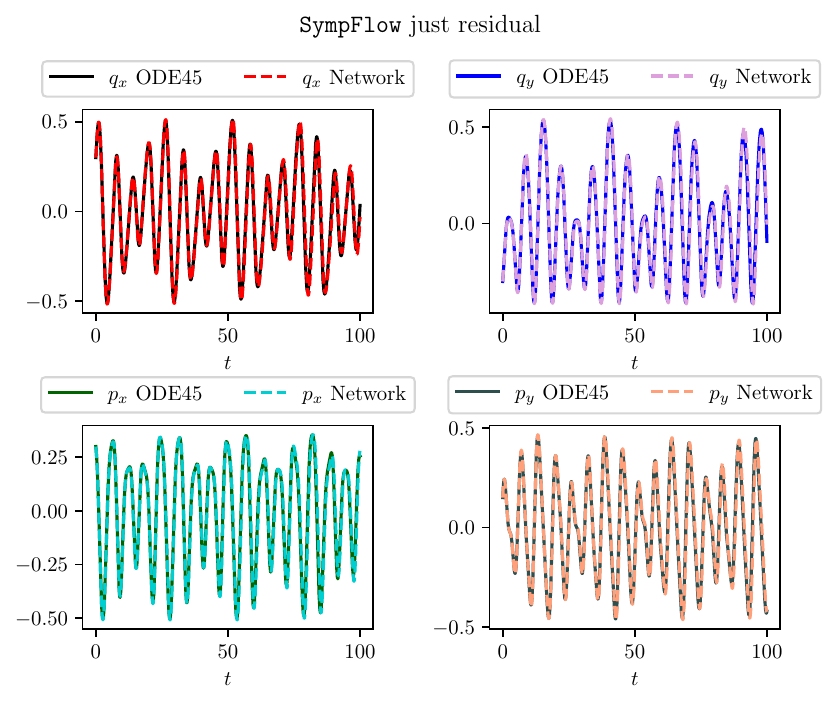}
    \includegraphics[scale=.5]{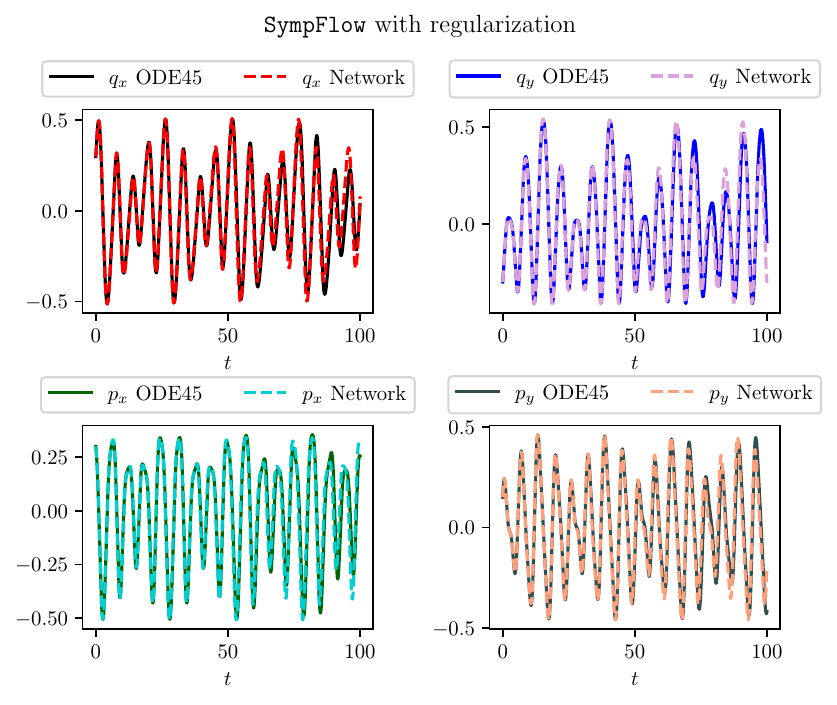}

    \includegraphics[scale=.5]{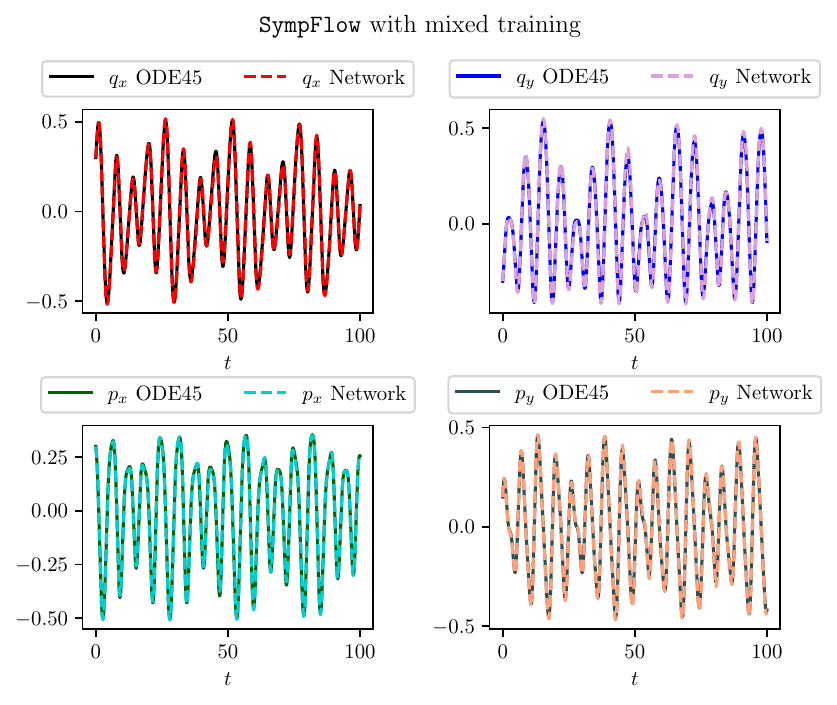}
    \caption{\textbf{Unsupervised experiment --- H\'enon--Heiles: }Comparison of the solutions obtained with the different methods, with integration time $[0,T=100]$.}
    \label{fig:heilesUnsupervisedSolutions}
\end{figure}

\begin{figure}
    \centering
    \includegraphics[scale=.5]{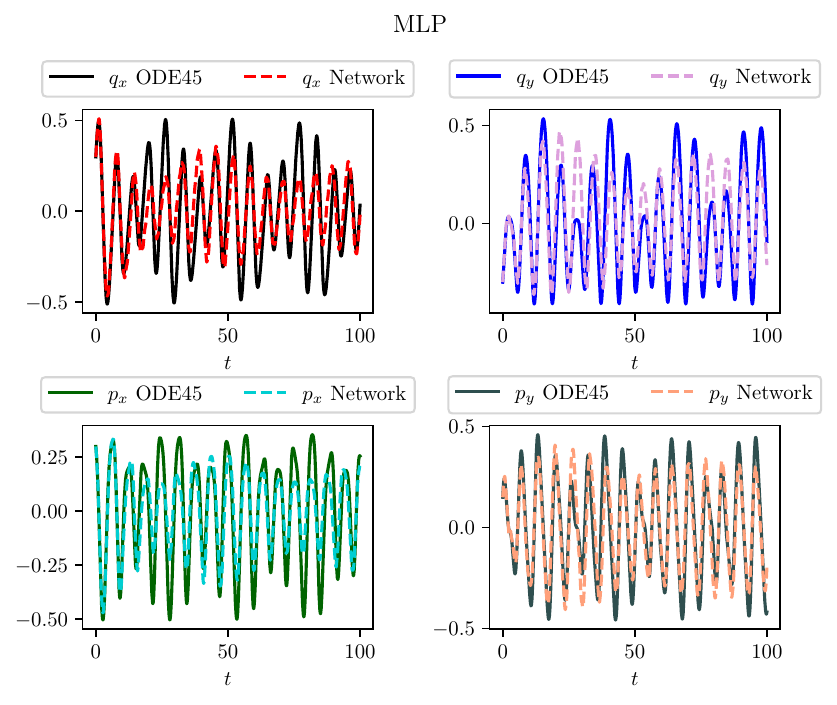}
    \includegraphics[scale=.5]{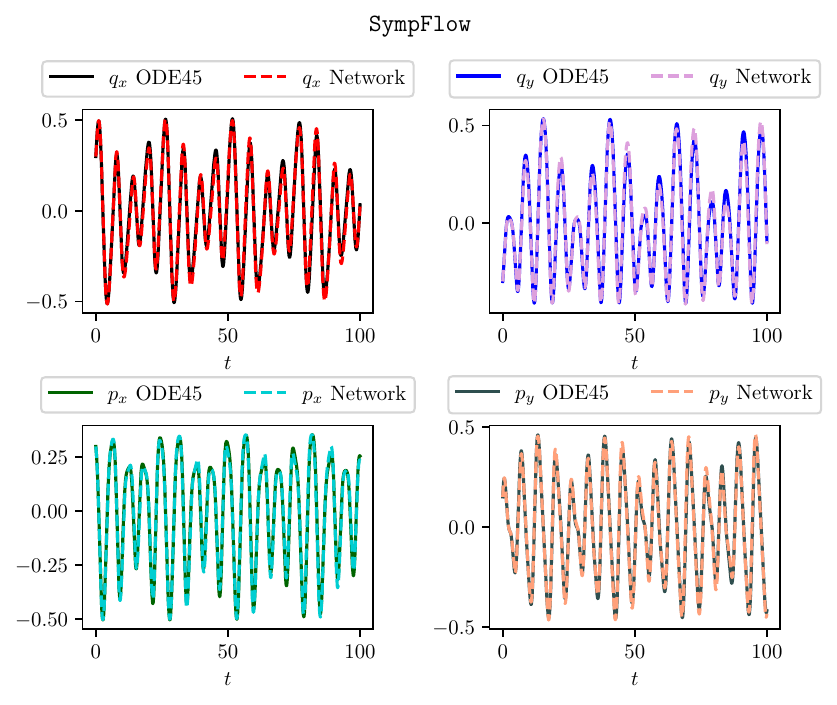}
    \caption{\textbf{Supervised experiment --- H\'enon--Heiles: }Comparison of the solutions obtained with the different methods, with integration time $[0,T=100]$.}
    \label{fig:heilesSupervisedSolutions}
\end{figure}

\begin{figure}
    \centering
    \includegraphics[scale=1]{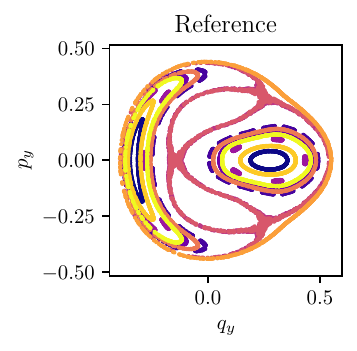}
    \caption{Poincarè section for H\'enon--Heiles obtained with Runge--Kutta (5,4) time integration.}
    \label{fig:referenceSection}
\end{figure}

\two{\section{Model comparison with varying depth}\label{app:timing}
We now train and compare \texttt{SympFlow} and the MLP with a varying number of layers. We consider models with $4$, $8$, and $16$ layers, focusing on the unsupervised experiments of the simple harmonic oscillator. For both models and all the layer numbers, we consider the training procedure with and without Hamiltonian regularization.}

\two{\Cref{tab:timing} collects the average time per training epoch, and the average time per evaluation of one second of the solution for a single initial condition at inference time. These timings correspond to unsupervised experiments for the simple harmonic oscillator. We train for $10,000$ epochs, where for each we sample $700$ new collocation points to evaluate the loss function, and we also test the accuracy over $5,000$ test samples. The average training time is obtained by timing the full training loop and dividing by $10,000$. For inference, we predict the solution up to time $T=1,000$ for $N=100$ initial conditions. The total time to produce these $100$ solutions is then divided by $N\cdot T$, hence producing an average inference time per second of simulation. \texttt{SympFlow} is more expensive to evaluate, and further work is needed to improve its implementation efficiency. This additional cost is mostly due to the need to compute gradients for the forward pass, compared to a conventional MLP. On top of the evaluation of \texttt{SympFlow}, it is also evident that the Hamiltonian regularization contributes to an increased overhead cost at training time. In contrast, it has no impact on the inference time.  Still, this additional cost is worth it given the considerably improved long-term reliability of the produced solution, compared to that of the MLP.}
\begin{table}[t]
  \small
  \centering
  \caption{\two{\textbf{Timings.} This table collects the average training time per epoch and inference time per unit of time of four models with varying numbers of layers $L\in\{4,8,16\}$. The experiments are done on a Quadro RTX 6000 24220MiB GPU. The time is expressed in seconds ($s$), and it is measured with the Python package \texttt{time}.}}
  \setlength{\tabcolsep}{3pt}
  \begin{tabular}{c|ccc|ccc}
    \toprule
    \multirow{2}{*}{\textbf{Model}}
      & \multicolumn{3}{c|}{\textbf{Training time ($s$)}} 
      & \multicolumn{3}{c}{\textbf{Inference time ($s$)}} \\
      \cmidrule(lr){2-4}\cmidrule(l){5-7}
      & \textbf{4} & \textbf{8} & \textbf{16} & \textbf{4} & \textbf{8} & \textbf{16} \\
    \midrule
    \texttt{SympFlow} just residual & 0.3914 & 0.7231 & 1.3931 & 0.2730 & 0.5328 & 1.0528 \\
    \texttt{SympFlow} with regularization & 0.4659 & 0.8395 & 1.6222  & 0.2747 & 0.5292 & 1.0427  \\
    MLP just residual & 0.0232 & 0.0326 & 0.05112 & 0.0197 & 0.0212 & 0.0226    \\
    MLP with regularization &0.0264 & 0.0369 & 0.0584 & 0.0199 & 0.0213 & 0.0226 \\
    \bottomrule
  \end{tabular}
  \label{tab:timing}
\end{table}

\two{In \Cref{tab:accuracies}, we instead compare the average relative $\ell^2$ error of the produced solution and the average relative absolute energy error for the considered models, after having applied them up to time $k\Delta t$, $\Delta t=1$, for $k=1,10,100$. These errors are defined as in the supervised experiments of the simple harmonic oscillator, see \Cref{eq:averageSolError} for the solution error, and \Cref{eq:averageHamError} for the error in the Hamiltonian. The models considered here are the same leading to the timings in \Cref{tab:timing}. We see that the \texttt{SympFlow} models consistently outperform the corresponding MLPs. As expected, for any number of layers, the energy behavior of \texttt{SympFlow} is stable, especially when compared with the one produced by MLPs. We also notice that these experiments do not display any significant improvement when the number of layers is increased. This behavior is more evident in the MLP models than for \texttt{SympFlow}. However, this may reflect that the deeper networks have not fully converged within $10,000$ epochs. The choice of stopping at this fixed number of epochs is made to have comparable timings of the $12$ experiments. Furthermore, focusing on the effects of regularization, we see that the performance of \texttt{SympFlow} tends to be slightly worse when regularized. On the other hand, the long-time energy behavior of the MLPs is consistently more stable with a regularized loss function.} 

\begin{table}[t]
  \small
  \centering
  \caption{\two{\textbf{Accuracy.} This table collects the relative solution and energy average errors of four models with varying numbers of layers $L\in\{4,8,16\}$. We compute the errors after applying the networks to reach the prediction time $k\Delta t$, with $\Delta t=1$ and $k\in\{1, 10, 100 \}$. The errors correspond to an average computed over $100$ test initial conditions. The number on the side of the model name stands for the number of layers in the network. The solution errors are scaled by a factor of $10^{-3}$, %$10^{-4}$,
  whereas the energy errors by $10^{-2}$.}}
  \setlength{\tabcolsep}{3pt}     % tighter horizontal padding
  \begin{tabular}{c|ccc|ccc}
    \toprule
    \multirow{2}{*}{\textbf{Model}}
      & \multicolumn{3}{c|}{\textbf{Solution error $(\cdot 10^{-3})$}}
      & \multicolumn{3}{c}{\textbf{Energy error $(\cdot 10^{-2})$}} \\
      \cmidrule(lr){2-4}\cmidrule(l){5-7}
      & $\Delta t$ & $10\Delta t$  & $100\Delta t$ & $\Delta t$ & $10\Delta t$  & $100\Delta t$ \\
    \midrule
\texttt{SympFlow} with regularization, 4 & $7.277$ & $17.604$ & $85.731$ & $1.431$ & $1.552$ & $0.506$\\
\texttt{SympFlow} with regularization, 8 & $1.983$ & $17.910$ & $173.982$ & $0.206$ & $0.580$ & $0.515$\\
\texttt{SympFlow} with regularization, 16 & $1.554$ & $12.506$ & $117.964$ & $0.163$ & $0.281$ & $0.221$\\
\texttt{SympFlow} just residual, 4 & $1.915$ & $6.698$ & $42.731$ & $0.190$ & $0.726$ & $0.370$\\
\texttt{SympFlow} just residual, 8 & $0.902$ & $9.086$ & $92.343$ & $0.098$ & $0.142$ & $0.124$\\
\texttt{SympFlow} just residual, 16 & $0.848$ & $6.376$ & $66.430$ & $0.087$ & $0.138$ & $0.124$\\
MLP with regularization, 4 & $3.292$ & $14.426$ & $124.304$ & $0.422$ & $1.786$ & $12.972$\\
MLP with regularization, 8 & $5.663$ & $31.249$ & $247.627$ & $0.669$ & $2.357$ & $12.800$\\
MLP with regularization, 16 & $8.603$ & $43.347$ & $276.006$ & $1.102$ & $6.347$ & $38.693$\\
MLP just residual, 4 & $3.986$ & $20.827$ & $199.441$ & $0.480$ & $2.158$ & $16.440$\\
MLP just residual, 8 & $5.680$ & $40.154$ & $325.044$ & $0.759$ & $4.178$ & $35.408$\\
MLP just residual, 16 & $7.398$ & $52.419$ & $436.470$ & $0.987$ & $7.287$ & $49.032$\\
 \bottomrule
  \end{tabular}
  \label{tab:accuracies}
\end{table}

\two{\Cref{fig:loss} provides a comparison of the loss curves while training, and over a fixed test set. We observe that \texttt{SympFlow} converges faster than the MLP, as it achieves better loss function values after a relatively low number of epochs. This is easier to appreciate in the models trained without regularization, where the loss functions of the two models coincide. We also observe that, for this problem, neither \texttt{SympFlow} nor the MLP appear to benefit from an increased number of layers, as the final training and test losses are relatively close. Still, we remark that the final loss for \texttt{SympFlow} is consistently smaller than that of the MLP, which supports the interest in further exploring this type of constrained architecture for these tasks. We note that the test loss only includes the residual component. This is why the regularized \texttt{SympFlows} have a test loss which is smaller than the training one.}

\begin{figure}
\centering
\begin{subfigure}{0.3\textwidth}
\centering
\includegraphics[width=\linewidth]{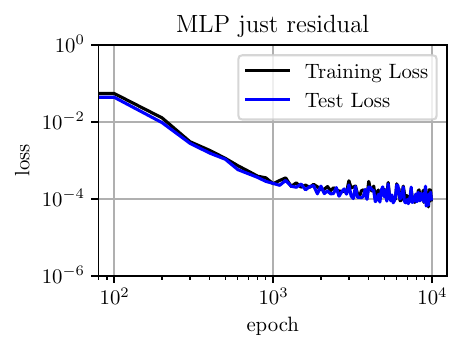}
\caption{4 layers}
\label{fig:pinnNoReg_4layers}
\end{subfigure}
\begin{subfigure}{0.3\textwidth}
\centering
\includegraphics[width=\linewidth]{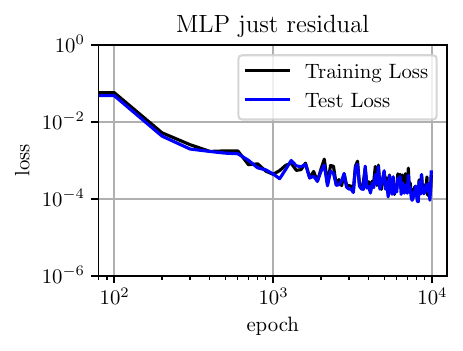}
\caption{8 layers}
\label{fig:pinnNoReg_8layers}
\end{subfigure}
\begin{subfigure}{0.3\textwidth}
\centering
\includegraphics[width=\linewidth]{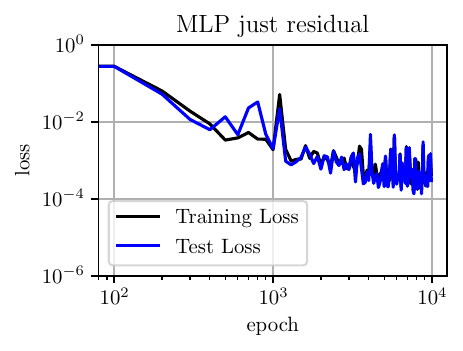}
\caption{16 layers}
\label{fig:pinnNoReg_16layers}
\end{subfigure}

\begin{subfigure}{0.3\textwidth}
\centering
\includegraphics[width=\linewidth]{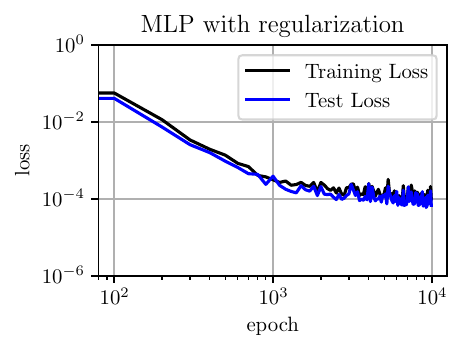}
\caption{4 layers}
\label{fig:pinnReg_4layers}
\end{subfigure}
\begin{subfigure}{0.3\textwidth}
\centering
\includegraphics[width=\linewidth]{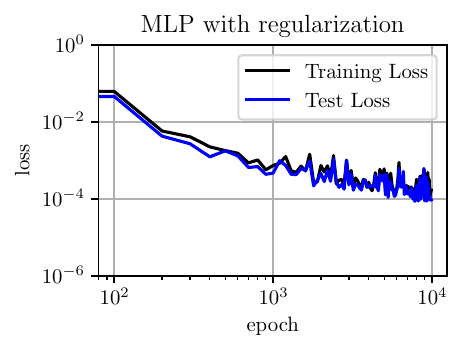}
\caption{8 layers}
\label{fig:pinnReg_8layers}
\end{subfigure}
\begin{subfigure}{0.3\textwidth}
\centering
\includegraphics[width=\linewidth]{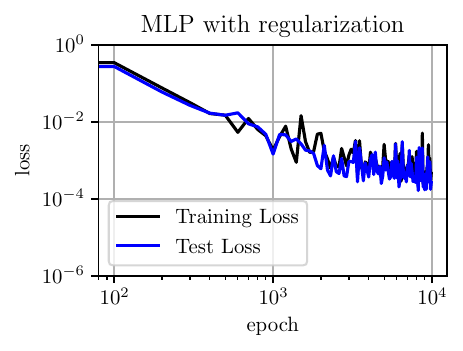}
\caption{16 layers}
\label{fig:pinnReg_16layers}
\end{subfigure}

\begin{subfigure}{0.3\textwidth}
\centering
\includegraphics[width=\linewidth]{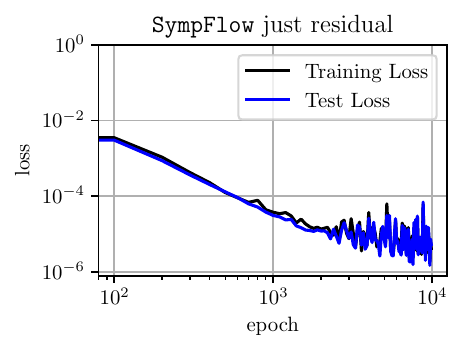}
\caption{4 layers}
\label{fig:noHamReg_4layers}
\end{subfigure}
\begin{subfigure}{0.3\textwidth}
\centering
\includegraphics[width=\linewidth]{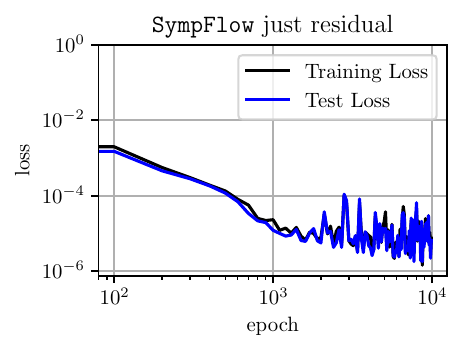}
\caption{8 layers}
\label{fig:noHamReg_8layers}
\end{subfigure}
\begin{subfigure}{0.3\textwidth}
\centering
\includegraphics[width=\linewidth]{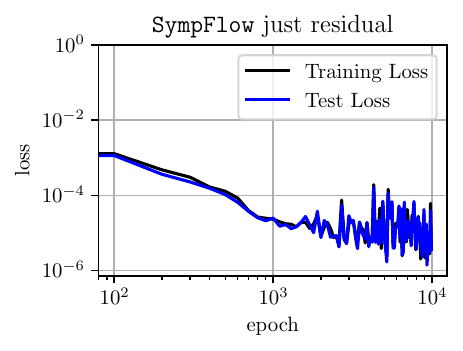}
\caption{16 layers}
\label{fig:noHamReg_16layers}
\end{subfigure}

\begin{subfigure}{0.3\textwidth}
\centering
\includegraphics[width=\linewidth]{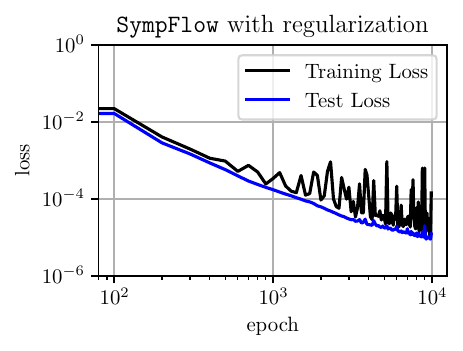}
\caption{4 layers}
\label{fig:hamReg_4layers}
\end{subfigure}
\begin{subfigure}{0.3\textwidth}
\centering
\includegraphics[width=\linewidth]{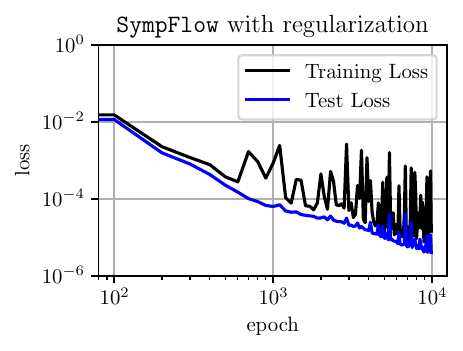}
\caption{8 layers}
\label{fig:hamReg_8layers}
\end{subfigure}
\begin{subfigure}{0.3\textwidth}
\centering
\includegraphics[width=\linewidth]{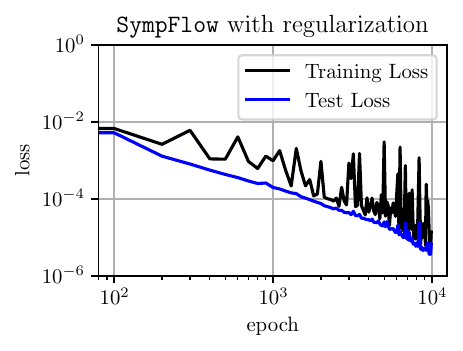}
\caption{16 layers}
\label{fig:hamReg_16layers}
\end{subfigure}
\caption{\textbf{Training and test loss:} We plot the training and test loss curves for the four models trained with a varying number of layers $L\in\{4,8,16\}$.}
\label{fig:loss}
\end{figure}

\section{Comparison with Symplectic Recurrent Neural Network}\label{app:srnn}

\two{We now include a comparison of the models we have considered with the Symplectic Recurrent Neural Network (SRNN) model introduced in \cite{Chen2020Symplectic}. We perform such a comparison for the supervised experiments targeting the simple harmonic oscillator.}

\two{The idea behind SRNN is to parameterize the unknown Hamiltonian energy $H:\mathbb{R}^{2d}\to\mathbb{R}$ with a scalar-valued neural network $H_\theta:\mathbb{R}^{2d}\to\mathbb{R}$. Then, the approximate solutions are generated with a symplectic integrator $\varphi_\theta^h:\mathbb{R}^{2d}\to\mathbb{R}^{2d}$. The authors of \cite{Chen2020Symplectic} propose to minimize the loss function
\[
\mathcal{L}_{\text{SRNN}}(\theta) = \frac{1}{NM} \sum_{n=1}^N \sum_{m=1}^M \left\|\varphi_{\theta}^{h_n^m}\circ \cdots \circ \varphi_\theta^{h_n^1}(x_n^0) - x_n^m\right\|_{2}^2.
\]
Since the role of this $M$ and the one used for the supervised experiments of our continuous-in-time networks is different, we focus on the case $M=1$, and $N=1500$, and compare the results corresponding to a varying noise magnitude in the training trajectories. To be able to use explicit symplectic numerical methods, the focus is on separable Hamiltonians, meaning that the neural network $H_\theta$ is defined based on two other functions $K_\theta,V_\theta: \mathbb{R}^d\to\mathbb{R}$ as $H_\theta(q,p)=K_\theta(p)+V_\theta(q)$. We model $K_\theta$ and $V_\theta$ as MLPs with an architecture leading to a comparable number of weights to those of \texttt{SympFlow}. The integrator $\varphi_\theta^h$ that we use is then the Symplectic Euler method. We now compare this model with \texttt{SympFlow}. We remark that this separability assumption is not necessary for \texttt{SympFlow}.}

\two{We include the results in \Cref{tab:srnn}. For this comparison, we reduce the $\Delta t$ from one, the usual value we consider in the paper, to $\Delta t=0.1$. This choice is because only for smaller time steps does a numerical method provide a solution which resembles the exact solution of the parametric Hamiltonian. \texttt{SympFlow} does not suffer from this limitation, given that $\Delta t$ is not to be interpreted as a time step, but as an upper bound on the region where the network is trained.} We observe that both SRNN and \texttt{SympFlow} perform well overall, as a consequence of their symplectic structure. In these experiments, \texttt{SympFlow} achieves lower solution errors in most settings, while SRNN is mainly better at zero noise (with a few isolated wins at other noise levels).

\begin{table}[t]
   \small
  \centering
  \caption{\two{\textbf{Comparison with SRNN.} This table collects the relative solution and energy average errors of the SRNN and \texttt{SympFlow} models with varying numbers of noise magnitudes $\varepsilon\in\{0,0.005,0.01,0.02,0.03,0.04,0.05\}$. The training set corresponds to $N=1500$ pairs $((x_n,h_n),y_n)$, $n=1,...,N$, of an initial condition $x_n$, a time step $h_n$ and the updated position $y_n\approx \phi_{H,h_n}(x_n)$, where $H$ is the Hamiltonian of the simple harmonic oscillator. We sample $h_n\in [0,\Delta t=0.1]$ uniformly. We compute the errors after having applied the networks to reach the prediction time $k\Delta t$, with $\Delta t=0.1$ and $k\in\{1,10,100\}$. The errors correspond to an average computed over $100$ test initial conditions. The number on the side of the model name stands for noise intensity. The solution errors are scaled by a factor of $10^{-3}$, whereas the energy errors by $10^{-2}$. Lower errors for a fixed noise and prediction level are highlighted.}}
  \setlength{\tabcolsep}{3pt}
  \begin{tabular}{l|ccc|ccc}
    \toprule
    \multirow{2}{*}{\textbf{Model}}
      & \multicolumn{3}{c|}{\textbf{Solution error $(\cdot 10^{-3})$}}
      & \multicolumn{3}{c}{\textbf{Energy error $(\cdot 10^{-2})$}} \\
      \cmidrule(lr){2-4}\cmidrule(l){5-7}
      & $\Delta t$ & $10\Delta t$  & $100\Delta t$ & $\Delta t$ & $10\Delta t$  & $100\Delta t$ \\
    \midrule
\texttt{SympFlow}, $\varepsilon=0.0$ & 1.562 & 19.021 & 167.652 & 0.144 & 1.136 & 2.179 \\
SRNN, $\varepsilon=0.0$ & \textbf{1.347} & \textbf{17.166} & \textbf{143.446} & \textbf{0.119} & \textbf{0.892} & \textbf{0.786} \\
    \midrule
\texttt{SympFlow}, $\varepsilon=0.005$ & \textbf{1.708} & \textbf{21.037} & \textbf{175.532} & \textbf{0.178} & \textbf{1.632} & 2.401 \\
SRNN, $\varepsilon=0.005$ & 1.930 & 25.487 & 185.259 & 0.212 & 1.759 & \textbf{2.184} \\
    \midrule
\texttt{SympFlow}, $\varepsilon=0.01$ & \textbf{2.507} & \textbf{29.709} & 184.588 & 0.324 & \textbf{3.166} & \textbf{4.190} \\
SRNN, $\varepsilon=0.01$ & 2.722 & 33.205 & \textbf{184.311} & \textbf{0.313} & 3.354 & 5.154 \\
    \midrule
\texttt{SympFlow}, $\varepsilon=0.02$ & \textbf{2.623} & \textbf{28.812} & \textbf{184.513} & \textbf{0.271} & \textbf{2.478} & \textbf{3.750} \\
SRNN, $\varepsilon=0.02$ & 4.923 & 51.195 & 297.099 & 0.526 & 4.691 & 9.688 \\
    \midrule
\texttt{SympFlow}, $\varepsilon=0.03$ & \textbf{2.469} & \textbf{27.012} & \textbf{162.378} & \textbf{0.329} & \textbf{2.638} & \textbf{3.417} \\
SRNN, $\varepsilon=0.03$ & 6.155 & 74.184 & 415.789 & 0.737 & 7.302 & 14.453 \\
    \midrule
\texttt{SympFlow}, $\varepsilon=0.04$ & \textbf{3.493} & \textbf{35.479} & \textbf{150.782} & \textbf{0.456} & \textbf{3.882} & 6.734 \\
SRNN, $\varepsilon=0.04$ & 5.106 & 54.381 & 350.609 & 0.545 & 4.062 & \textbf{6.124} \\
    \midrule
\texttt{SympFlow}, $\varepsilon=0.05$ & \textbf{4.246} & \textbf{48.107} & \textbf{235.879} & \textbf{0.487} & \textbf{4.646} & \textbf{7.506} \\
SRNN, $\varepsilon=0.05$ & 13.174 & 132.718 & 499.740 & 2.011 & 18.751 & 31.791 \\
    \bottomrule
  \end{tabular}
  \label{tab:srnn}
\end{table}

\two{
\begin{remark}
\texttt{SympFlow} and SRNN have several similarities. The main difference is that \texttt{SympFlow} has layers associated with the exact flows of different time-dependent Hamiltonian systems. In contrast, the SRNN has a single time-independent Hamiltonian governing the entire network, which is generated by simulating this Hamiltonian system using a symplectic integrator. To improve the robustness to noise, in \cite{Chen2020Symplectic}, the authors focused on considering a larger $M$. Given the experiments in \Cref{tab:srnn}, we see that our \texttt{SympFlow} has some inherent regularization in this sense. It will be worth exploring in the future different training strategies, such as those involving larger $M$ as proposed in \cite{Chen2020Symplectic}.
\end{remark}
}

\newpage

\bibliographystyle{plain}
\bibliography{references,referencesSupp}

%----- END -----%
\end{document}